\documentclass[12pt]{article}
\usepackage[top=1in, bottom=1in, left=1in, right=1in]{geometry}

\usepackage{enumerate,graphicx,enumitem}
\usepackage{lmodern,subfigure}
\usepackage[T1]{fontenc}
\usepackage{verbatim}
\usepackage{textcomp,wrapfig}

\usepackage[english]{babel}
\usepackage{hyperref,xcolor}
\usepackage{bbm}

\usepackage{amsmath,amsfonts,amssymb,amsthm,mathrsfs}

\newtheorem{theorem}{Theorem}[section]

\newtheorem{remark}[theorem]{Remark}
\theoremstyle{definition}

\newtheorem{assumption}[theorem]{Assumption}

\def\RR{\mathbb{R}}\def\R{\mathbb{R}}

\def\NN{\mathbb{N}}

\def\PP{\mathbb{P}}
\def\EE{\mathbb{E}}

\def\al{\alpha}
\def\be{\beta}
\def\ga{\gamma}

\def\dh2l{\mathbf{d}_{\mathbb{H}_{2\ell}}}
\def\d2{\mathbf{d}_2}

\def\d{\mathrm{d}}
\def\mgo{M^{\ga}(\Omega)}
\def\nh{\hat{\nu}}
\def\D{\mathcal{E}}
\def\Del{\Delta}
\def\single{\mathrm{single}}
\def\multi{\mathrm{multi}}
\def\ol{\overline}
\def\G{\mathcal{G}}
\def\inn{\mathrm{in}}
\def\out{\mathrm{out}}
\def\mg{M^\ga}
\allowdisplaybreaks

\newcommand\numberthis{\addtocounter{equation}{1}\tag{\theequation}}
\usepackage[noabbrev,capitalise,nameinlink]{cleveref}
\hypersetup{colorlinks={true},linkcolor={blue},citecolor=blue}

\begin{document}
\title{Fractal Gaussian Networks:  A sparse random graph model based on Gaussian Multiplicative Chaos}
\author{
Subhroshekhar Ghosh\thanks{
Department of Mathematics, National University of Singapore, SG 119076
~~\texttt{subhrowork@gmail.com}
 }~~~~
Krishnakumar Balasubramanian\thanks{
Department of Statistics, University of California, Davis, CA 95616 USA ~~\texttt{kbala@ucdavis.edu}
 }~~~~
Xiaochuan Yang\thanks{Department of Mathematics, Brunel University, London, UB8 3PH, UK ~~\texttt{xiaochuan.j.yang@gmail.com
}
}
}
\date{}
\maketitle
\begin{abstract}
We propose a novel stochastic network model, called Fractal Gaussian Network (FGN), that embodies well-defined and analytically tractable fractal structures. Such fractal structures have been empirically observed in diverse applications. FGNs interpolate continuously between the popular \textit{purely random} geometric graphs (a.k.a. the Poisson Boolean network), and random graphs with increasingly fractal behavior. In fact, they form a parametric family of \textit{sparse} random geometric graphs that are parametrized by a fractality parameter $\nu$ which governs the strength of the fractal structure. FGNs are driven by the latent spatial geometry of Gaussian Multiplicative Chaos (GMC), a canonical model of fractality in its own right. We asymptotically  characterize the expected number of edges, triangles, cliques and hub-and-spoke motifs in FGNs, unveiling a distinct pattern in their scaling with the size parameter of the network. We then examine the natural question of detecting the presence of fractality and the problem of parameter estimation based on observed network data, in addition to fundamental properties of the FGN as a random graph model. We also explore fractality in community structures by unveiling a natural stochastic block model in the setting of FGNs. Finally, we substantiate our results with phenomenological analysis of the FGN in the context of available scientific literature for fractality in networks, including applications to real-world massive network data. 
\end{abstract}

\section{Stochastic Networks and Fractality}
\textbf{The \emph{unreasonable effectiveness} of stochastic networks.} Stochastic networks have emerged as one of the fundamental modeling paradigms in the last few decades in our efforts to effectively understand interactions underlying vast amounts of data with increasing complexity, and in order to capture the effects of latent factors. At a broad level of abstraction, this involves \textit{nodes} representing agents, and \textit{edges} (weighted or otherwise) that embody the interactions between these agents.
Indeed, the ubiquity of stochastic network models in the modern applied sciences may justifiably remind one of Eugene Wigner's famous article, \cite{wigner1990unreasonable}, on the \textit{unreasonable effectiveness} of mathematics in the natural sciences. Within the domain of stochastic networks, many popular modeling formulations have been proposed and investigated, in order to understand different types of phenomena in large complex systems. These include the fundamental Erd\H{o}s-R\'enyi random graph model, the preferential attachment model and its variants, random geometric graphs, graphons and related exchangeable models, the stochastic block model and its various avatars, small world networks like the Watts-Strogatz model, models of scale-free networks, to provide a partial list of examples (see, e.g., \cite{BA},\cite{strogatz2001exploring}, \cite{lovasz2012large}, \cite{erdos1959random}, \cite{penrose2003random},  \cite{bickel2009nonparametric}, \cite{holland1983stochastic}, \cite{orbanz2014bayesian}). The application domains for stochastic network models are diverse, encompassing the world-wide web and inter/intra-nets, collaboration networks in academia, and  social and communication networks.
Indeed, modern day network science has developed into a unique discipline of its own, for an overview of which we refer the reader to any amongst a multitude of excellent texts - at this point we mention \cite{watts2004six}, \cite{chung2006complex}, \cite{caldarelli2007scale}, \cite{kolaczyk2009statistical}, \cite{mezard2009information}, \cite{jackson2010social}, \cite{bollobas2010handbook}, \cite{lewis2011network}, \cite{barabasi2016network}, \cite{van2016random} and \cite{crane2018probabilistic} only to provide a partial list. As a preview to connect our present contribution to this classical literature, in this paper we aim to propose a novel paradigm of statistical networks with a view to capturing fractal phenomena.

\textbf{Fractal structures in large scale networks.} An important feature which has come to the fore in recent investigations of networks  is the emergence of inherent fractal structures in diverse application domains. Heuristically, fractal structures are often  characterized by \emph{non-standard} and \emph{anomalous} behavior of various scaling and growth exponents, and truncated power law tails for naturally associated statistics (c.f., \cite{falconer2004fractal}, \cite{mandelbrot1983fractal},  \cite{avnir1998geometry}). There are many instances of emergence of fractality in networks. To provide a detailed example, in human mobility networks, it has been observed that the layout of the way-points in the trajectories and the boundaries of popular sojourn domains exhibit fractal properties on a global scale, and the flight/pause times and inter-contact times between the agents exhibit power law tails (see, e.g., \cite{lee2011slaw}, \cite{rhee2011levy}).  Another important class of examples is the discovery of fractal structures in transportation networks, like urban bus transport networks and railway networks (\cite{benguigui1992fractal}, \cite{pavon2017fractal}, \cite{murcio2015multifractal}, \cite{salingaros2003connecting}) and drainage networks (\cite{rinaldo1992minimum}, \cite{rinaldo1993self}, \cite{la1989fractal} \cite{claps1996informational}). Fractality and multifractality are also known to arise in the context of scale-free and other complex networks (\cite{song2005self}, \cite{song2006origins}, \cite{kim2007fractality}), internet traffic (\cite{caldarelli2000fractal}) and financial networks; in fact, financial data in general present an important class of problems where fractal properties are known to occur (c.f., \cite{caldarelli2004emergence}, \cite{de2017fractal}, \cite{mandelbrot2013fractals}, \cite{mandelbrot2010mis}, \cite{Japan04}, \cite{evertsz1995fractal}). Fractal phenomena have emerged in sociological and ecological networks,  dense graphs and graphons (\cite{de2013models}, \cite{hill2008network}, \cite{gao2012culturomics}, \cite{palla2010multifractal}, \cite{lyudmyla2017fractal}), biological neural networks (\cite{bassett2006adaptive}), network dynamics (\cite{orbach1986dynamics}, \cite{goh2006skeleton}) and even in the field of development economics (\cite{barrett2006fractal}).

In view of the diversity of settings in which fractality has been observed to occur in networks, it is natural to investigate concrete mathematical models of fractality in networks which, on one hand, are amenable to rigorous theoretical analysis, and on the other hand, allow a broad enough horizon to study a reasonably wide class of interesting phenomena. Furthermore,  it would be of great interest to have a parametric statistical model, e.g. in the spirit of exponential families of classical parametric statistics (\cite{bickel2015mathematical}). This will open up a natural programme of investigation in terms of parameter estimation, tests of hypothesis with regard to fractal structures and examination of the model under parametric modulation. Towards that, in this work, we propose a parametric model of fractality in sparse networks, to understand fractal structures in a rigorous and analytically tractable manner. Based on a latent random field structure accorded by \textit{Gaussian Multiplicative Chaos} (GMC), a canonical model of fractal phenomena in various branches of natural and applied sciences, we call our model the \textit{Fractal Gaussian Network} model, which we will henceforth abbreviate as FGN.

\textbf{\textcolor{black}{Organization}.} The rest of the paper is organized as follows. In \cref{sec:gmcintro}, we discuss the fundamentals of Gaussian Multiplicative Chaos in a formal manner. In~\cref{sec:fgnmodel}, we introduce FGNs and discuss basic network properties including connectivity threshold, clustering coefficient and degree distribution. We conclude~\cref{sec:fgnmodel} by introducing a version of stochastic block model with the FGN paradigm. In~\cref{sec:mainresults}, we first provide theoretical results on basic motifs including edge, triangle, hub-and-spoke and $k$-clique counts. We next present interesting empirical observations regarding the spectrum of FGN, and provide estimators of the parameters of the FGN model. In~\cref{sec:perspective}, we provide a phenomenological perspective of FGN including a real-world data analysis.


\section{Gaussian Multiplicative Chaos: An Overview}\label{sec:gmcintro}

GMC is a canonical probabilistic model of fractal behavior in nature, endowed with statistical invariance properties that make it both an attractive mathematical structure as well as a robust modeling paradigm. Originating in the study of quantum field theory (\cite{hoegh1971general}, \cite{simon2015p}) and the seminal work of Jean-Pierre Kahane (\cite{kahane1985chaos}, \cite{kahane1976certaines}), it has many applications to fundamental problems like the study of quantum gravity (see, e.g., \cite{duplantier2009duality}, \cite{duplantier2011liouville}), as well as applied sciences where the GMC and related ideas have been effectively used to model volatility in financial assets and problems of turbulence (see, e.g., \cite{liu1999statistical}, \cite{duchon2012forecasting}, \cite{kolmogorov1941local}, \cite{kolmogorov1962refinement}, \cite{fyodorov2010freezing} and related literature). In this section, we provide a brief introduction to GMC, introducing tools which will aid in our  analytical investigations subsequently. For an elaborate discussion, we refer the reader to the extensive accounts \cite{rhodes2014gaussian}, \cite{rhodes2016lecture}, \cite{berestycki2015introduction}, \cite{berestycki2017elementary}, \cite{lacoin2019short} for a partial list, and the references contained therein. On this note, we also refer to the work \cite{robert2010gaussian}, which essentially revived interest in GMC among probabilists, after the seminal works of Kahane several decades ago.

Let $\{X_t(x), t\ge 0, x \in \RR^d\}$ be a centered Gaussian field, which is a standard Brownian motion as $t$ evolves for each fixed $x$ and
\begin{equation} \label{eq:pre-kernel}
\EE[X_s(x)X_t(y)] = \int_1^{e^{\min(s,t)}} \frac{k(u(x-y))}{u} du,
\end{equation}
therefore stationary in space variable.
The introduction of the above Brownian motion is helpful for computations; for more detail we refer the interested reader to \cite{duplantier2014critical}, \cite{duplantier2014renormalization} and \cite{duplantier2014log}.
We make the following assumptions on the kernel throughout this work.
\begin{assumption}\label{assumption:main}
The map $k:\RR^d \to [0,\infty)$ in~\cref{eq:pre-kernel}
\begin{itemize}[noitemsep]
\item satisfies $k(0)=1$,
\item is radial, i.e. $k(x)=k(\|x\|\vec{e})$ for any $x\in \RR^d$ and $\vec{e}=(1,0,...,0) \in \RR^d$, where $\|\cdot\|$ denotes the Euclidean norm.
\item is continuous and decays at infinity such that $\int_1^\infty \frac{k(u\vec{e})}{u} du<\infty$.
\end{itemize}
\end{assumption}
As $t\to\infty$ in~\cref{eq:pre-kernel}, one obtains a log-correlated Gaussian field $X$ as a random distribution. Indeed, it is easy to check that such functions $k$ lead to the limiting covariance function of the Gaussian field $X$ that has the following form :
\begin{align}\label{eq:covariance_function}
 K(x,y) = \ln_+ \frac{T}{\|x-y\|} + g(x-y)
 \end{align}
  where $T>0$ and $g$ is a bounded continuous function. Here we adopt the notation $\ln_+ = \max(\ln, 0).$

Gaussian Multiplicative Chaos (GMC) form  a natural family of random fractal measures. Roughly speaking, the GMC is defined on a Euclidean base space (e.g., a domain $\Omega \in \R^d$, scaled to have volume 1), and originates from an underlying centered Gaussian field $(X(x))_{x \in \Omega}$). Typically, on Euclidean spaces the Gaussian field $X$ is taken to be translation invariant and \textit{logarithmically correlated}. This entails, for example, that the covariance kernel $K$ of the Gaussian field $X$ has the form in~\cref{eq:covariance_function}.  Such fields arise naturally in many areas of mathematics, statistical physics and their applications, an important example being the celebrated Gaussian Free Field model (see, e.g., \cite{sheffield2007gaussian} and the references therein).

Let $\mu$ be a Radon measure on $\Omega$.  For any $\ga>0$ (with $\ga^2<2\dim(\mu)$ in order to ensure non-degeneracy of the limiting measure), we consider the random measure defined on $\Omega$ that is given, heuristically speaking, by the formula
\begin{equation} \label{eq:intro-eq}
\d M^\ga(x) : = \exp(\ga X(x) - \frac{\ga^2}{2}\EE[X(x)^2]) \d \mu(x).
\end{equation}
In the common setting of translation-invariance and $\mu$ the $d$-dimensional Lebesgue measure, this simply reduces to the form $\d M^\ga(x)=C_\ga \exp(\ga X(x)) \d x$, which is the setting on which we are going to focus in this article.  It is known that in this case the expected measure $\EE[\d M^\ga(x)]=\d x$, i.e. the $d$-dimensional Lebesgue measure, which provides a convenient background measure to compare a typical realization of the GMC with.  
With these ingredients in hand, we may define
\begin{equation} \label{eq:approx}
M^\ga := \lim_{t\to \infty} M^\ga_t \text{ a.s., where }  M^\ga_t(\d x)= e^{\ga X_t(x)-\frac{\ga^2}{2}\EE[X_t(x)^2]}  \d x,
\end{equation}
and the convergence is guaranteed by a martingale structure that is known to be inherent in this setting. Since, for each $x$, we have $\EE[X_t(x)^2]=t$, we may write \[M^\ga_t(\d x)= e^{\ga X_t(x)-\frac{\ga^2 t}{2}} \d x.\]

In this work, we set $\nu:=\frac{\ga^2}{d}$ to be the \textit{fractality parameter}. If $\ga^2<2d$ (equivalently, $\nu<2$), the limit $M^\ga$ is a non-degenerate measure, otherwise $M^\ga$ is a trivial zero measure. This regime $\nu<2$ where the GMC is a non-degenerate measure will be referred to as the subcritical regime. In our analytical considerations, we will  assume the GMC is subcritical and we consider the GMC on the $d$-dimensional unit cube $\Omega=[-1/2,1/2]^d$.

A crucial point is that, because of the logarithmic singularity of the covariance kernel, the Gaussian field $X$ is usually not well-defined as a function, but can be made sense of only as a Schwarz distribution (that acts on a smooth enough class of functions). Consequently, \cref{eq:intro-eq} that essentially purports to give a formulaic description of the GMC in terms of a random density with respect to  the Lebesgue measure, is only valid as a heuristic description. In fact, significant technical effort needs to be dedicated to make rigorous sense of the GMC as a random measure (without a well-defined density), a natural path to which is via approximating Gaussian fields for which everything is well-defined and taking limits.
The fact that the density in \cref{eq:intro-eq} does not exist as a well-defined, albeit random, function indicates that as a random measure GMC is indeed almost surely a  \textit{fractal measure}. This can also be demonstrated rigorously, and it can be shown that the GMC a.s. has a fractal dimension $d - \frac{\ga^2}{2}$ (in the case $\mu(\d x)=\d x$). It may be noted that, compared to the \textsl{ambient dimension} $d$, it is this fractal dimension that is more intrinsic to the GMC measure. In~\cref{fig:gmcplot}, we provide surface-plots of discrete approximations to the GMC measure as the parameter $\nu$ varies. Such approximation may be obtained, e.g., by approximating the Gaussian field $X$ (that underlies the GMC) on a fine grid on $\R^d$.

\begin{figure*}[t]
\centering
\includegraphics[scale=0.750]{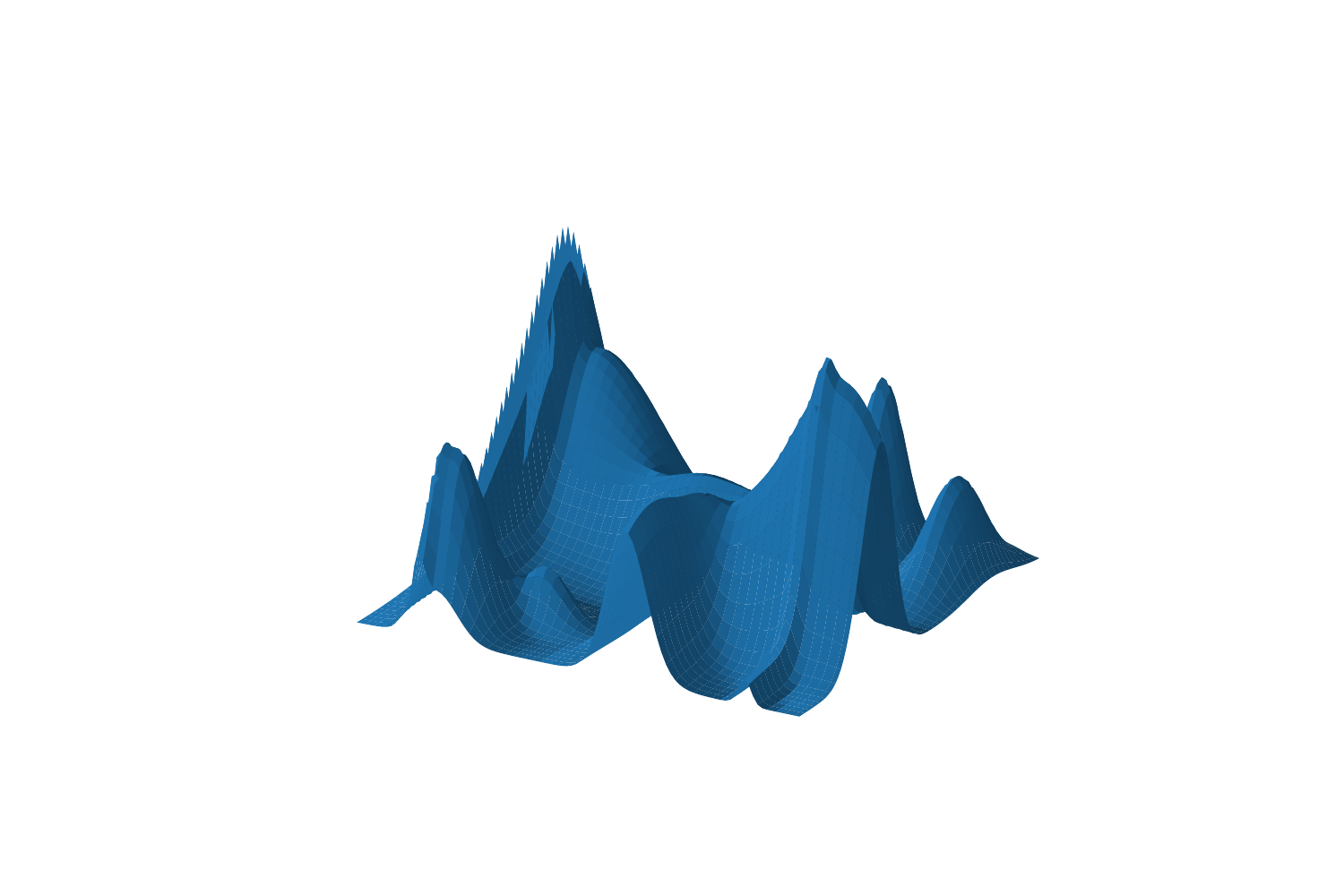}
\includegraphics[scale=0.750]{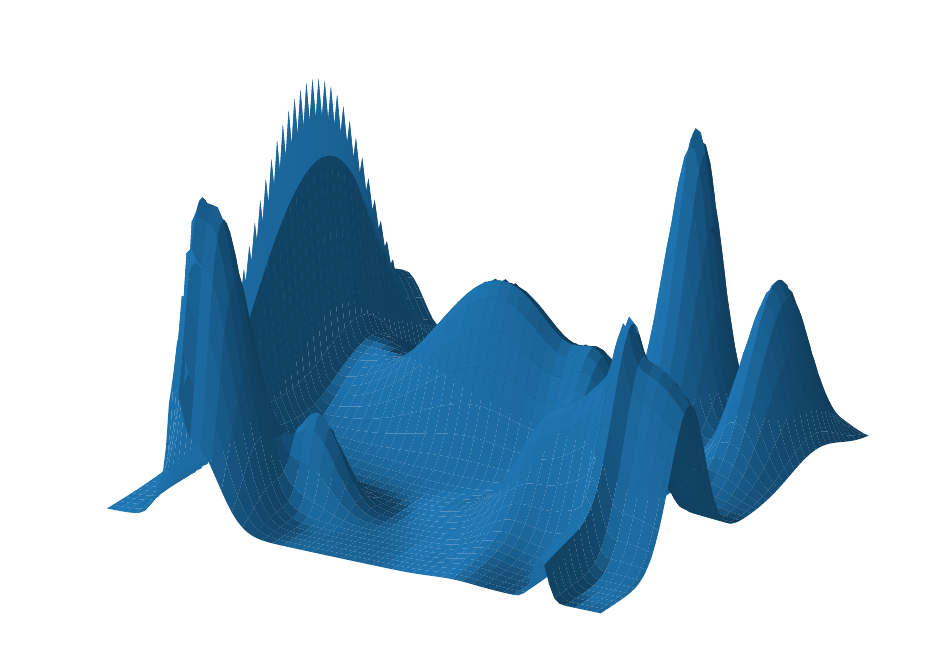}\\
\includegraphics[scale=0.750]{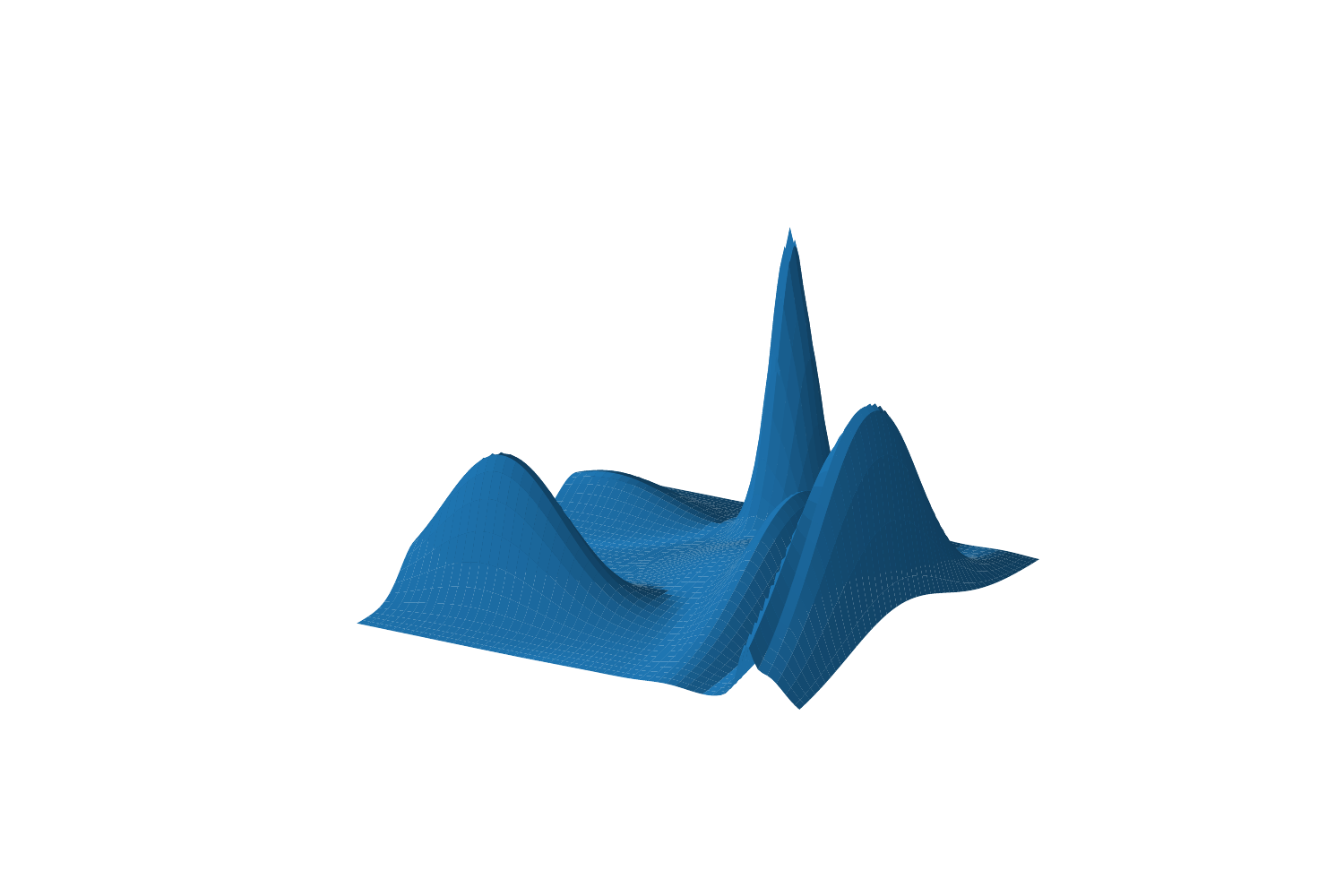}
\includegraphics[scale=0.77]{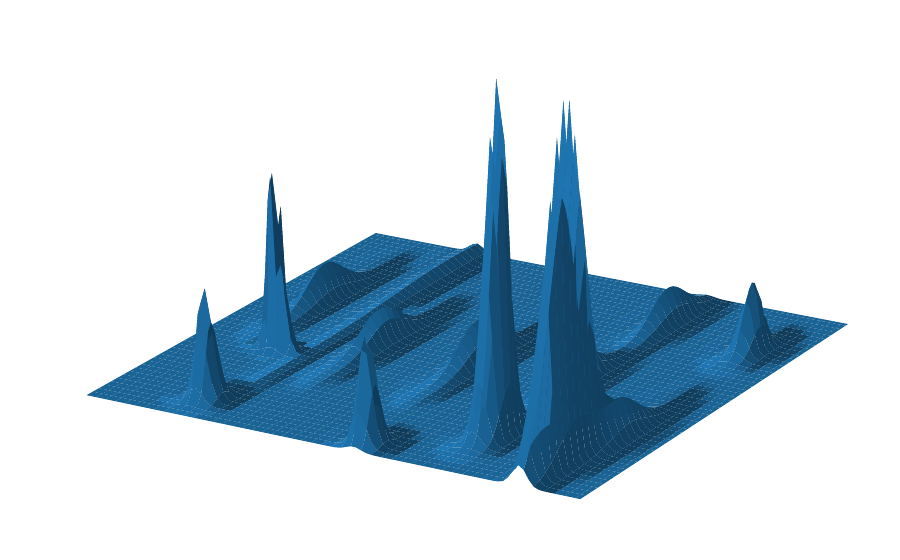}
\caption{Surface plots of discrete approximation to the GMC measure: As we move from top to bottom and from left to right, the value of $\nu$ increases.}
\label{fig:gmcplot}
\end{figure*}

\section{Fractal Gaussian Networks}\label{sec:fgnmodel}
We next proceed to describe the construction of the FGN based on the GMC.
To this end, we will require the following ingredients :
\begin{itemize}[noitemsep]
\item An integer $d>0$, a parameter $\ga>0$ with $\ga^2<2d$, and a domain $\Omega \subset \R^d$ with $\mathrm{Vol}(\Omega)=1$.
\item A centered Gaussian random field $X$ that lives on $\Omega$,  with a logarithmically singular covariance kernel at the diagonal.
\item A realization of the GMC $M^\ga$ on the domain $\Omega$ and based on the random field $X$ .
\item A  \textit{size parameter} $n$, which is a positive integer (to be thought of as large but finite).
\item A \textit{connectivity threshold} $\sigma$ (whose natural size will turn out to be $\propto n^{-1/d}$)
\item  A Poisson random variable $N$ that is distributed with mean $n M^\ga(\Omega)$
\end{itemize}
With the above ingredients in hand, we now proceed to construct the FGN model via the following steps:
\begin{itemize}[noitemsep]
\item Sample $N$-many points, denoted by $V:=\{x_1,\ldots,x_N\}$ at random from the given realization $M^\ga$ of the GMC measure (after normalizing it to have total mass 1). The points in $V$ will form the nodes of the FGN.
\item Connect  each $x_i$ with any other $x_j$ that is within distance $\sigma$ of $x_i$. It turns out that there are multiple ways of implementing such connectivity that, broadly speaking, leads to similar behavior of various network statistics.
\begin{itemize}
\item A direct approach to just connect two points in $V$ if and only if they are within distance $\sigma$ of each other.
\item A refined approach to connect two vertices  $x_i,x_j \in V$ with probability $\propto \exp(-\frac{\|x_i-x_j\|^2}{\sigma^2})$. This allows for the possibility of long range connectivity. In our considerations in this article, for the sake of definiteness we will set the connection probability to be exactly equal to $\exp(-\frac{\|x_i-x_j\|^2}{\sigma^2})$.
\end{itemize}
\end{itemize}
In the last step of constructing the edges, it is the latter, more refined approach of adding edges randomly according to a Gaussian kernel that we will follow for the rest of this paper. However, we note in the passing that we believe the key phenomena will largely be true for the direct approach of connecting vertices merely based on their Euclidean separation. It turns out that $\EE[\mgo]=|\Omega|=1$, therefore $\EE[N]=n\EE[\mgo]=n$, so $n$ is the natural large parameter indexing a growing network size.


\subsection{Single-Pass and Multi-Pass Observation Models}
Our data access model is that we have access to the \textit{combinatorial data} of the graph. In other words, our information will consist merely of a graph with vertices labelled $\{1,\ldots,N\}$ and vertices $i$ and $j$ connected by an edge if and only if the points $x_i$ and $x_j$ are connected in the above geometric graph. Thus, the spatial geometric structure of the GMC is purely a \textsl{latent factor} in the FGN, which we have no direct access to in our statistical investigations. We will explore two different observation models for the FGN. One observation model, which we call the \textit{single-pass observation model} is that we have access to a single realization of the network, in the regime where the network size parameter $n$ is very large. The other observation model, which we call the \textit{multi-pass observation model}, entails that we have access to a moderately large number $m$ of i.i.d. copies of the network, in the regime where the size parameter $n$ is also moderately large.

Both these observation models are well-motivated as modeling paradigms. In particular, for the FGN model, it may be noted that the underlying Gaussian field $X(x)$ is often taken to be translation invariant on $\R^d$. Hence, if two samples of the spatial geometric graph are obtained from two sub-domains of the full space that are translates of each other (i.e., we observe the nodes and edges for points in two domains $\mathcal{D}$ and $\mathcal{D}+x_0$ for some vector $x_0 \in \R^d$), then the subgraphs so obtained are identically distributed (because of the translation invariance of the underlying Gaussian field). On the other hand, if the sub-domains are well-separated in the ambient space, then they can be taken to be approximately independent because of decay of correlations of the Gaussian field. Thus, several approximately independent and identically distributed realizations of the same FGN can be obtained by taking samples of a very large, \textsl{universal} network based on surveying spatially similar and well-separated regions. Since the fractal properties may be reasonably assumed to be similar in different segments of a very large network, this provides us with a way of obtaining multiple samples from a FGN model that can capture fractal structures similar to the original graph. This can be compared, for example, with taking localized snapshots of a different parts of a vast communication network like the internet.


\subsection{Inherent Fractal Structure of the FGN}
\begin{wrapfigure}{r}{0.5\textwidth}
\centering
\includegraphics[scale=0.075]{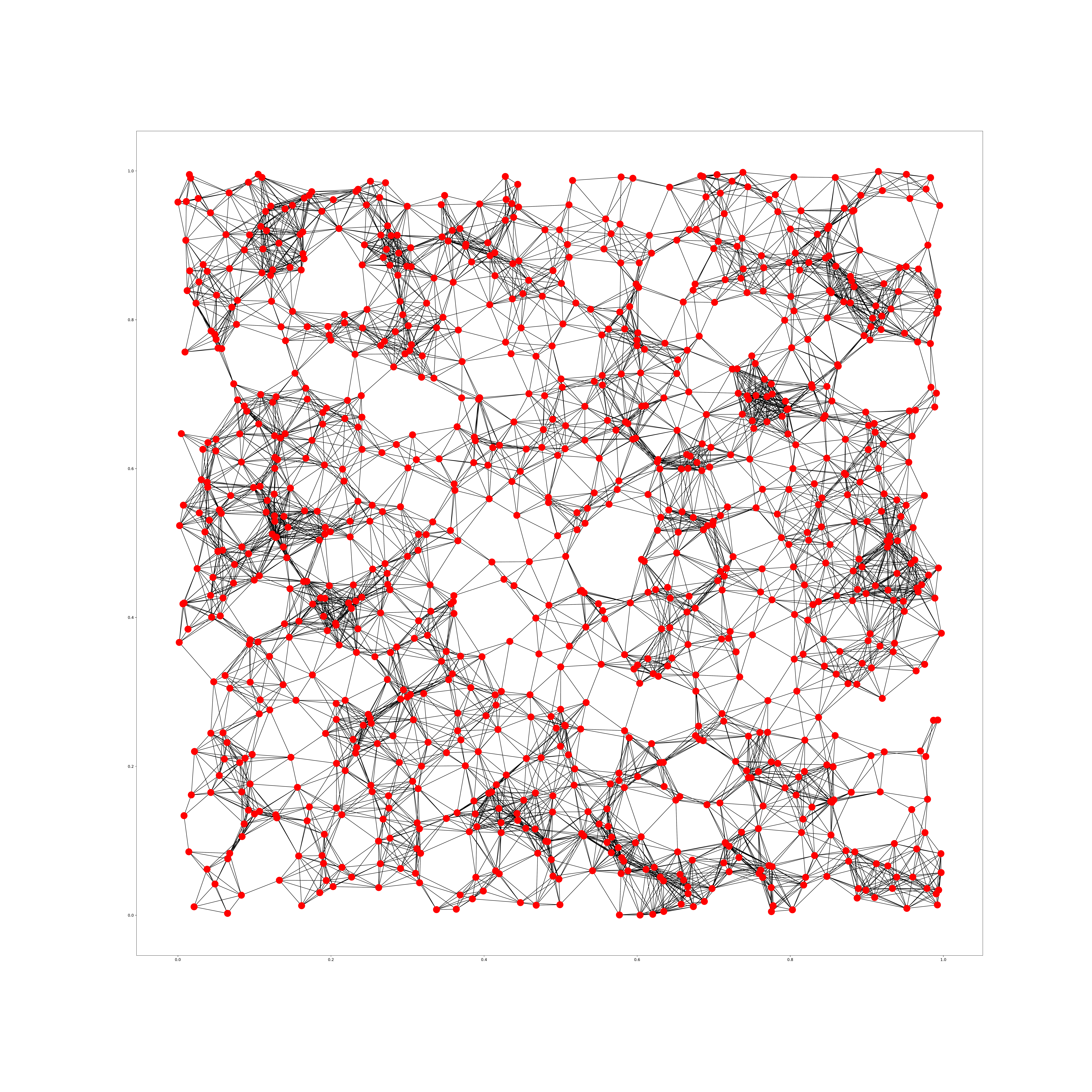}
\caption{A realization of the FGN for the purpose of Illustration.}
\label{figfgnrealization}
\end{wrapfigure}

The inherent fractal nature of a typical realization of the GMC measure induces fractality in the FGN. For instance, one consequence of fractality in terms of the network structure is a large measure of heterogeneity, often manifested in terms of the irregular distribution of nodes in the form of dense clusters and rarefied neighborhoods in the graph. The GMC is characterized by regions of high concentration of measure, interspersed with regions of low mass distribution. To see this in more detail, we refer the reader to Figure 1 in~\cite{rhodes2014gaussian} and~\cref{fig:gmcplot}. In fact, a progressive increase in the irregularity of the GMC can be observed as the parameter $\nu$ increases.   The  FGN, because of its latent spatial geometry being derived from the GMC, also inherits these heterogeneities in its graphical structure, characterized by certain vertex clusters of high connectivity interspersed with sparsely connected vertices (see, e.g.,~\cref{figfgnrealization} for an illustration), with such heterogeneous effects increasing in intensity as the parameter $\nu$ increases in value.

It may be observed that, once the realization of the GMC measure is in our hands, the rest of the construction of the FGN is spatial geometric in nature, and can actually be carried out for any non-negative measure on the domain $\Omega$ - random or otherwise. This spatial geometric construction employs the commonly used technique for the construction of random geometric graphs (RGG, c.f. \cite{penrose2003random}, \cite{gilbert1961random}), popularly considered in the setting of the uniform distribution on $\Omega$ (which is going to be our ``pure noise'' case and the point of comparison with the FGN regarding the presence of fractal structures). At this point, a word is in order regarding the spatiality inherent in the construction of the FGN. It turns out that many natural applications  of stochastic networks have spatiality built into their construction - mobility networks, transportation networks or drainage networks are all examples of this phenomenon. But even more generally, our construction of the FGN does not necessitate the ambient Euclidean space $\R^d$ to correspond to our application in a \textit{physical sense}. In fact, the ambient space $\R^d$ can be taken to be the \textit{feature space} obtained from a feature mapping of the nodes, whose specifics can be completely problem-dependent. This is exemplified by its applications in the social networks, where the feature mapping corresponding to a person corresponds to his/her interests, and two persons are connected in the social network if their interests (i.e., feature vectors) are close in the metric of the latent \textit{social space} (c.f., \cite{jackson2010social}, \cite{racz2017basic}, \cite{sarkar2006dynamic}, \cite{grover2016node2vec}).

Such graphs are of interest as statistical networks in both low and high dimensional spatial settings (see, e.g., \cite{bubeck2016testing}, \cite{bubeck2015influence}, \cite{mossel2018proof}).  \textsl{Physical spatiality} would often correspond to a  \textsl{low} ambient dimension (as in the case of transportation or drainage networks), whereas \textsl{latent spatiality} in the social/feature space may naturally correspond to a  \textsl{relatively high} ambient dimension $d$. It may be pointed out that the FGN model encompasses both low and high dimensions of the latent space, thereby catering to both types of spatial structure.

\subsection{The connectivity threshold, Locality and the Sparse Regime}
In this section, we determine the \textit{right regime} of the connectivity threshold $\sigma$. In doing so, our guiding principle would be to obtain a sparse random graph model in the end, one in which the number of neighbors from the FGN of a given point in the latent social space is typically $O(1)$. This is most natural in the context of most real world networks - even though the total network might be huge and highly complex, seen from the viewpoint of a particular node it has a  \textsl{finite local neighborhood}, which does not scale with the growing size of the network (see, e.g., \cite{johnson1977efficient}, \cite{krzakala2013spectral}, \cite{guedon2016community}, \cite{batagelj2001subquadratic} and the references therein). We show that, for any value of $\ga$, the normalization $\sigma =\frac{1}{\sqrt{\pi}} \rho^{1/d}n^{-1/d}$ will lead to, in expectation, $\rho$ neighbors for a given point under our connection model.

\begin{theorem}
In the FGN model with size parameter $n$, setting the threshold parameter $$\sigma= \frac{1}{\sqrt{\pi}}\rho^{1/d}n^{-1/d},$$ one has that the expected number of neighbors of a given point is asymptotically $\rho\in(0,\infty)$. \label{thm:threshold}
\end{theorem}
\begin{proof}
Consider the FGN with $N \sim \mathrm{Poi}(n \mgo)$, nodes $\{x_1,\ldots,x_N\}$ and threshold $\sigma$. Fix a deterministic point $x_0\in \Omega$. We use the notation $x \sim y$ to denote that  the point $x$ is connected to the point $y$ by an edge.  Observe that, under our connection model for edge formation (once we are given some nodes), $\PP [x_0 \sim x_i]= e^{-\frac{\|x_i-x_0\|^2}{\sigma^2}}$, and the total number of points to which $x_0$ may be connected to in this manner is $\left(\sum_{i=1}^N \mathbbm{1}_{x_0 \sim x_i}\right)$. Therefore, in the regime of small connection threshold $\sigma$, since
$$\int_{\RR^d} e^{-\|x\|^2}\d x = \pi^{\frac d 2},$$
the expected number of points in this FGN that would be connected of $x_0$ is :

\begin{align*}
\EE \left[ \sum_{i=1}^N \mathbbm{1}_{x_0 \sim x_i} \right]
= &~\EE \left[ \EE \left[ \sum_{i=1}^N \mathbbm{1}_{x_0 \sim x_i}  \big| \text{FGN}  \right] \right] \\
= &~\EE \left[ \sum_{i=1}^N e^{-\frac{\|x_i-x_0\|^2}{\sigma^2}}\right] \\
= &~\EE \left[ \EE \left[ \sum_{i=1}^N e^{-\frac{\|x_i-x_0\|^2}{\sigma^2}} \big| N , \d M^\ga \right] \right] \\
= &~\EE \left[ N \cdot \int_\Omega e^{-\frac{\|x-x_0\|^2}{\sigma^2}} \frac{M^\ga(\d x)}{\mgo} \right]  \\
= &~\EE \left[ \EE \left[  N \cdot \int_\Omega e^{-\frac{\|x-x_0\|^2}{\sigma^2}} \frac{M^\ga(\d x)}{\mgo}  \big| \d M^\ga \right] \right] \\
= &~\EE \left[ \EE \left[ N \big| \d M^\ga  \right] \cdot   \int_\Omega e^{-\frac{\|x-x_0\|^2}{\sigma^2}} \frac{M^\ga(\d x)}{\mgo} \right] \\
= &~n \cdot \EE \left[\int_\Omega e^{-\frac{\|x-x_0\|^2}{\sigma^2}} M^\ga(\d x)\right]  \\
= &~n \cdot \left[\int_\Omega e^{-\frac{\|x-x_0\|^2}{\sigma^2}} \d x \right]   \\
= &~n\sigma^d \cdot \int_{\Omega/\sigma} e^{-\|x- \frac{x_0}{\sigma}\|^2}  \d x\\
= &~n(\sqrt{\pi}\sigma)^d (1+o(1)),
\end{align*}
where, $\Omega/\sigma=[-1/2\sigma,1/2\sigma]^d$, the fourth equality follows since the $\{x_i\}_{i=1}^N$ are i.i.d. $\d M^{\ga}$ given $N, \d M^\ga$, the seventh equality follows since $\EE[N \big| \d M^\ga ]=n \mgo$, and, the eighth inequality follows since the expected measure $\EE[\d M^\ga(x)]$ is Lebesgue.
\end{proof}

We call $\rho$ the \textit{density parameter}  of the FGN. {From a statistical point of view, the density parameter $\rho$ may be learnt from the number of neighbors of vertices. This is motivated by Theorem 3.1, which shows that the expected number of neighbors of a given point is asymptotically $\rho$ (as $n \to \infty$). This renders the density $\rho$ a \textsl{local parameter} in the FGN model.}
Local parameters are much easier to investigate because they can be learnt by sampling small local neighborhoods, which for practical purposes can be taken to be approximately independent if they are well separated (e.g., in the graph distance). On the other hand, in real world networks, fractality is often observed at the scale where one \textsl{zooms out}, i.e., at mesoscopic scales or higher (c.f., \cite{franovic2009percolation}, \cite{pook1991multifractality}, \cite{daqing2011dimension}). This necessitates the investigation of fractality to be contingent on more global aspects of the FGN, which makes it much more challenging but at the same time more interesting to study and is the principal focus of this article.

\textbf{The intrinsic \textit{fractality parameter} $\nu$.} For the FGN model, the key determinant of fundamental network statistics turns out be the quantity $\nu=\frac{\ga^2}{d}$, which we refer to as the \textsl{fractality parameter}. 
Accordingly, we will maintain a particular consideration for the \textit{fractality parameter} $\nu$ in our statistical analysis of the FGN model.

\subsubsection{Clustering Coefficient}
In this section, we examine the clustering coefficient of FGNs as a function of the fractality parameter $\nu$. Network-average clustering coefficient, proposed by~\cite{watts1998collective}, is a well-motivated heuristic to measure how much the nodes of a graph tend to cluster together and has been widely used to characterize the properties real-world networks. It is defined as follows: For a graph
with nodes $V=\{v_1,\ldots, v_n\}$ and with $e_{jk} \in \{ 0,1\}$, $1\leq j,k \leq n$, representing the edges between the nodes, let the set $N_i \subset V $ denote the immediate neighbors of the node $v_i$. Then, note that if $|N_i| = k_i$, $k_i(k_i-1)/2$ edges could potentially exists among the nodes in the set $N_i$. The local clustering coefficient (for $1\leq i\leq n$) and the network-average clustering coefficient are defined respectively as $$C_i = \frac{2 |\{e_{jk}: v_j,v_k \in N_i, e_{j,k}=1 \}|}{k_i (k_i-1)},\qquad~\text{and}~ \qquad \bar{C} = \frac{1}{n}\sum_{i=1}^n C_i.$$

\begin{wrapfigure}{r}{0.5\textwidth}
\centering
\includegraphics[scale=0.5]{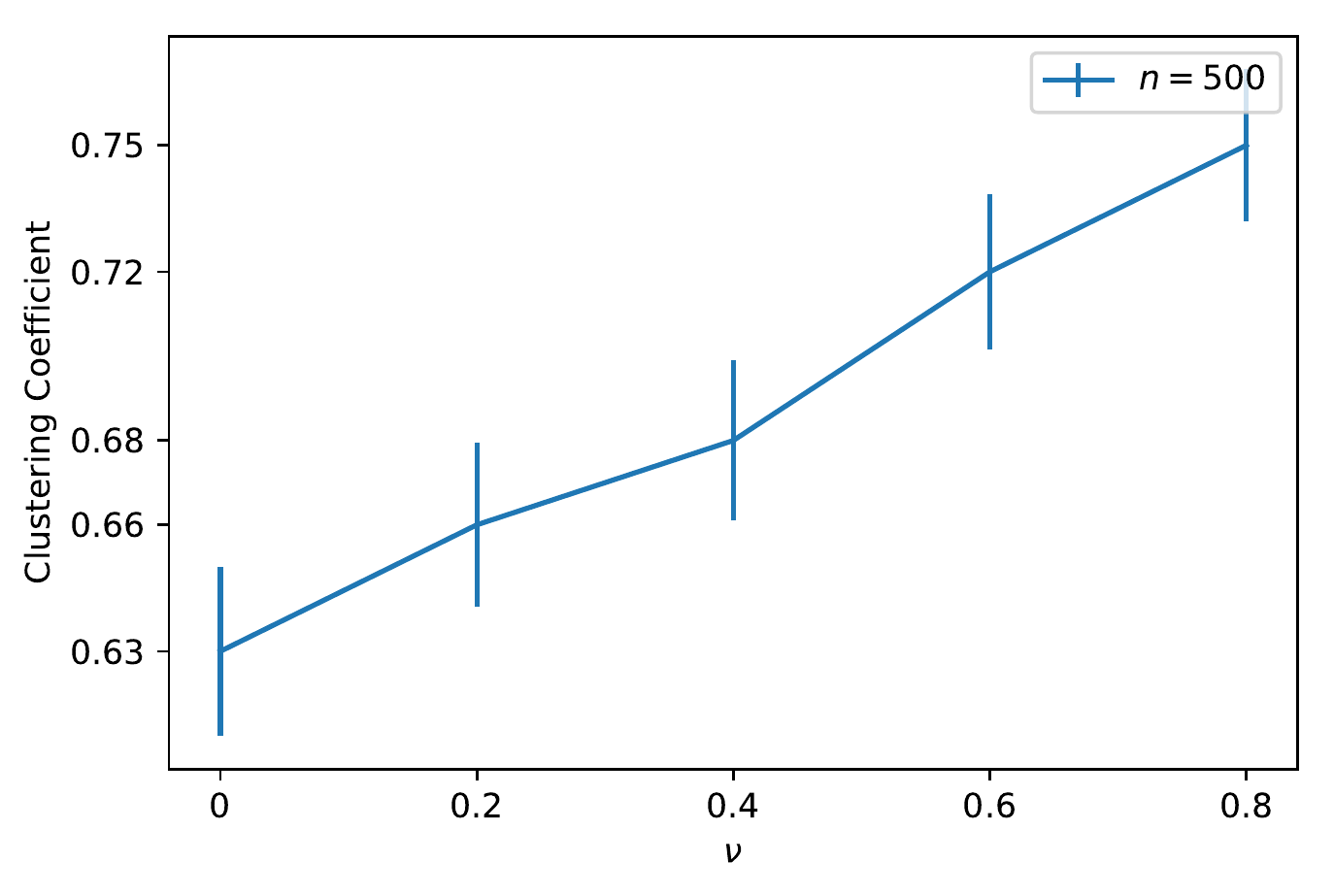}
\caption{Network-average clustering coefficient as a function of the fractality parameter $\nu$. The line represents the average over 1000 instances and the bars represent the standard deviations.}
\label{fig:ClusteringCoeff}
\end{wrapfigure}

For our examination, for each fixed value of $\nu$, we generated 1000 instantiations of FGNs with $n=500$ and calculated the average (across the instantiations) of the network-average clustering coefficient $\bar{C}$. \cref{fig:ClusteringCoeff} shows the observed results for various values of $\nu$. We see that as the fractality parameter increases, the clustering coefficient increases, thereby empirically confirming our model property.

\subsubsection{Degree distribution: Interpolating Poisson and power laws}
\begin{figure*}[t]
\centering
\includegraphics[scale=0.35]{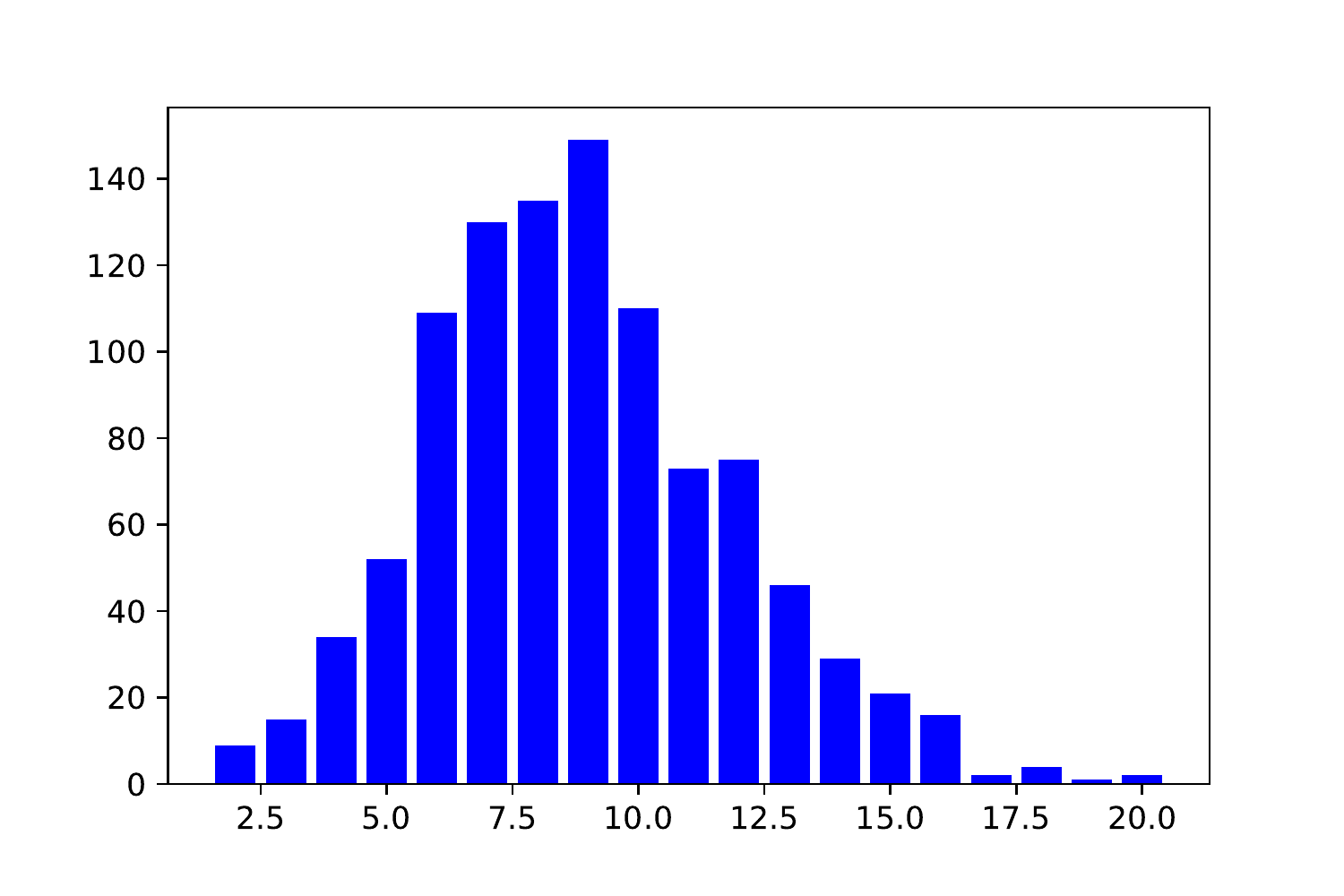}
\includegraphics[scale=0.35]{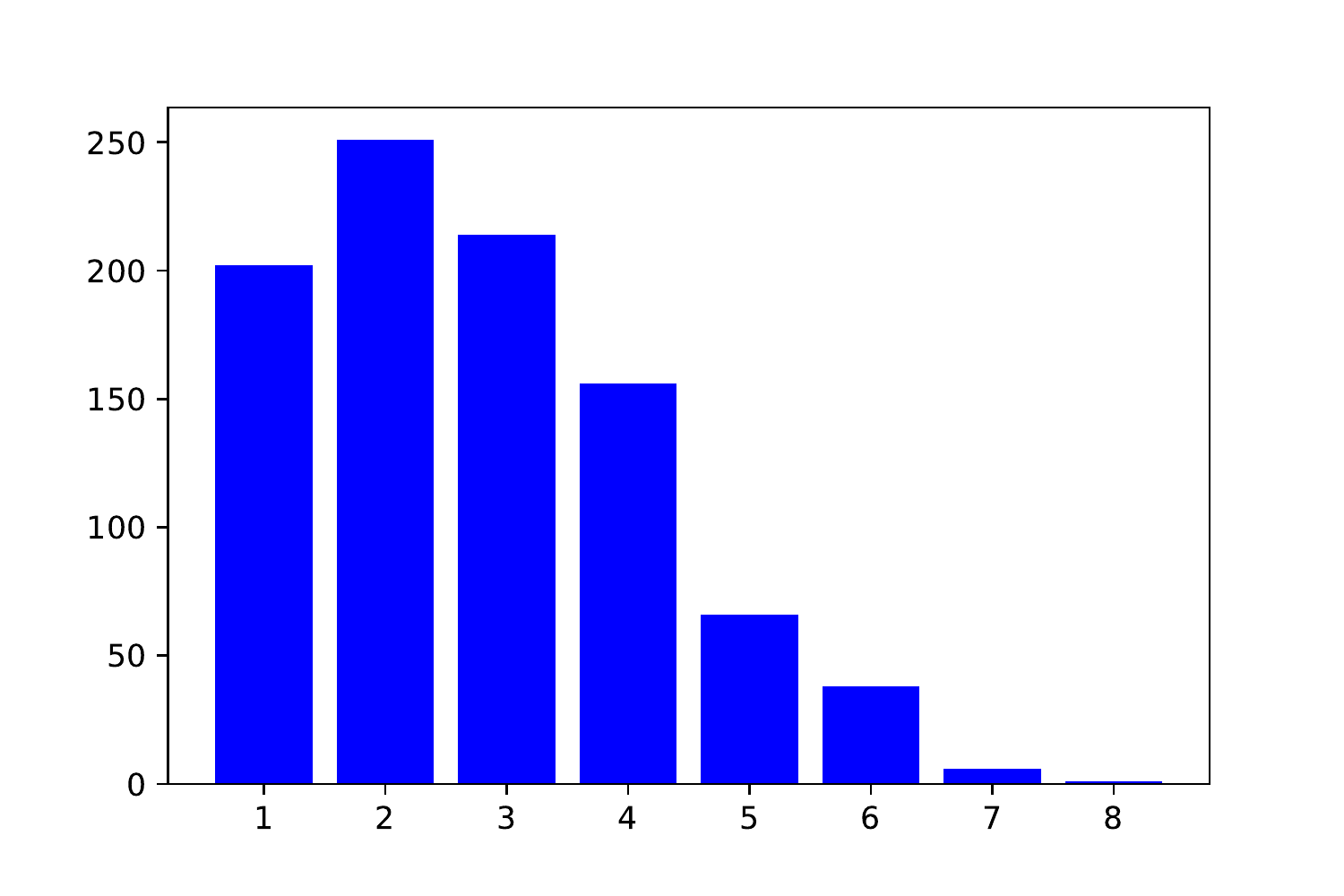}
\includegraphics[scale=0.35]{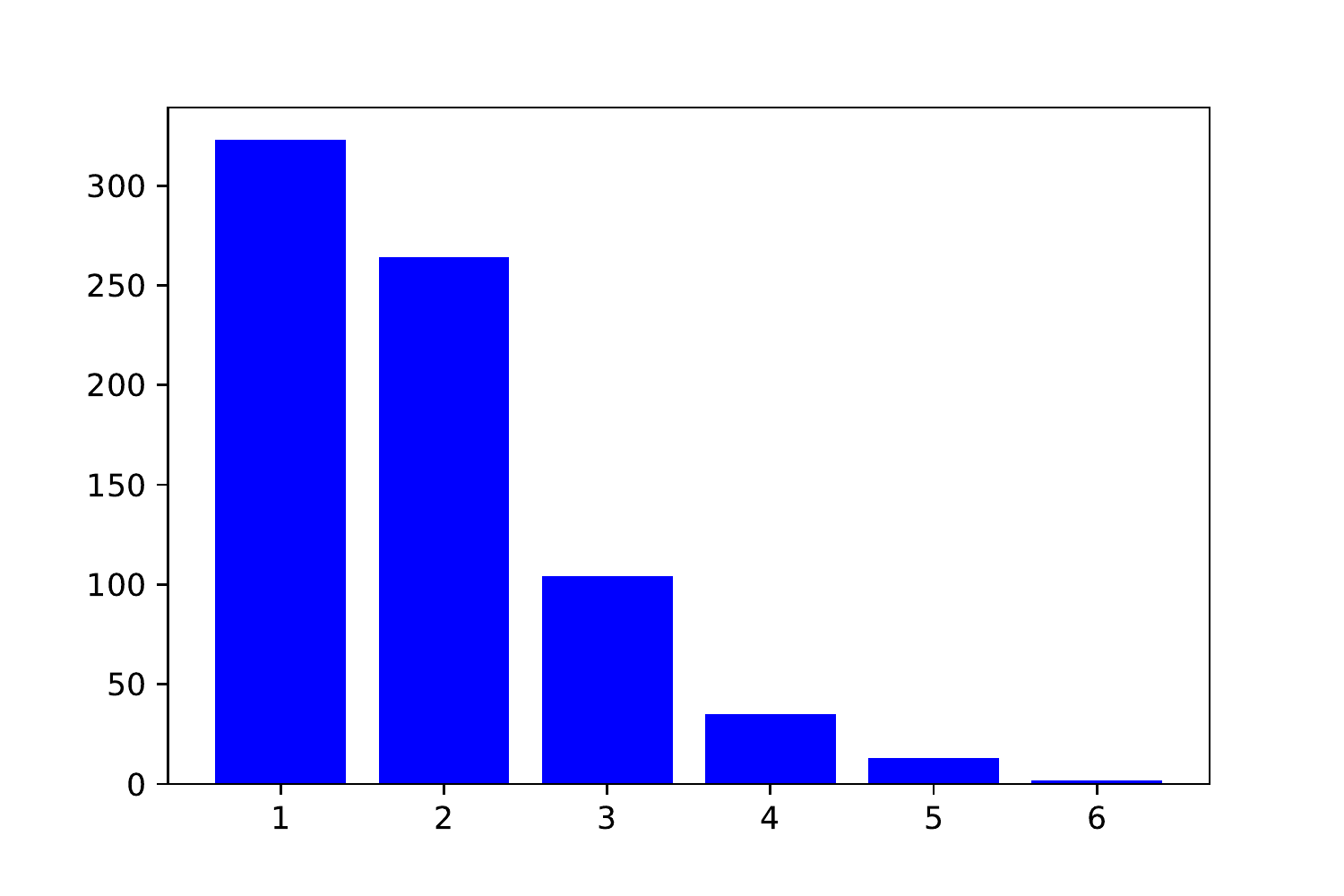}
\caption{Degree distribution of non-isolated nodes of the FGN: the value of $\nu$ increases from left to right.}\label{fig:degreedist}
\end{figure*}
We now investigate the degree distribution of the FGN model empirically (c.f.~\cref{fig:degreedist}). We observe that, for small $\nu$, the degree distribution is Poissonian, whereas with increasing values of the parameter $\nu$, it deforms into a truncated power law like distribution. It may be noted that power laws and truncated power laws are ubiquitous in many real-world networks (c.f. \cite{BA} and the references therein), whereas Poissonian behavior is a hallmark of classical \textsl{mean-field} models (like the Erd\H{o}s-R\'enyi random graphs, c.f. \cite{van2016random}). As parametric statistical model, the FGN continuously interpolates between these two very different worlds
which are  the two major paradigms  the distributions of degrees in networks.

\subsection{Stochastic Block Models in the FGN paradigm}
\label{sec:SBM}
Stochastic Block Models (henceforth abbreviated as SBM) has become an important paradigm for understanding and investigation community structures in networks, social or otherwise. A long series of ground breaking results in this regard have been achieved in recent years; we refer the interested reader to (\cite{holland1983stochastic},  \cite{abbe2017community}, \cite{abbe2015community},   \cite{racz2017basic}, \cite{mukherjee2018some}) for a partial overview of this vast and rapidly evolving field of research. However, a majority of the models and the results are based on block models constructed out of Erd\H{o}s-R\'enyi  random graphs; a few exceptions include~\cite{abbe2019community, berthet2019exact, meylahn2020two}.

In the context of networks with fractal structures, it is natural to envisage a situation where there are multiple distinct communities in the network with potentially different fractal structures. It is also natural to posit that the communities have differing degrees of affinity to connect within each other as compared to connections across community boundaries, which might be rarer. We encapsulate this idea in the form of a natural SBM structure in the context of the FGN model. We need the following ingredients:
\begin{itemize}[noitemsep]
\item Two independent GMC-s $M^{\ga_1}, M^{\ga_2}$ corresponding to (possibly different) positive parameters $\ga_1,\ga_2$ respectively on the same domain $\Omega \subset \R^d$.
\item Two different positive threshold parameters $\sigma_\inn$ and $\sigma_\out$.
\item A size parameter $n \in \mathbb{N}$ and two independent Poisson random variables $N_1 \sim \mathrm{Poi}(n M^{\ga_1}(\Omega))$ and $N_2 \sim \mathrm{Poi}(n M^{\ga_2}(\Omega))$.
\end{itemize}

Given these ingredients, we construct the SBM on the FGN model as follows.
\begin{itemize}[noitemsep]
\item We generate $N_1$ points $\{x_1,\ldots,x_{N_1}\}$ i.i.d. from the (normalized) measure $M^{\ga_1}$ and $N_2$ points $\{y_1,\ldots,y_{N_2}\}$ i.i.d. from the (normalized) measure $M^{\ga_2}$.
\item For each pair of points $x_i,x_j$, we connect them with an edge with probability $\propto \exp(-\frac{\|x_i-x_j\|^2}{\sigma_{\inn}^2})$. Likewise, For each pair of points $y_i,y_j$, we connect them with an edge with probability $\propto \exp(-\frac{\|y_i-y_j\|^2}{\sigma_{\inn}^2})$. These are the \textit{intra-community links}.
\item For each pair of points $x_i,y_j$, we connect them with an edge with probability $\propto \exp(-\frac{\|x_i-y_j\|^2}{\sigma_{\out}^2})$. These are the \textit{inter-community links}.
\end{itemize}
We then forget the spatial identities of the points, and consider the resulting combinatorial graph $\G$, whose node set is the union of the node sets of the FGN-s $\G_1$ and $\G_2$, and whose edges are those of $\G_1 \cup \G_2$ along with the cross-community edges defined in the last step. This forms a natural SBM structure in the context of the FGN model. A natural statistical question in this context would be to understand separation thresholds between in the intra-community connection radius $\sigma_\inn$ and the inter-community connection radius $\sigma_\out$ which allow for detection of the different communities with reasonable accuracy and probabilistic guarantees, as the network size parameter $n \to \infty$. We leave this and related questions for future investigation.

\section{FGN as a Parametric Statistical Model}\label{sec:mainresults}
\textbf{Interpolating  homogeneity and fractality.} We will formulate our analysis of the FGN as a statistical model of network data in terms of the quantity $\nu$, a choice that is well-motivated by the discussions in the preceding sections. It may be noted that, when $\nu=0$ (equivalently, $\ga=0$), the GMC reduces to the Lebesgue measure, and we have a usual Poisson random geometric graph, which we will consider as the \textit{pure noise} case in our setting. We will compare this against the presence of fractality in the network, a situation which would correspond to $\nu>0$. Thus, the FGN model \textit{interpolates continuously} between Poisson random geometric graphs and networks with increasing degree of fractality, as the value of the parameter $\nu$ increases from 0. On a related note, it would also be of interest to learn the value of fractality parameter $\nu$ in its own right, which would correspond naturally to the problem of parameter estimation in the FGN model. In our investigation of the above standard statistical questions on the FGN model, such as parameter estimation and testing (undertaken in \cref{sec:estimation} and \cref{sec:testing} respectively), we will make extensive use of the statistics of small subgraph counts (in particular the edge counts). This is well-motivated by the effectiveness of small subgraph counts as statistical observables in the study of usual spatial network models (see, e.g., \cite{racz2017basic}, \cite{bubeck2016testing}, and the references therein).

\subsection{Inferring the size parameter}\label{sec:sizepar}
It may be noted that the  parameter $n$ driving the network size is not given to us in the combinatorial data that we can access. In some sense, it is also a latent parameter of the model that we do not directly focus on in our study of the fractal properties of the network.  However, statistical procedures typically utilize large sample effects, and for that purpose, it becomes imperative to develop an idea of the underlying size parameter $n$ from the combinatorial graph.

To this end, we observe that, conditioned on the GMC, the network size $N$ is a Poisson random variable with mean $n \mgo$. As such, if $\{Y_i\}_{i \ge 0}$ are i.i.d. Poisson random variables with  mean $\mgo$, then $N=\sum_{i=1}^n Y_i$ in distribution. Consequently, for a given realization of the GMC, the quantity $\frac{N}{n}=\frac{1}{n}\sum_{i=1}^n Y_i \to \mgo$ a.s. as $n \to \infty$. As a result, $\frac{N}{n}=O(1)$ with high probability, as the size parameter $n \to \infty$. We record this as the following theorem, complete with a concentration bound that is faster than any polynomial rate.
\begin{theorem}
\label{thm:sizepar}
For an FGN model with size parameter $n$, the number of nodes $N$ satisfies the following: for any $p>0$, there exists $c=c(p)$ such that
\[ \PP\left[\frac{N}{n} \ge 2M_\ga(\Omega)\right]  \le c n^{-p}.\]
\end{theorem}
\begin{proof}
Recall that a Poisson random variable $Y$ with parameter $\al$ has moment generating function $\EE[e^{tY}]=e^{\al(e^t-1)}$ for all $t\in \RR$. Hence by Markov inequality, we have $\PP[Y\ge 2\al] \le e^{-2\al t + \al(e^t-1)}$, yielding $\PP[Y\ge 2\al]\le e^{-0.38 \al}$ by choosing $t=\ln(2)$. Recall that $N$ conditioning on $\mgo$ has Poisson distribution with parameter $n\mgo$, thus the exponential tail bound for Poisson distribution implies the following upper tail estimate
\begin{align*}
\PP[N\ge  2n\mgo] \le \EE[e^{-0.38n\mgo}].
\end{align*}
To evaluate the expectation, we consider two cases: first on the event $\{\mgo\le n^{-\be}\}$ with $\be\in (0,1)$, we use the fact that $\mgo$ has negative moments of all orders \cite[Theorem 2.11]{rhodes2014gaussian} to deduce that for any $p>0$, there exists $c=c(p)=\EE[\mgo^{-p}]<\infty$, such that
\begin{align*}
\EE[e^{-0.38n\mgo} \mathbbm{1}(\mgo\le n^{-\be})] \le \PP[\mgo\le n^{-\be}]\le cn^{-\beta p}
\end{align*}
where we used Markov inequality in the last step. On the other hand,
\begin{align*}
\EE[e^{-0.38n\mgo}\mathbbm{1}(\mgo > n^{-\be})] \le e^{-0.38 n^{1-\be}}.
\end{align*}
The desired conclusion follows immediately from the two estimates.
\end{proof}

\begin{remark} [\textcolor{black}{Improving Concentration}]\label{rem:sizepar}
It is of interest to consider strengthening of the concentration bound in Theorem \ref{thm:sizepar}, e.g. to an exponential decay in $n$. In the current state of the art, this seems to be beyond the scope of the available understanding of the GMC measure. Herein, we outline the main issues to this end, and indicate the additional ingredients that would be necessary from the theory of the GMC to obtain such a result.

Set $c(p): p\mapsto E[M^{-p}]$. A potential approach to improving the bound in Theorem \ref{thm:sizepar} would be to use a detailed understanding of $c(p)$ as a function of $p$ in order to optimize over $p$ in the existing bound. However, to the best of our knowledge,  no quantitative understanding is available in the literature about the growth of $c(p)$ (as a function of $p$); existing results merely establish that $c(p)<\infty$ for any $p>0$. A sharp estimate of this would improve greatly the knowledge of the lower tail near 0 of the total mass of the GMC measure. This is lacking in the theory of GMC.

Assuming that $c(\log n)) \ll  e^{(\log n)^2}$, then the one but last equation in the proof of Theorem 4.1 is at most (by setting $p=\log n$)
$$\exp( \log c(\log n) - \beta(\log n)^2) ) \le \exp( -  (\beta/2) (\log n)^2)),$$
which decays super-polynomially (compared to the polynomial decay in Theorem \ref{thm:sizepar}). The availability of such refined estimated (as above on $c(\log n)$) would therefore pave the way to improved concentration behavior of the network size $N$ in terms of the size parameter $n$.
\end{remark}

In view of Theorem \ref{thm:sizepar}, for the single-pass observation model, in the regime of large size parameter $n$ (which is the regime in which we envisage the FGN model) we may justifiably employ the network size  $N$ as an estimator for the latent size parameter $n$.  In particular, on the logarithmic scale we may deduce that
\begin{equation} \label{eq:N-n}
\log N = \log n (1+o_P(1)),
\end{equation}
which is a form that will be particularly useful in our later analysis. In the multi-pass observation model, where we have $m$ i.i.d. realizations of the FGN with node counts $N_1,N_2, \ldots,N_m$, we will use $\ol{N}=\frac{1}{m}\sum_{i=1}^m N_i$, which will strongly concentrate around its expectation $n$.

\subsection{Edge Counts}\label{sec:EC}

In this section, we investigate the statistic of edge counts in the FGN model. To this end, consider the FGN with $N \sim \mathrm{Poi}(n \mgo)$, nodes $\{x_1,\ldots,x_N\}$ and threshold $\sigma$. Let $\D$ denote the number of edges in this FGN. We will work in the setting $\ga^2<d$ (equivalently, $\nu<1$) for our analysis. Interestingly, this corresponds to the so-called $L^2$ regime of the GMC, where the model is believed to be technically more tractable in relative terms. We believe that similar results would be true for the full range of validity of the GMC and the FGN model (i.e., all the way up to $\nu<2$), albeit technically more challenging. We leave this as an interesting direction for future study.

In~\cref{fig:edgecounts} (top left), we empirically plot the expected edge count as a function of the number of nodes. Specifically, we consider various values of $\nu$ and $n$ and compute the average number of edges in 100 realizations of FGNs. We also empirically illustrate the distribution of the edge count by means of a histogram, for various values of $\nu$ in~\cref{fig:edgecounts}. From the histogram, we observe that the edge count is concentrated around its expected value, albeit with possibly heavy tails.  Based on the above empirical observations, we provide an analytical characterization of the expected edge count in~\cref{thm:EC}. Our result below is asymptotic but agrees with the finite-sample results obtained top left plot of~\cref{fig:edgecounts}.

\begin{theorem}
\label{thm:EC}
For an FGN model with size parameter $n$, density parameter $\rho$ and fractality parameter $\nu=\ga^2/d<1$, under Assumption~\ref{assumption:main}, the expected edge count satisfies
\[ \EE [ \D ] =  C(\ga,d)  \rho^{1-\nu} n^{1+\nu}  \left( 1+ o(1) \right), \]
as $n \to \infty$, where  $$C(\ga,d)=\frac{1}{2}{\pi^{\frac{1}{2}d - \frac{1}{2}\ga^2}} \cdot \left( \int_{\R^d} \frac{1}{\|x\|^{\ga^2}} e^{-\|x\|^2} \d x \right).$$
\end{theorem}
\begin{proof}
In the computations that follow, we will use the fact that if $\Lambda \sim \mathrm{Poi}(\lambda)$, then the second \textsl{factorial moment} of $\Lambda$ is given by the relation $\EE [{\Lambda \choose 2}]=\frac{\lambda^2}{2!}$ (c.f. \cite{haight1967handbook}). Consequently, recalling that given the GMC $\d \mg$ the node count $N \sim \mathrm{Poi}(n\mgo)$, we may deduce that $\EE [{N \choose 2}|\mgo] =\frac{1}{2} \cdot n^2\mgo^2 $. In view of this, we may proceed as
\begin{align*}
 \EE [\D] =& \EE \left[ \EE [ \D | \text{FGN} ] \right] \\
= & \EE \left[ \EE \left[ \sum_{1 \le  i< j \le N} \mathbbm{1}_{x_i \sim x_j}  \big| \text{FGN} \right] \right]\\
= & \EE \left[ \sum_{1 \le  i < j \le N} e^{ -  \|x_i-x_j\|^2/\sigma^2} \right] \\
= & \EE \left[ \EE \left[\sum_{1\le i<j\le N} e^{ -  \|x_i-x_j\|^2/\sigma^2}  \big| N , \d M^\ga \right] \right] \\
= &\EE \left[ {N \choose 2} \cdot \iint_{ \Omega \times \Omega} e^{-\|x-y\|^2/\sigma^2} \frac{M^\ga(\d x)M^\ga(\d y)}{\mgo^2}\right]  \\
= &\EE \left[ \EE \left[  {N \choose 2} \cdot \iint_{ \Omega \times \Omega} e^{-\|x-y\|^2/\sigma^2} \frac{M^\ga(\d x)M^\ga(\d y)}{\mgo^2} \big| \d \mg \right]  \right]  \\
= &\EE \left[ \EE \left[  {N \choose 2} \big| \d \mg \right]  \cdot \iint_{ \Omega \times \Omega} e^{-\|x-y\|^2/\sigma^2} \frac{M^\ga(\d x)M^\ga(\d y)}{\mgo^2}  \right] \\
= &\EE \left[ \frac{1}{2}n^2\mgo^2 \cdot \iint_{ \Omega \times \Omega} e^{-\|x-y\|^2/\sigma^2} \frac{M^\ga(\d x)M^\ga(\d y)}{\mgo^2}  \right] \\
= & \frac{n^2}{2}\EE \left[ \iint_{\Omega \times \Omega} e^{-\|x-y\|^2/\sigma^2}  M^\ga(\d x) M^\ga(\d y)\right],
\end{align*}
where the fifth equality follows since the $\{x_i\}_{i=1}^N$ are i.i.d. $\d M^{\ga}$ given $N, \d M^\ga$. For further analysis, we consider
\[ I := \EE \left[  \iint_{\Omega^2} \exp \left( - {\frac{\|x-y\|^2}{\sigma^2}} \right)M^\ga(\d x)M^\ga(\d y) \right] .  \]
and
\begin{align*}
 I_t := \EE\Bigg[ \iint_{\Omega^2} \exp\Bigg\{-\frac{\|x-y\|^2}{\sigma^2} +  \ga(X_t(x)+X_t(y) ) - \ga^2 t \Bigg\}  \d x \d y  \Bigg] .
 \end{align*}
In view of the convergence~\cref{eq:approx}, we will use $I_t$ as an approximation for $I$ as $t \to \infty$. Using the fact that for fixed $t$ the field $\{X_t(x)\}$ is a centered Gaussian random field with covariance structure as in~\cref{eq:pre-kernel}, we may then proceed further as
\begin{align*}
\EE [\D] = &\frac{n^2}{2}\lim_{t\to\infty} \EE \Bigg[ \iint_{\Omega\times \Omega} \exp\Bigg\{ -  \frac{\|x-y\|^2}{\sigma^2} +\ga(X_t(x) +X_t(y))-\ga^2 t \Bigg\} \d x\d y \Bigg] \\
= &\frac{n^2}{2}\lim_{t\to\infty} \iint_{\Omega\times \Omega} \exp\Bigg\{ -  \frac{\|x-y\|^2}{\sigma^2} -\ga^2 t \Bigg\} \cdot\EE\left[ \exp\left( \ga(X_t(x)+X_t(y)) \right)  \right]  \d x\d y  \\
= &\frac{n^2}{2}\lim_{t\to\infty}  \iint_{\Omega\times \Omega} \exp\left\{ -  \frac{\|x-y\|^2}{\sigma^2} -\ga^2 t\right\}  \cdot  \exp\left\{ \frac{\ga^2}{2}(2t+ 2\int_1^{e^t} \frac{k(u(x-y))}{u} \d u)  \right\} \d x\d y\\
= &\frac{n^2}{2}\lim_{t\to\infty}\iint_{\Omega\times \Omega} \exp\Bigg\{ -  \frac{\|x-y\|^2}{\sigma^2}  +  \ga^2\int_1^{e^t} \frac{k(u(x-y))}{u} \d u \Bigg\} \d x\d y\\
= &\frac{n^2}{2} \iint_{\Omega\times \Omega} \exp\Bigg\{ -  \frac{\|x-y\|^2}{\sigma^2}    + \ga^2\int_1^{\infty} \frac{k(u(x-y))}{u} \d u \Bigg\} \d x\d y. \numberthis \label{eq:EC-1}
\end{align*}
A change of variables shows that
\begin{align*}
\int_1^{\infty} \frac{k(ux)}{u} \d u = \int_{\|x\|}^\infty \frac{k(u\vec{e})}{u} \d u=: \phi(\|x\|). \numberthis \label{eq:EC-2}
\end{align*}
Combining \cref{eq:EC-1} and \cref{eq:EC-2}, together with another change of variables $(x,y) \mapsto ({x}/{\sigma},{y}/{\sigma})$, gives
\begin{align*}
\EE [ \D ]& = \frac{n^2\sigma^{{2d}}}{2} \iint_{(\Omega/\sigma)^2 } \exp\Big\{- \|x-y\|^2 + \ga^2\phi \left(\|x-y\|\sigma \right) \Big\} \d x\d y \numberthis \label{eq:EC-3}
\end{align*}
 Since $\|x-y\|\sigma\le 1$ for $x,y\in\Omega/\sigma$, one has
\begin{align*}
\phi(\|x-y\|\sigma) =\int_{\|x-y\|\sigma}^1 \frac{k(u\vec{e})}{u} \d u + \int_1^\infty \frac{k(u\vec e)}{u} \d u \numberthis \label{eq:EC-4},
\end{align*}
where the second term is finite by assumption. Thus, as $\sigma\downarrow 0$, one obtains from \cref{eq:EC-4} and Assumption \ref{assumption:main}
\begin{align*}
\phi(\|x-y\|\sigma) &= \left( \int_{\|x-y\|\sigma}^1 \frac{k(u\vec{e})}{u} \d u \right) \left( 1 + o(1) \right) \nonumber \\
&=\left( \log \frac{1}{\|x-y\|} + \log \frac{1}{\sigma} \right) \left( 1 + o(1) \right). \numberthis \label{eq:EC-5}
\end{align*}
Note that when $\ga^2<d$, we have $\int_{\R^d} \frac{1}{\|x\|^{\ga^2}}  e^{ - \|x\|^2 }  \d x < \infty$. Combining \cref{eq:EC-3} with \cref{eq:EC-5}, in the regime of small $\sigma$ and $\ga^2<d$, we obtain
\begin{align*}
\EE[ \D ] = & \frac{n^2\sigma^{{2d}}}{2} \cdot \iint_{(\Omega/\sigma)^2} \exp \left( - \|x-y\|^2 \right) \cdot \frac{1}{\|x-y\|^{\ga^2}\sigma^{\ga^2}} \d x \d y \cdot (1+o(1))\\
= &\frac{n^2}{2} \sigma^{{2d-\ga^2}} \cdot  \int_{\Omega/\sigma} \left( \int_{{z=y-x} \atop {y \in \Omega/\sigma}} \frac{1}{\|z\|^{\ga^2}} e^{ - \|z\|^2 }  \d z \right) \d x \cdot (1+o(1)) \\
= &\frac{n^2}{2} \sigma^{{2d-\ga^2}} \cdot \left(  \int_{\Omega/\sigma} \left( \int_{\R^d} \frac{1}{\|z\|^{\ga^2}} e^{ - \|z\|^2 }  \d z \right) \d x  \right) \cdot \left( 1+ o(1) \right) \\
= &\frac{n^2}{2} \sigma^{{2d-\ga^2}} \cdot \left( \int_{\R^d} \frac{1}{\|z\|^{\ga^2}}  e^{ - \|z\|^2 }  \d z \right) \cdot \sigma^{-d} \cdot \left( 1+ o(1) \right) \\
= & C_1(\ga,d) \cdot n^2 \sigma^{d-\ga^2}  \left( 1+ o(1) \right)  , \numberthis \label{eq:EC-6}
\end{align*}
where,
\[
C_1(\ga,d) = \frac{1}{2} \int_{\R^d} \frac{1}{\|z\|^{\ga^2}} e^{-\|z\|^2} \d z < \infty.
\]
Using the choice $\sigma=\frac{1}{\sqrt{\pi}} \rho^{1/d}n^{-1/d}$ and $\nu=\ga^2/d$, we finally obtain
\begin{equation} \label{eq:EC-final}
\EE [ \D ] = C(\ga,d)  \rho^{1-\nu} n^{1+\nu}  \left( 1+ o(1) \right).
\end{equation}
\end{proof}

\begin{figure*}[t]
\centering
\includegraphics[scale=0.40]{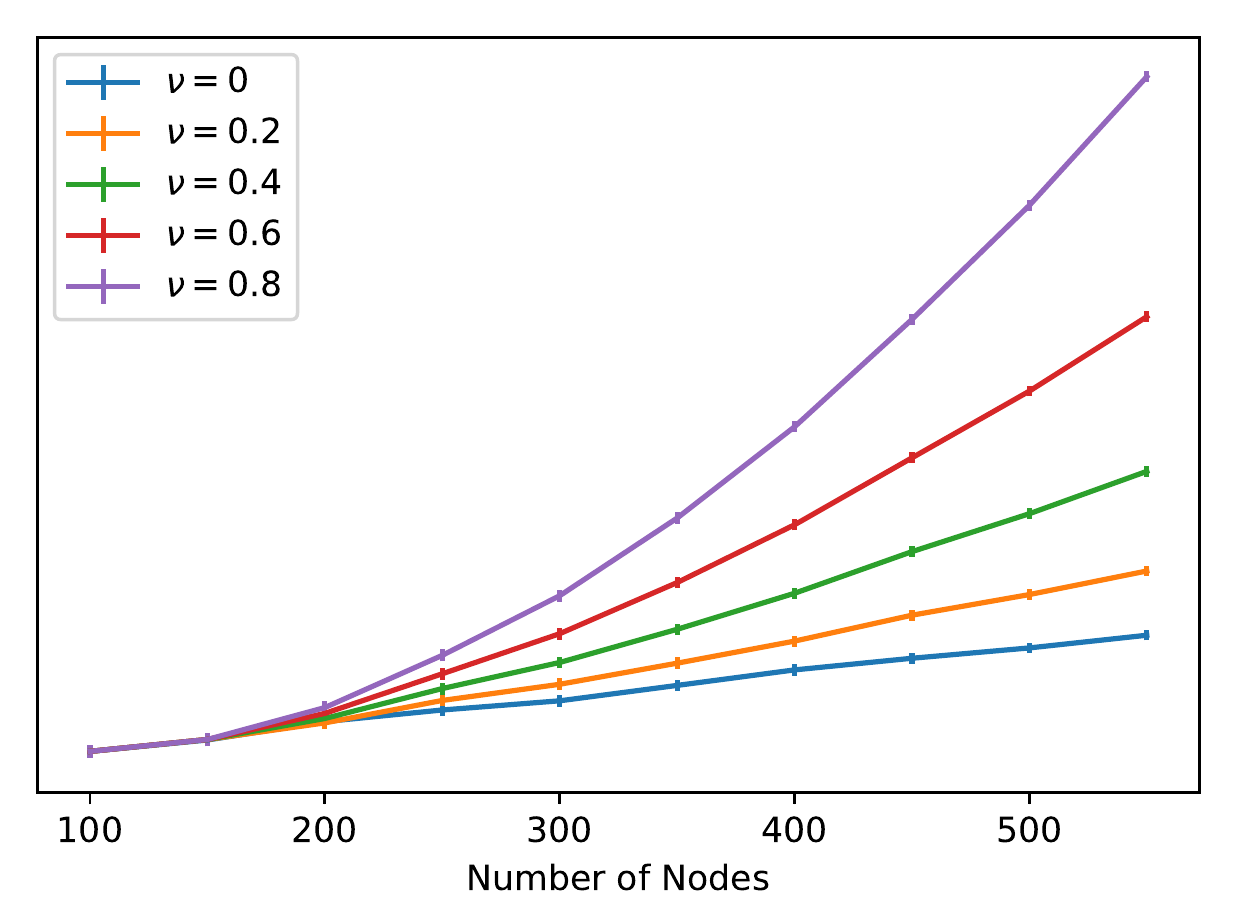}
\includegraphics[scale=0.40]{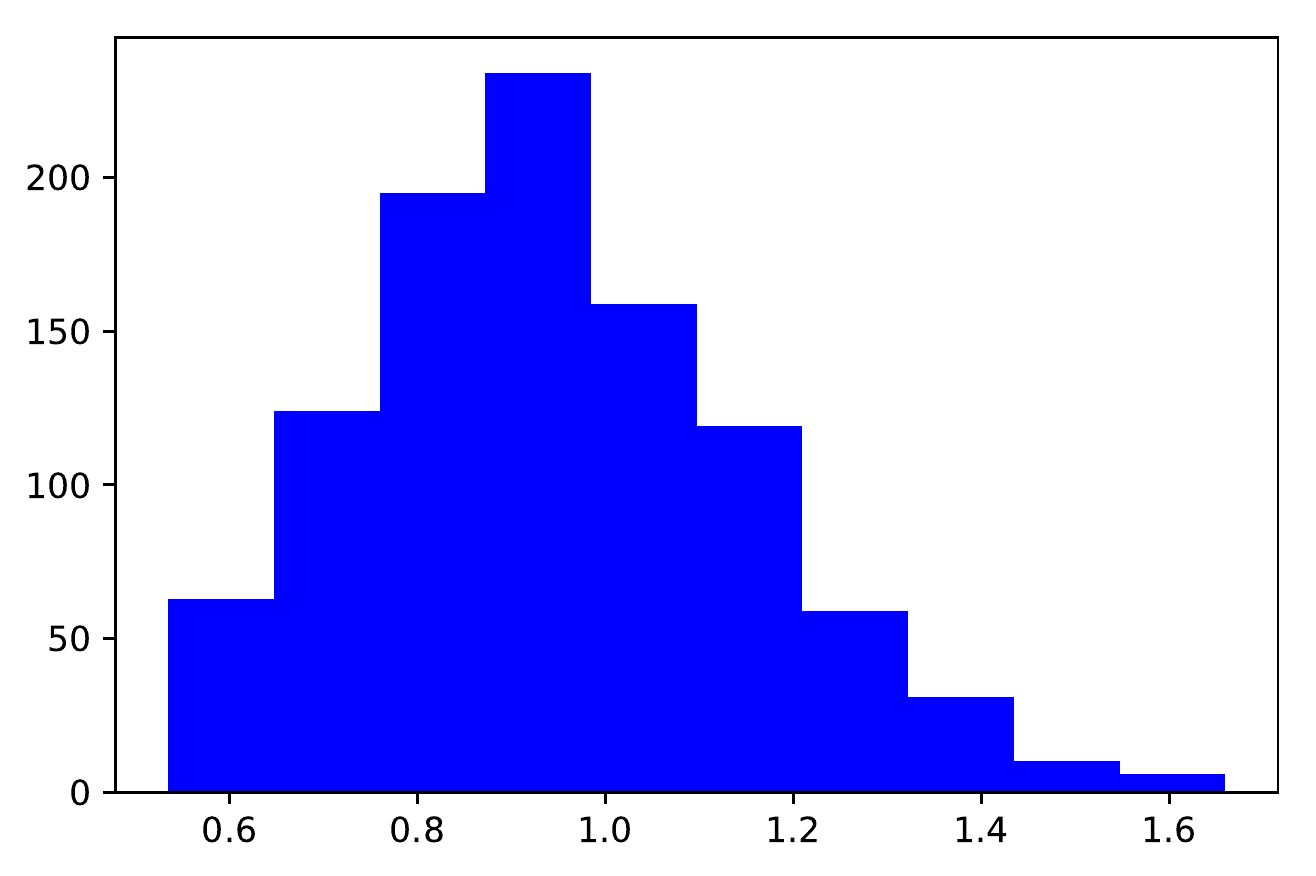}
\includegraphics[scale=0.40]{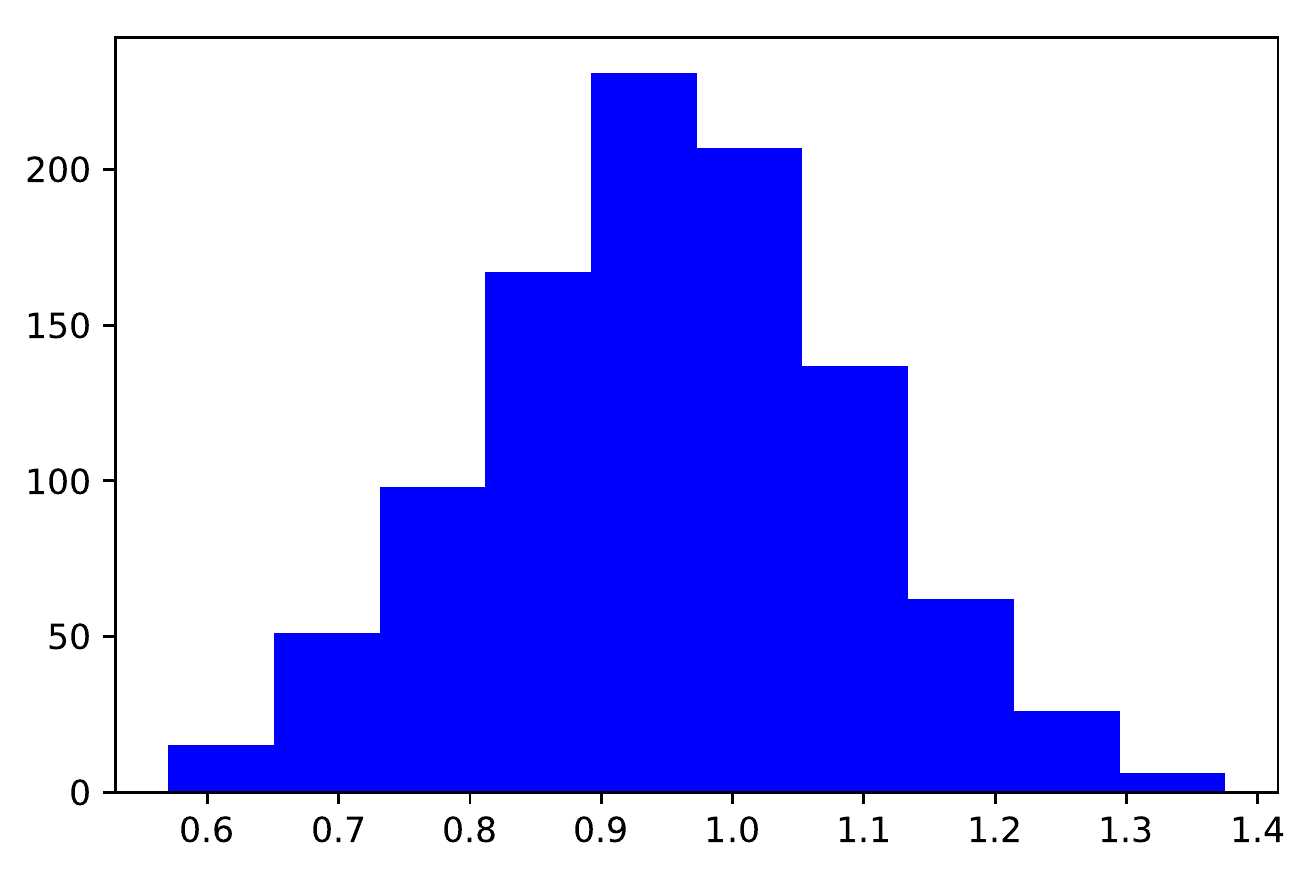}\\
\hspace{-0.1in}\includegraphics[scale=0.40]{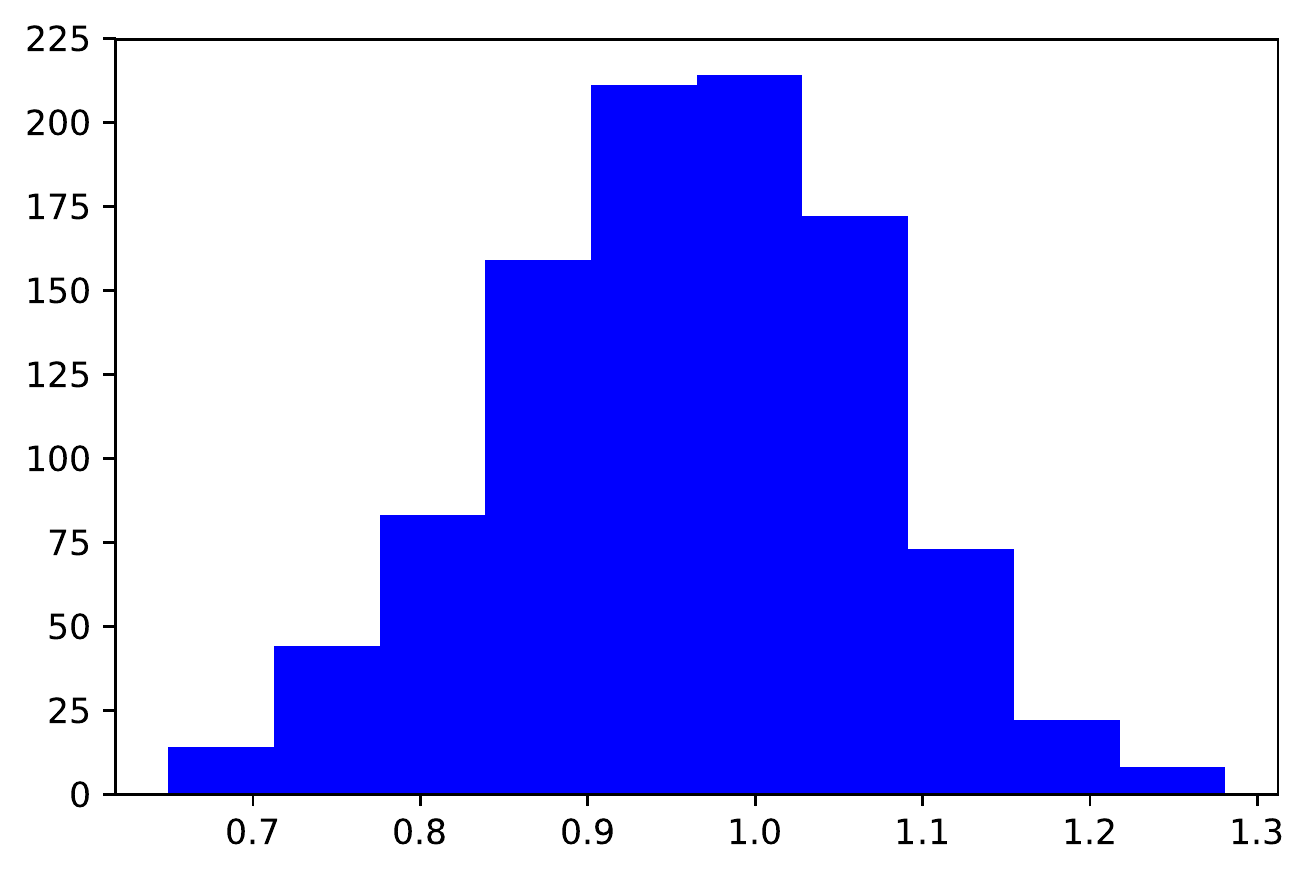}
\includegraphics[scale=0.40]{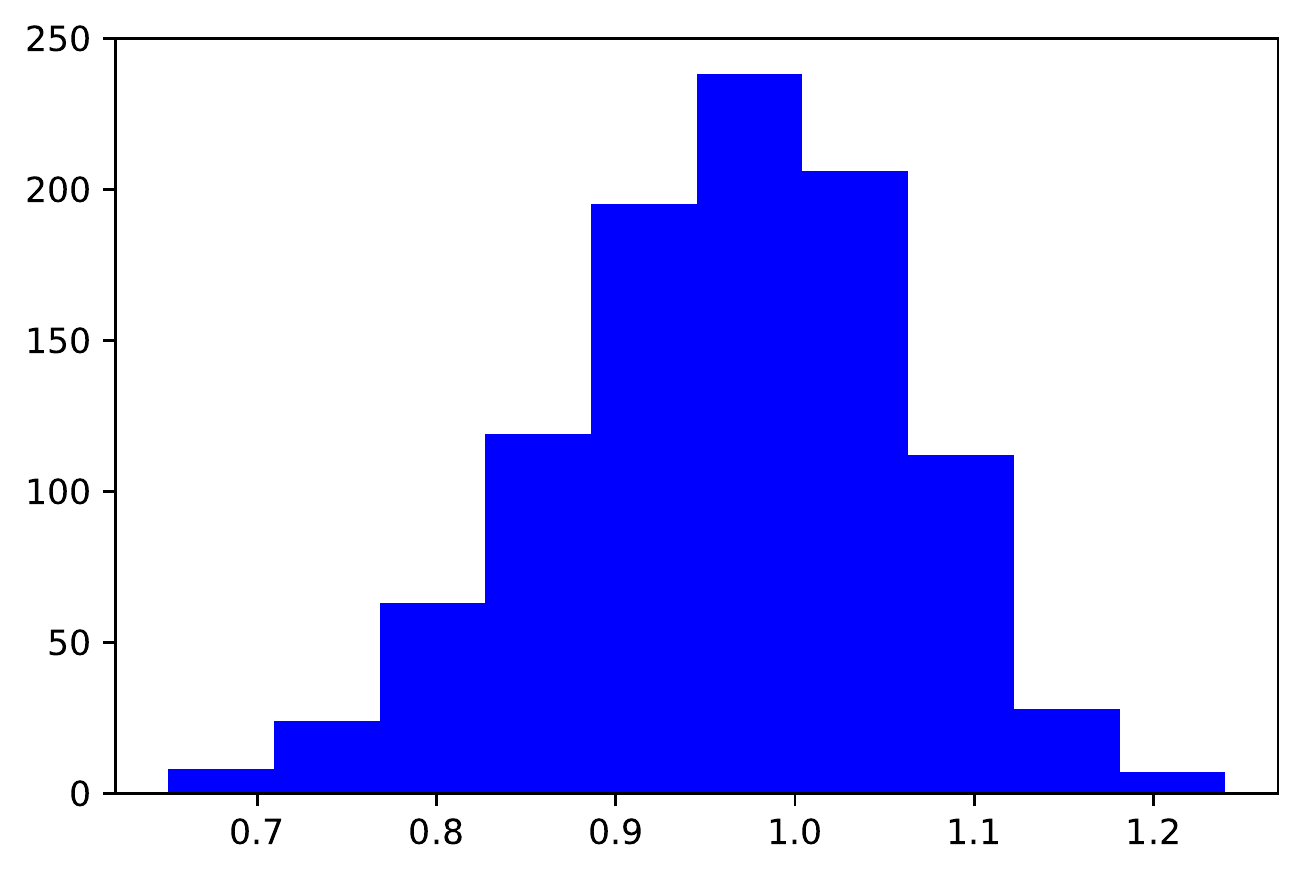}
\includegraphics[scale=0.40]{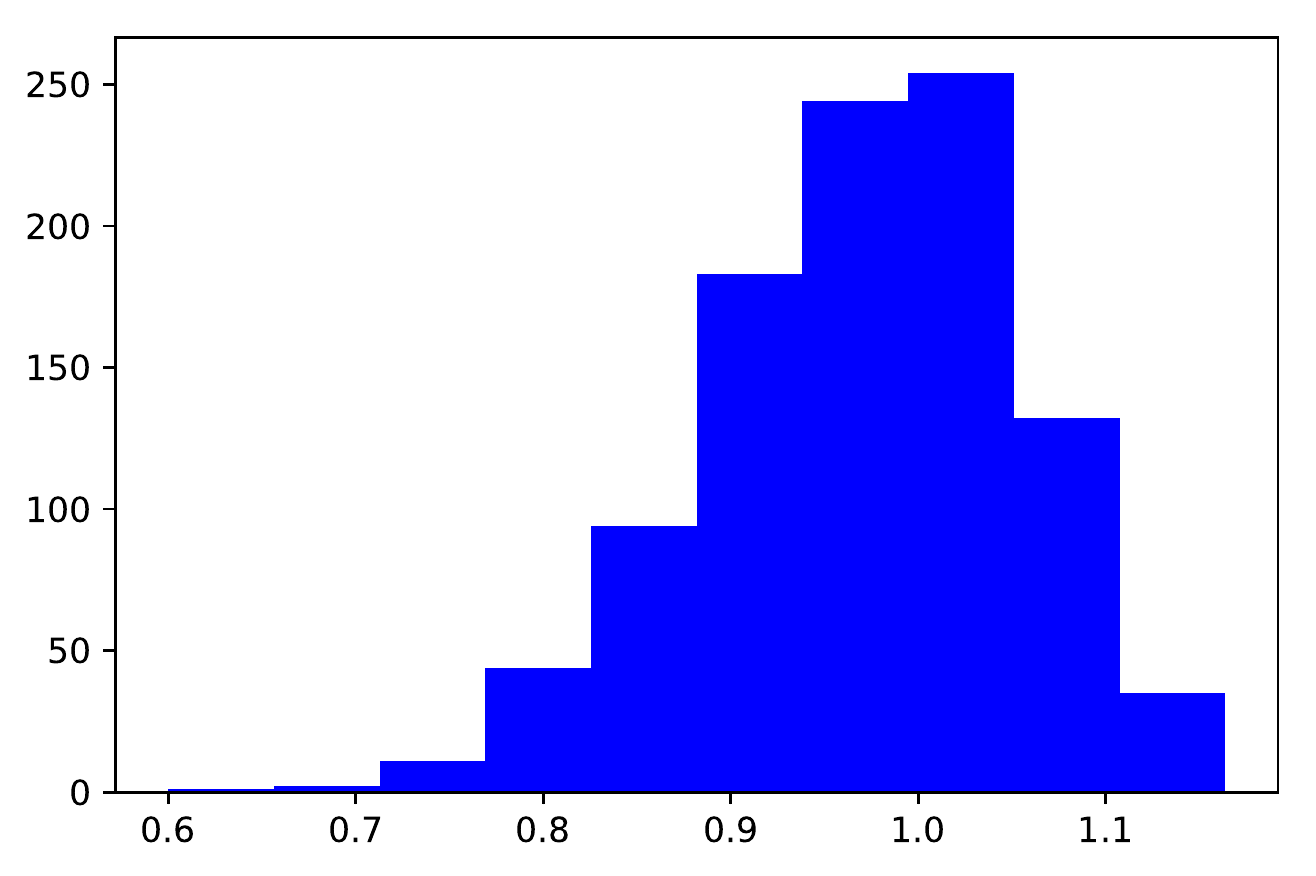}
\caption{The top left figure plots the expected edge count as a function of number of nodes for various increasing of $\nu$. The other five figures plots the histogram of edge counts for a fixed value of $n$ and with increasing values of $\nu$, as we move from left to right and from top to bottom. The $x$-axis is normalized by the sample mean.}
\label{fig:edgecounts}
\end{figure*}

\subsection{Triangle Counts}\label{sec:TC}
We now investigate the statistics of triangle counts in the FGN model.  In~\cref{fig:trianglecounts} (top left), we empirically plot expected triangle count. Specifically, we consider various values of $\nu$ and $n$ and compute the average number of triangles in 100 realizations of FGNs. We also plot the histogram of the triangle counts for various values of $\nu$ in~\cref{fig:trianglecounts}. Similar to the edge count, this statistic is concentrated around its expected value, albeit with possibly heavy tails.  Focussing on the asymptotics thereof, in the limit of large size parameter $n$, we obtain the following analytical result for the expected triangle count in~\cref{thm:TC}.

\begin{theorem}
\label{thm:TC}
For the FGN model with size parameter $n$, density parameter $\rho$ and fractality parameter $$\nu=\ga^2/d<1/2,$$ under Assumption~\ref{assumption:main}, the expected triangle count satisfies
\[ \EE [ \Del ] =  C(\ga,d)  \rho^{2-\nu} n^{1+\nu}  \left( 1+ o(1) \right), \]
as $n \to \infty$, where  $C(\ga,d):=$
\[\frac{1}{6}\pi^{\frac{d}{2}} \cdot \left(  \iint_{\R^d \times \R^d}\frac{1}{\|u-v\|^{\ga^2} }\frac{e^{-\|u\|^2}}{\|u\|^{\ga^2}} \frac{e^{-\|v\|^2}}{\|v\|^{\ga^2}} e^{-\|u-v\|^2}\d u \d v \right).\]
\end{theorem}

\begin{proof}
In the computations that follow, we will use the fact that if $\Lambda \sim \mathrm{Poi}(\lambda)$, then the second \textsl{factorial moment} of $\Lambda$ is given by the relation $\EE [{\Lambda \choose 3}]=\frac{\lambda^3}{3!}$ (c.f. \cite{haight1967handbook}). Consequently, recalling that given the GMC $\d \mg$ the node count $N \sim \mathrm{Poi}(n\mgo)$, we may deduce that $\EE [{N \choose 3}|\mgo] =\frac{1}{6} \cdot n^3\mgo^3 $.

 In view of this, we may proceed with the expectation of the triangle count $\Del$ as
\begin{align*}
&\EE [ \Del ] \\
= &\EE \left[ \sum_{1\le i<j<k\le N}  \exp\left( - \frac{\|x_i-x_j\|^2 + \|x_i-x_k\|^2 + \|x_j-x_k\|^2}{\sigma^2} \right) \right] \\
= &\EE \left[ \EE \left[ \sum_{1<i<j<k\le N}  \exp\left(  - \frac{\|x_i-x_j\|^2 + \|x_i-x_k\|^2 + \|x_j-x_k\|^2}{\sigma^2} \right) \big|  N, \d M^\ga \right] \right] \\
= &\EE \left[ {N\choose 3} \cdot \left( \iiint_{\Omega^3} \exp \left( - {\frac{\|x-y\|^2+\|x-z\|^2+\|y-z\|^2}{\sigma^2}} \right) \frac{M^\ga(dx)M^\ga(dy)M^\ga(dz)}{\mgo^3} \right) \right]\\
= &\EE \left[ \EE \left[ {N\choose 3} \cdot \left( \iiint_{\Omega^3} \exp \left( - {\frac{\|x-y\|^2+\|x-z\|^2+\|y-z\|^2}{\sigma^2}} \right)\frac{M^\ga(dx)M^\ga(dy)M^\ga(dz)}{\mgo^3} \right) \big| \d \mg  \right] \right]\\
= &\EE \left[ \EE \left[ {N\choose 3} \big| \d \mg  \right] \cdot \left( \iiint_{\Omega^3} \exp \left( - {\frac{\|x-y\|^2+\|x-z\|^2+\|y-z\|^2}{\sigma^2}} \right) \frac{M^\ga(dx)M^\ga(dy)M^\ga(dz)}{\mgo^3} \right)  \right]\\
= &\EE \left[\frac{1}{6} \cdot n^3\mgo^3 \cdot \left( \iiint_{\Omega^3} \exp \left( - {\frac{\|x-y\|^2+\|x-z\|^2+\|y-z\|^2}{\sigma^2}} \right)\frac{M^\ga(dx)M^\ga(dy)M^\ga(dz)}{\mgo^3} \right)  \right]\\
= &\frac{1}{6} n^3 \cdot \EE \left[  \iiint_{\Omega^3} \exp \left( - {\frac{\|x-y\|^2+\|x-z\|^2+\|y-z\|^2}{\sigma^2}} \right)M^\ga(dx)M^\ga(dy)M^\ga(dz) \right] .\\
\end{align*}
For further analysis, we introduce
\[ I := \EE \left[  \iiint_{\Omega^3} \exp \left( - {\frac{\|x-y\|^2+\|x-z\|^2+\|y-z\|^2}{\sigma^2}} \right)M^\ga(dx)M^\ga(dy)M^\ga(dz) \right] .  \]
and
\[ I_t := \EE \left[ \iiint_{\Omega^3} \exp\left\{-\frac{\|x-y\|^2+\|x-z\|^2+\|y-z\|^2}{\sigma^2} + \ga(X_t(x)+X_t(y)+X_t(z)) - \frac{3\ga^2t}{2}\right\}  \d x \d y \d z \right] .  \]
In view of the convergence in \cref{eq:approx}, we will use $I_t$ as an approximation for $I$ as $t \to \infty$.

Using the fact that for fixed $t$ the field $\{X_t(x)\}$ is a centered Gaussian random field with covariance structure as in \cref{eq:pre-kernel}, we may then proceed further as
\begin{align*}
&I_t \\
= &\EE \left[ \iiint_{\Omega^3} \exp\left\{-\frac{\|x-y\|^2+\|x-z\|^2+\|y-z\|^2}{\sigma^2} + \ga(X_t(x)+X_t(y)+X_t(z)) - \frac{3\ga^2t}{2}\right\}  \d x \d y \d z \right] \\
= &\iiint_{\Omega^3}  \exp\Bigg\{ -\frac{\|x-y\|^2+\|x-z\|^2+\|y-z\|^2}{\sigma^2} \\
 &\quad + \frac{\ga^2}{2}\left(3t + 2\int_1^{e^t} \frac{k(u(x-y))+k(u(y-z))+k(u(x-z))}{u} du\right) -\frac{3t\ga^2}{2}\Bigg\} \d x \d y \d z \\
= &\iiint_{\Omega^3}  \exp\Bigg\{ -\frac{\|x-y\|^2+\|x-z\|^2+\|y-z\|^2}{\sigma^2} \\
 &\quad + \ga^2 \int_1^{e^t} \frac{k(u(x-y))+k(u(y-z))+k(u(x-z))}{u} du\Bigg\} \d x \d y \d z
\end{align*}
Invoking \cref{eq:EC-2} and using the change of variables $(x,y,z) \mapsto (x/\sigma,y/\sigma,z/\sigma)$, we obtain
\begin{align*}
\EE[\Del] = \frac{n^3\sigma^{{3d}}}{6} \iiint_{(\frac \Omega \sigma)^3} \exp\Bigg\{ - \|x-y\|^2 - \|y-z\|^2 - \|x-z\|^2  \\
+\ga^2(\phi(\|x-y\|\sigma)+\phi(\|x-z\|\sigma)+\phi(\|y-z\|\sigma))\Bigg\} \d x \d y \d z,
\end{align*}
As $\sigma\downarrow 0$, using the properties of $\phi$ (c.f. \cref{eq:EC-2}, \cref{eq:EC-4} and \cref{eq:EC-5}), one has
\begin{equation} \label{eq:TC-0}
\EE [\Del] =(1+o(1)) \frac{n^3}{6}\sigma^{{3d - \ga^2}} \iiint_{(\frac \Omega \sigma)^3}  \frac{1}{(\|x-y\|\|y-z\|\|z-x\|)^{\ga^2}} e^{-\|x-y\|^2-\|x-z\|^2-\|y-z\|^2} \d x \d y \d z.
\end{equation}
We are left to estimate the triple integral above in the limit of large size parameter $n$ (and $\sigma$ decaying as $n^{-1/d}$).
Setting $u:=x-z,v:=y-z$ and making the change of variables $(x,y,z) \mapsto (u,v,z)$ we obtain
\begin{equation} \label{eq:TC-1}
\EE[\Del] = \frac{n^3\sigma^{{3d - \ga^2}}}{6}  \int_{\Omega/\sigma} \left(\iint_{(u,v):(x,y,z) \in (\Omega/\sigma)^3} \frac{e^{-\|u\|^2-\|v\|^2-\|u-v\|^2}}{\|u\|^{\ga^2} \|v\|^{\ga^2} \|u-v\|^{\ga^2} } \d u \d v  \right) dz.
\end{equation}
Notice that, from the definition of $u$ and $v$ it follows that for a given $z$, the set $\{(u,v):(x,y,z) \in (\Omega/\sigma)^3\}=(\Omega/\sigma)^2-(z,z)$. Since $z$ itself belongs to the set $\Omega/\sigma$, it follows that the domain of $(u,v)$ is always of the order $\frac{1}{\sigma}$. Since $\sigma \downarrow 0$, the inner double integral tends to \[J:= \iint_{\R^d \times \R^d}\frac{1}{\|u-v\|^{\ga^2} }\frac{e^{-\|u\|^2}}{\|u\|^{\ga^2}} \frac{e^{-\|v\|^2}}{\|v\|^{\ga^2}} e^{-\|u-v\|^2}\d u \d v.\]
If we can show that $J$ as defined above is finite, it would imply that the triple integral in \cref{eq:TC-1} is of the order $\sigma^{-d}$, and hence $\EE[\Del] = \Theta(n^3\sigma^{2d-\ga^2})=\Theta(\rho^{2-\nu}n^{1+\nu})$, wherein we have used the optimal choice of $\sigma$ as $\frac{1}{\sqrt{\pi}}\rho^{1/d}n^{-1/d}$.
Thus, we  focus on showing that $J<\infty$.

It may be noted that, for $\ga^2<d$, the integral $$\int_{\|u\|>1} \frac{1}{\|u-v\|^{\ga^2} }\frac{e^{-\|u\|^2}}{\|u\|^{\ga^2}} e^{-\|u-v\|^2} \d u$$ is finite and uniformly bounded in $v$, a fact that follows essentially from the exponentially decaying terms in the integrand. Since $v\mapsto \frac{e^{-\|v\|^2}}{\|v\|^{\ga^2}}$ is integrable for $\ga^2<d$, this implies that the double integral
\[\iint_{\{\|u\|>1\} \times \R^d}\frac{1}{\|u-v\|^{\ga^2} }\frac{e^{-\|u\|^2}}{\|u\|^{\ga^2}} \frac{e^{-\|v\|^2}}{\|v\|^{\ga^2}} e^{-\|u-v\|^2}\d u\d v < \infty.\]
This argument is symmetric in $u$ and $v$,  so in order to show the finiteness of $J$, we are reduced to showing the finiteness of the associated double integral
 \[J^*:= \iint_{\{|u|\le 1\} \cap \{|v|\le 1\}}\frac{1}{|u-v|^{\ga^2} }\frac{\d u}{\|u\|^{\ga^2}} \frac{\d v}{\|v\|^{\ga^2}},\]
where we have used the fact that $e^{-\|u\|^2}$ is uniformly bounded on $\|u\|\le 1$.
 We consider the measure $\mu$ on the unit Euclidean ball $\mathcal{B}$ of $\R^d$ given by $d\mu(u):=\frac{du}{\|u\|^{\ga^2}}$. Notice that, if $B(x,r)$ denotes the Euclidean ball of radius $r>0$ centred at $x \in \mathcal{B}$, we have $\mu(B(x,r))\le c_d  r^{d-\ga^2}$. Then the integral $J^*$ reduces to
 \[J^*=\iint_{\mathcal{B} \times \mathcal{B}}\frac{d\mu(u) d\mu(v)}{\|u-v\|^{\ga^2}}. \]

To consider the finiteness of $J^*$, we fix $u \in \mathcal{B}$ and consider the integral $\int_{\mathcal{B}}\frac{d \mu(v)}{\|u-v\|^{\ga^2}}$; our goal is to show that for $2\ga^2<d$, this integral is finite and uniformly bounded in $u$,  which would complete our proof. We proceed as
\begin{align*}
\int_{\mathcal{B}}\frac{d \mu(v)}{|u-v|^{\ga^2}}
=& \sum_{n=-1}^\infty \int_{2^{-n-1}<\|u-v\| \le 2^{-n}} \frac{d \mu(v)}{|u-v|^{\ga^2}} \\
\le & \sum_{n=-1}^\infty 2^{(n+1)\ga^2}\mu \big(B(u,2^{-n}) \setminus B(u,2^{-n-1})\big) \\
\le & \sum_{n=-1}^\infty 2^{(n+1)\ga^2}\mu \big(B(u,2^{-n})\big) \\
\le & \sum_{n=-1}^\infty 2^{(n+1)\ga^2} \cdot c_d 2^{-n(d-\ga^2)} \\
\le & c_d\cdot 2^{\ga^2} \cdot \sum_{n=-1}^\infty 2^{-n(d-2\ga^2)},
\end{align*}
which is finite and uniformly bounded in $u \in \mathcal{B}$ when $d-2\ga^2>0$ ; equivalently, when $\nu=\frac{\ga^2}{d}<\frac{1}{2}$. This completes the proof.
\end{proof}

\begin{figure*}[t]
\centering
\includegraphics[scale=0.40]{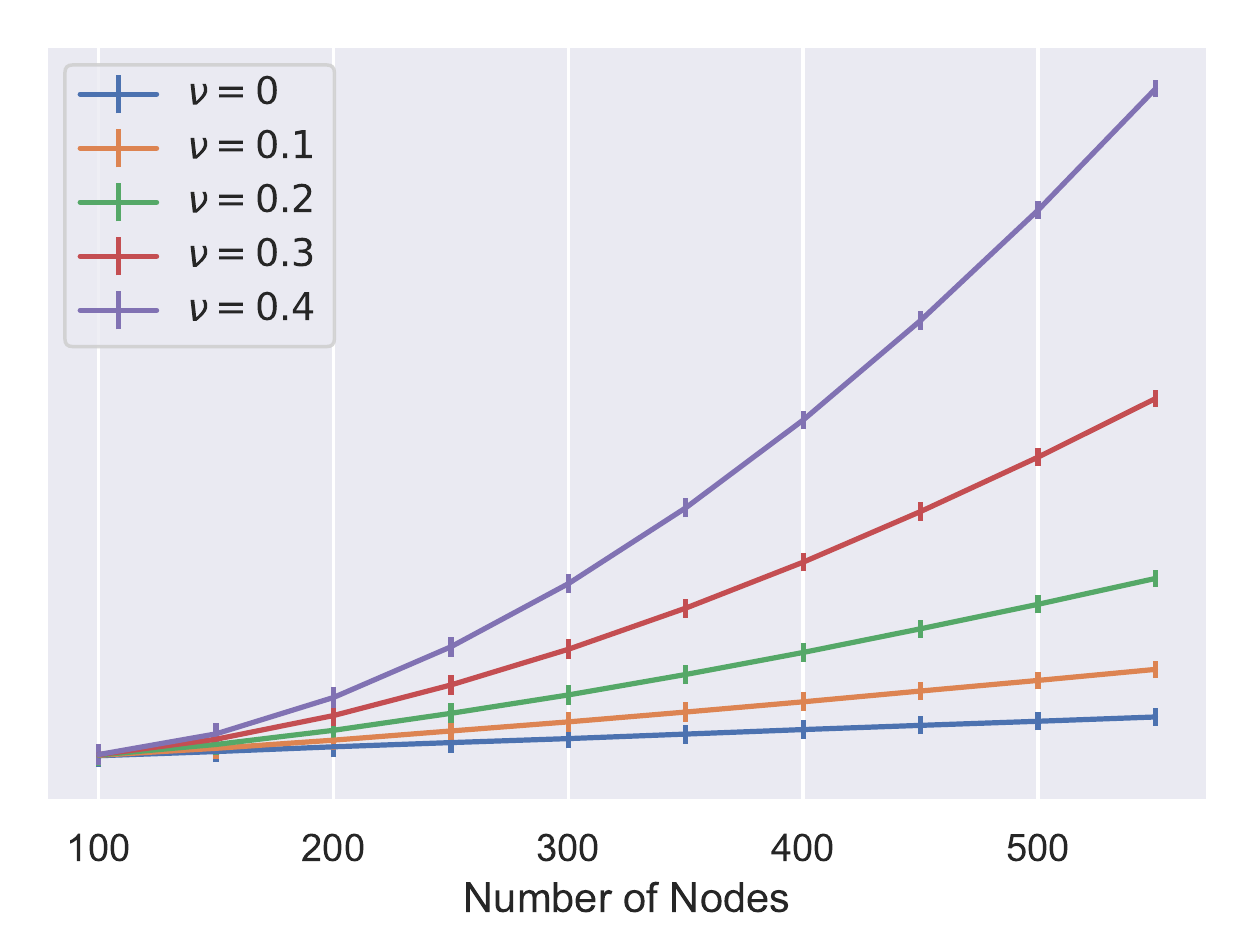}
\includegraphics[scale=0.40]{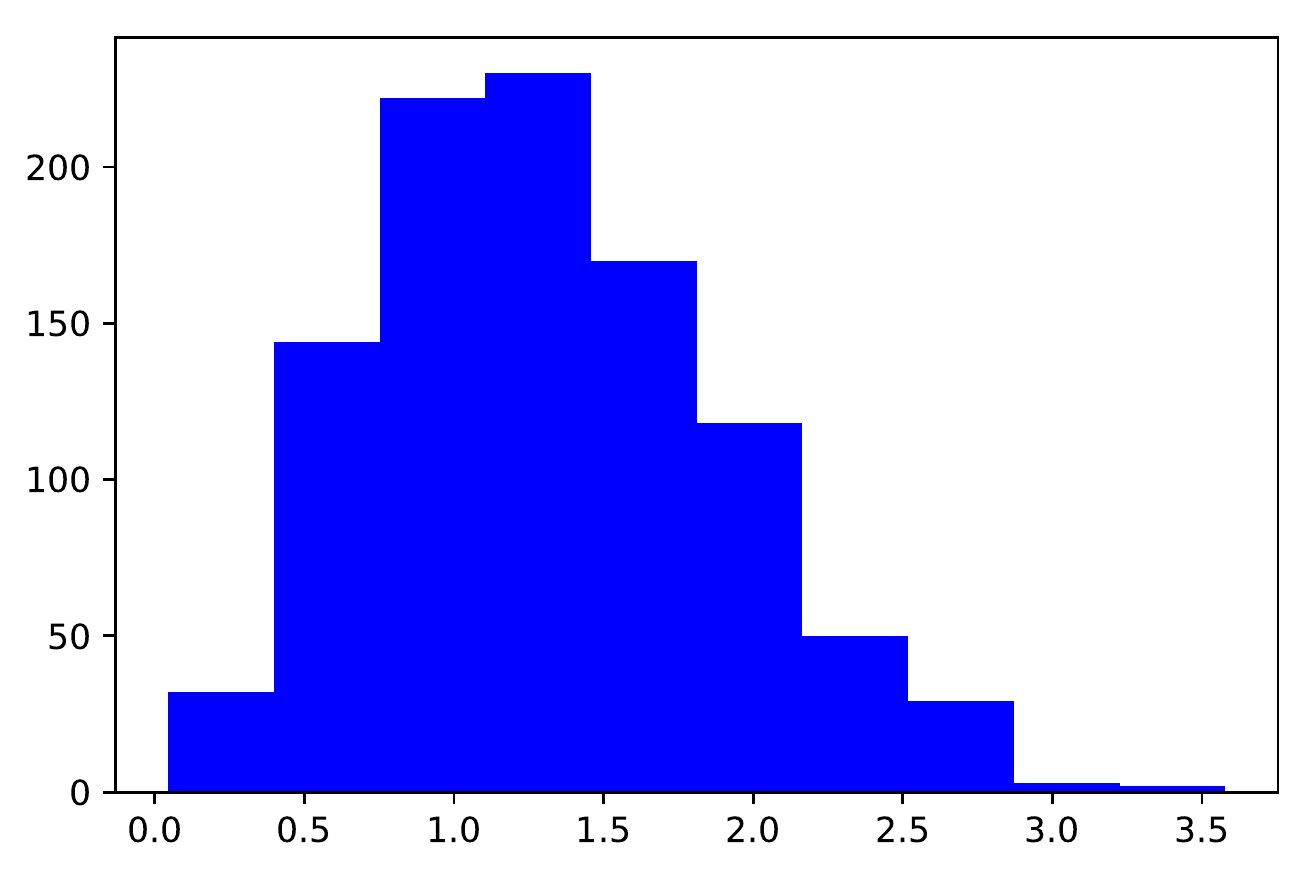}
\includegraphics[scale=0.40]{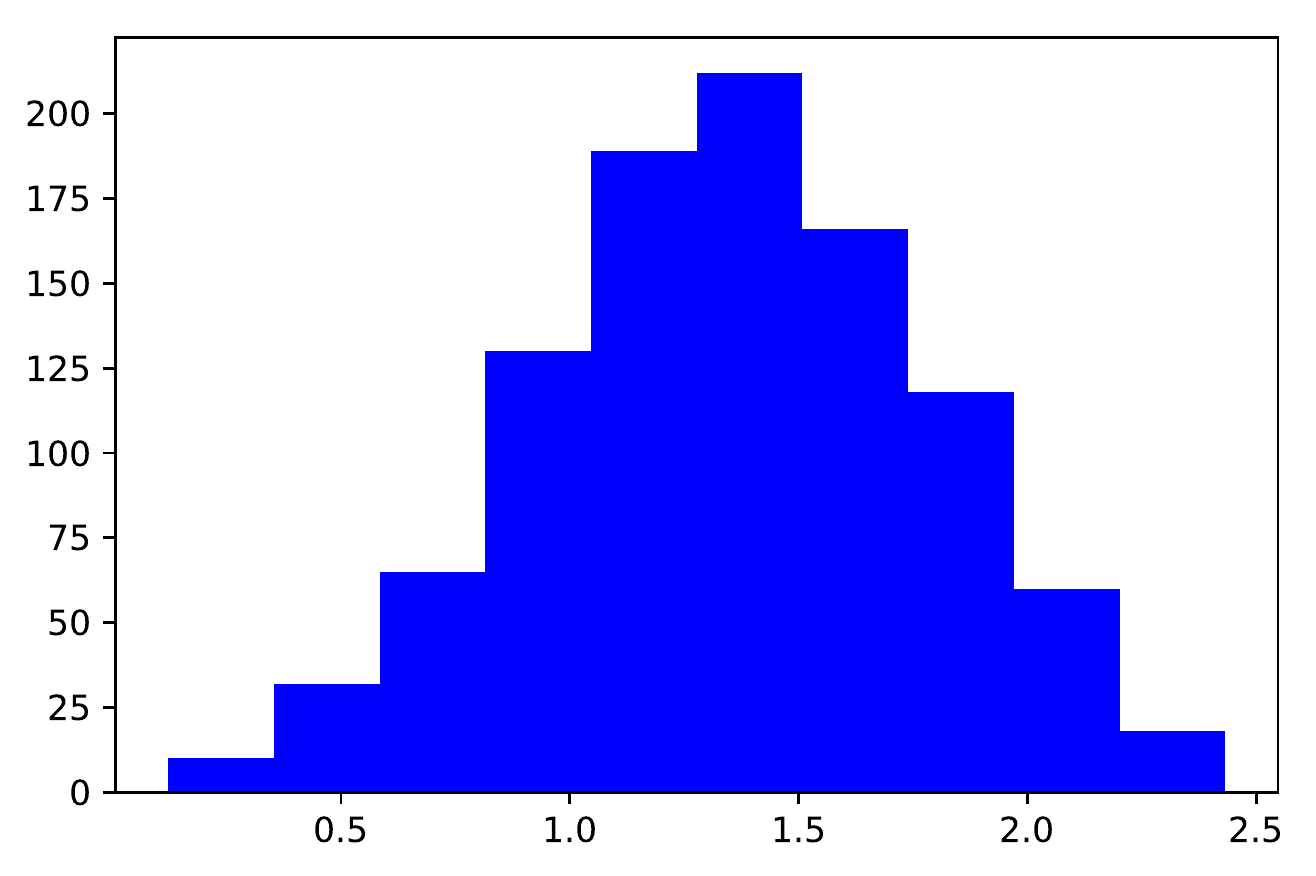}\\
\hspace{-0.1in}\includegraphics[scale=0.40]{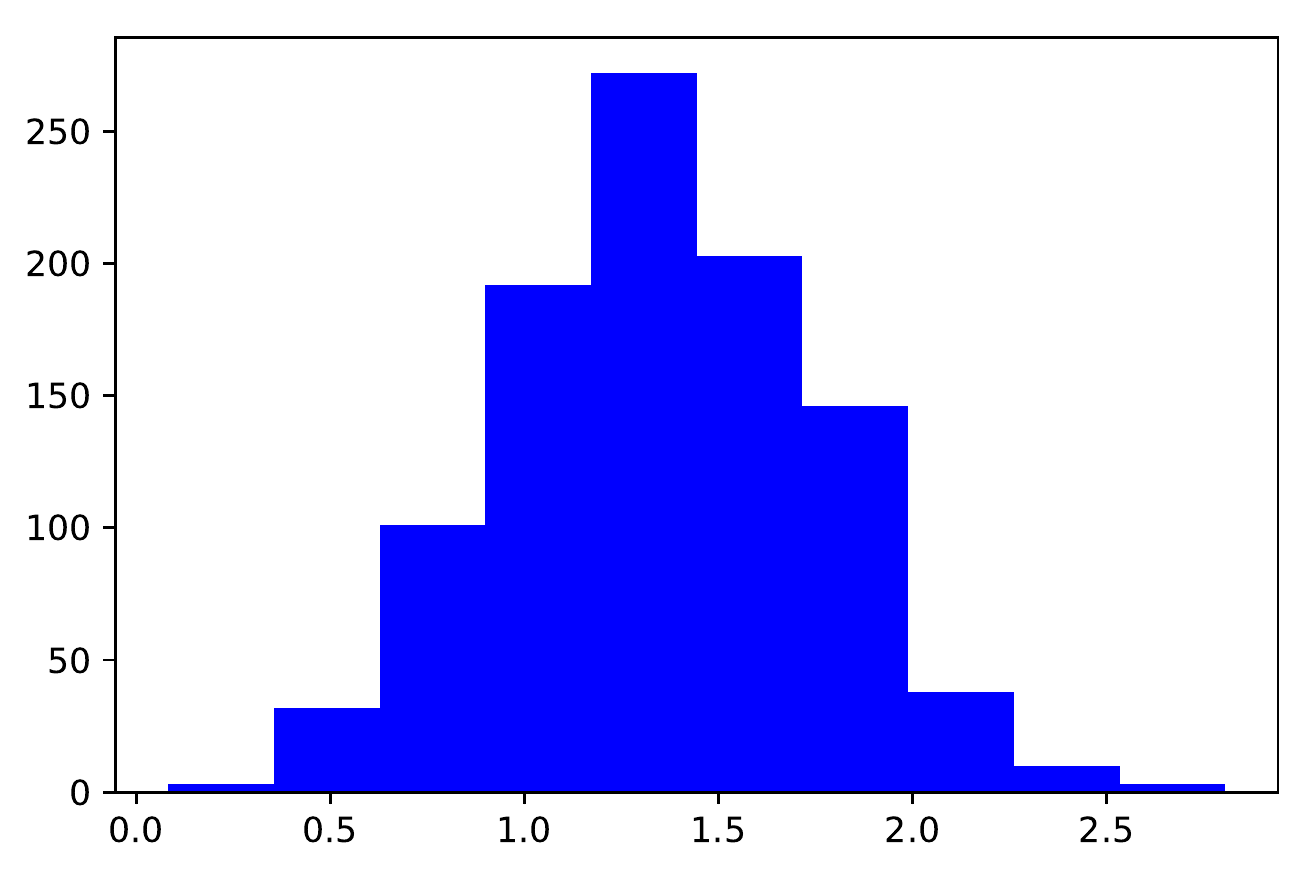}
\includegraphics[scale=0.40]{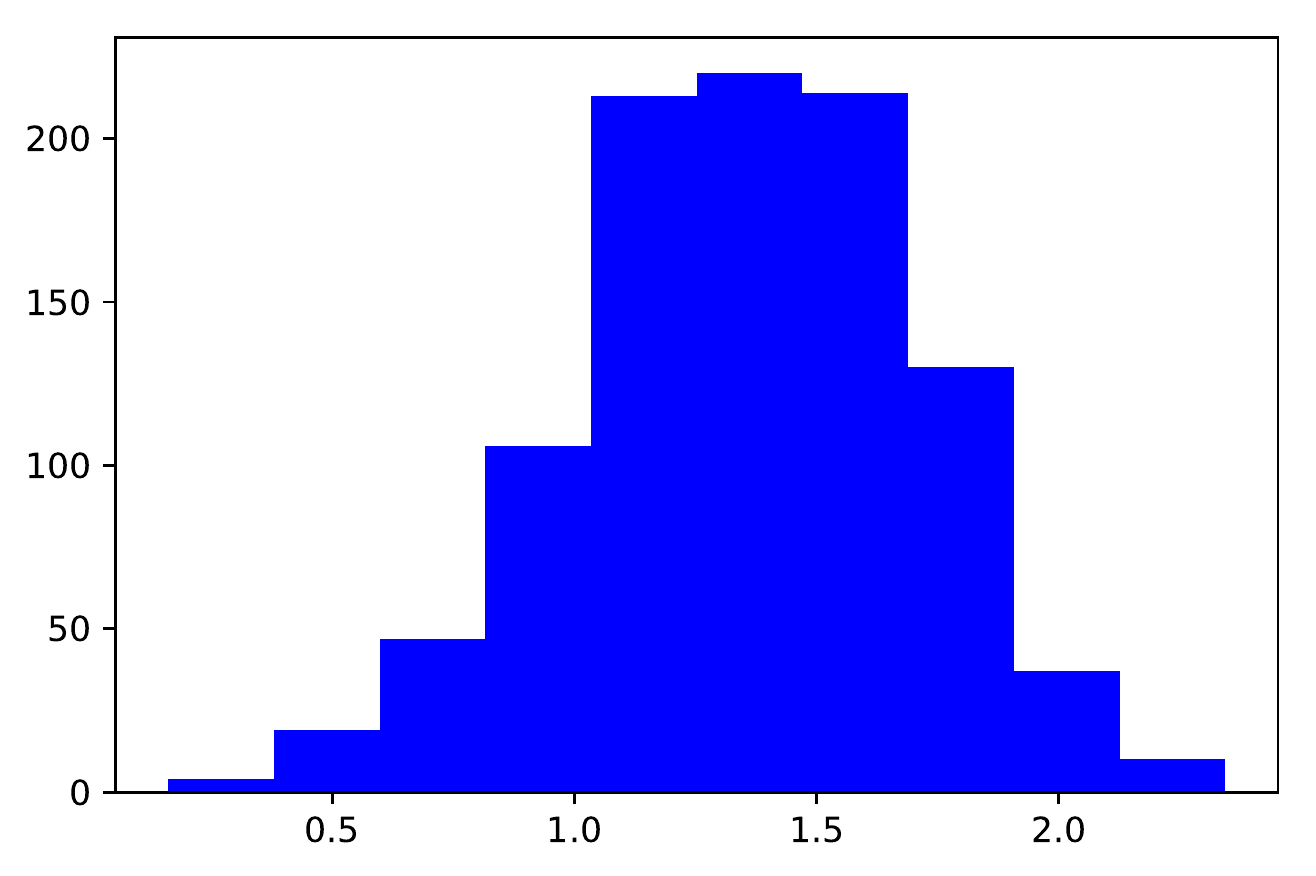}
\includegraphics[scale=0.40]{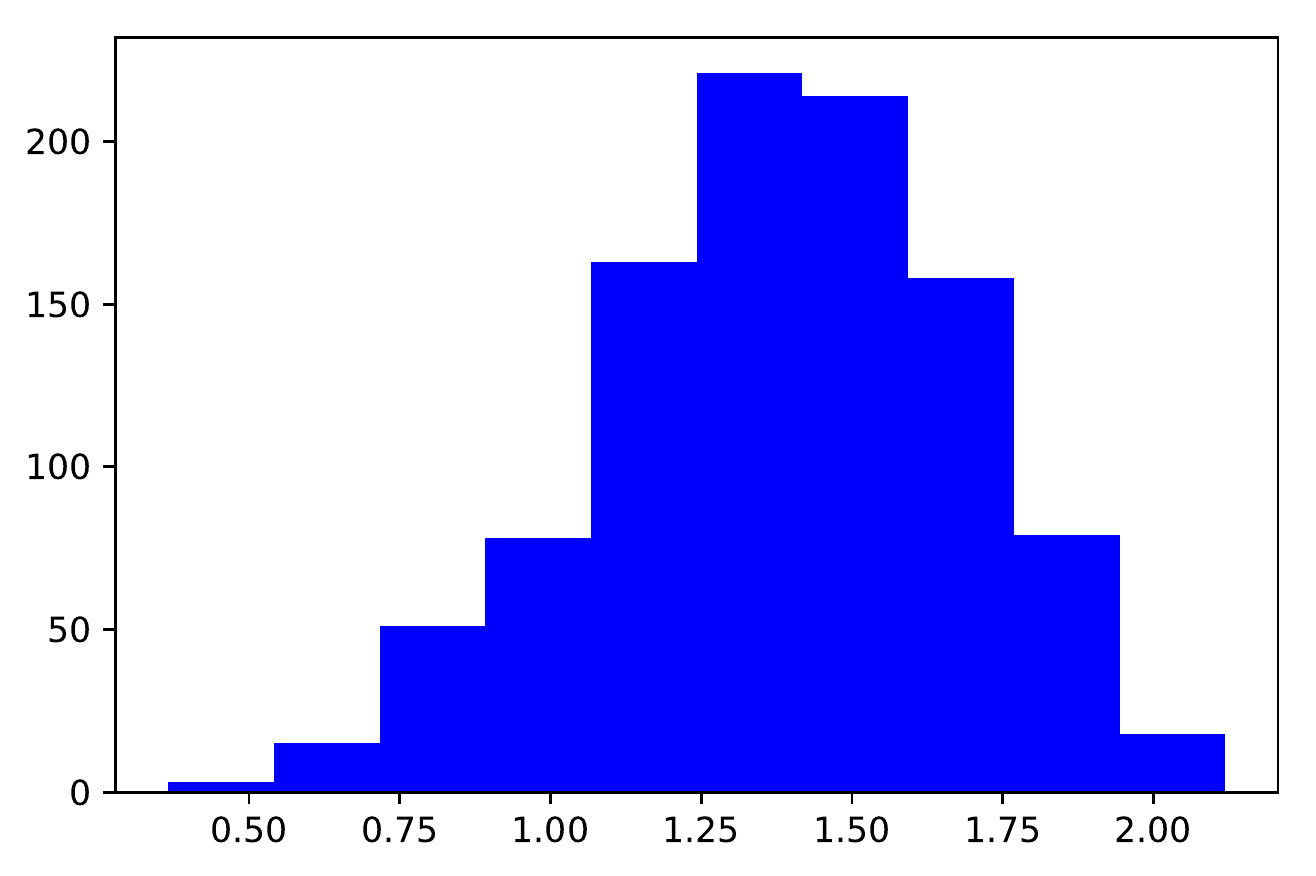}
\caption{The top left figure plots the expected triangle count as a function of number of nodes for various increasing of $\nu$. The other five figures plots the histogram of triangle counts for a fixed value of $n$ and with increasing values of $\nu$, as we move from left to right and from top to bottom. The $x$-axis is normalized by the sample mean.}
\label{fig:trianglecounts}
\end{figure*}

\subsection{\textcolor{black}{Hub-and-Spoke Counts}} \label{sec:HSC}

A $k$-spoke with hub $v$ in a simple graph $G=(V,E)$ is a subgraph of $G$ formed by $v\in V$ and $\{u_1,...,u_k\}\subset V$ such that $v\sim u_i$ for all $i\in[k]$, where $u\sim v$ means that $\{u,v\}$ is connected by an edge in $E$. In this section, we consider the number of $k$-spokes in FGN, denoted by $S_k$, $k\in\NN$. Notice that the number of $1$-spoke is twice the number of edges. A $2$-spoke is also known as a cherry.  

\begin{theorem} \label{thm:HSC}
 For integer $k\geq 2$, suppose that $\nu= \ga^2/d < \min(k^{-1},  2(k(k-1))^{-1})$. Then there exists a finite constant $C(k, \gamma, d, \rho)>0$ such that
\begin{align*}
\EE[S_k] = C(\gamma, d, k, \rho)  n^{1+\nu} (1+o(1)).
\end{align*}
\end{theorem}

\begin{proof}
Notice that the point process $\{X_1,...,X_N\}$ is a Poisson process with random intensity measure $M^\ga$ (also known as Cox process). By the conditional independence of the edges given the underlying point process, together with  the multivariate Mecke equation for Poisson processes (see, for example~\cite[Chapter 4]{last2017lectures}), one has that 
\begin{align*}
\EE[S_k] = \frac{n^{k+1}}{k!} \EE \int_{\Omega^{k+1}}  \exp( - \sum_{i=1}^k \|x_i-y\|^2 ) M^\ga(dy) \prod_{i=1}^k M^\ga( dx_i) .
\end{align*}
Here $k!$ counts the number of isomorphic $k$-spokes with the same hub. 

Now we evaluate this expectation by approximating $M^\ga$ with $M^\ga_t$, then letting $t\to\infty$. More precisely, we have
\begin{align*}
\EE[S_k] = \frac{n^{k+1}}{k!} \lim_{t\to\infty} \int_{\Omega^{k+1}} 
\EE \exp\Big(- \sum_{i=1}^k \|x_i-y\|^2 + \ga(X_t(y)+\sum_{i=1}^k X_t(x_i))- \frac{k+1}{2}\ga^2 t \Big) dy \prod_{i=1}^k dx_i.
\end{align*}
Notice that  the sum of variances of the Gaussian variables cancels with the term $\frac{k+1}{2}\ga^2 t$  in the exponent when one evaluates the expectation on the right-hand side. By scaling each variable by $1/\sigma$, then handling the covariances as we did in the proof of Theorem \ref{thm:TC}, we are led to the formula
\begin{align*}
\EE[S_k] = (1+o(1))\frac{n^{k+1}}{k!} \sigma^{(k+1)d -\ga^2} \int_{(\Omega/\sigma)^{k+1}} 
\exp( - \sum_{i=1}^k \|x_i-x_{k+1}\|^2 ) \prod_{\{i,j\}\subset[k+1], i\neq j} \|x_i - x_j\|^{-\ga^2} \prod_{i=1}^{k+1} dx_i.
\end{align*}
where we have renamed the variable $y$ to $x_{k+1}$.

In an analogous fashion as \cref{eq:TC-1} was obtained, as well as the fact that $\sigma$ converges to 0 as $n\to\infty$,  one can shift the position of $x_{k+1}$ and inspect the integral with respect to $x_1,...,x_k$ on the whole space $\RR^d$. Observing that integration over the $x_{k+1}$ variable on $\Omega/\sigma$ yields a multiplicative factor of $\sigma^{-d}$ in the right hand side of the display above, and recalling the choice of $\sigma=\frac{1}{\sqrt{\pi}} \rho^{1/d}n^{-1/d}$, the desired result follows as soon as one can prove that
\begin{align*}
K:= \int_{\RR^{d k}} \exp( - \sum_{i=1}^k \|x_i\|^2 ) \prod_{\{i,j\}\subset[k], i\neq j} \|x_i - x_j\|^{-\ga^2} \prod_{j\in[k]} \|x_j\|^{-\ga^2} \prod_{i=1}^{k} dx_i 
\end{align*}
is finite. 

Now we prove $K<\infty$. This is a bit delicate and we shall introduce three points which reduces the study of $K$ to much simpler integrals. Let $y_1$ denote the point amongst $\{x_i, i\in[k]\}$ that is closest to the origin. Let $y_2$ and $y_3$ be the closest pair in the family $\{x_i, i\in[k]\}$ where $y_2$ is closer to $y_1$ than $y_3$ is. Potentially $y_2=y_1$.  

We claim that the contribution of the case $y_1\neq y_2$ is at most
\begin{align*}
K_1&:=k {{k-1} \choose 2}\Big( \int_{\RR^d} \exp(-\|x\|) dx\Big)^{k-3} \\
&\qquad\qquad\int_{\RR^{3d}} \exp( - \|y_1\|^2 -\|y_2\|^2 -\|y_3\|^2 ) \|y_1\|^{-k\ga^2}   \|y_2-y_3\|^{-{k\choose 2}\ga^2}  dy_1 dy_2 dy_3 
&.
\end{align*}
Indeed, the factor $k{{k-1}\choose 2}$ counts the number of ways of choosing $y_1,y_2,y_3$, and that $\|x_i-x_j\|\ge  \|y_2-y_3\|$ and $\|x_i\|\ge \|y_1\|$ by the definition of the $y_i$'s. 

As a result, we have 
$K_2\le c \int_{\RR^{2d}} \exp( -\|y_2\|^2 -\|y_3\|^2 )  \|y_2-y_3\|^{-{k\choose 2}\ga^2} dy_2 dy_3 \le c'$
for some finite constants $c, c'>0$ because the integral with respect to $y_1$ is finite by the assumption that $k\ga^2<d$, and it is easily checked that the integral with respect to $y_2$ and $y_3$ is finite as soon as ${k\choose 2} \ga^2 <d$, as required. 

Similarly, we notice that the contribution of the case $y_1=y_2$ is at most
\begin{align*}
K_2:=k(k-1) \int_{\RR^{2d}} \exp(- \|x\|^2- \|y\|^2) \|x\|^{-k\ga^2} \|x-y\|^{-{k\choose 2}\ga^2} dx dy \Big( \int_{\RR^d} \exp(-\|z\|^2) dz\Big)^{k-2},
\end{align*}
where we have renamed $y_1$ and $y_3$ to $x$ and $y$ respectively. Next we distinguish four cases according to all combinations of the cases $\|x\|\le 1$, $\|x\|>1$, $\|x-y\|\le 1$ and $\|x-y\|>1$. When both $\|x\|$ and $\|y-x\|$ are larger than 1, the integral in question is clearly finite. If we have $\|x\|>1$ and $\|x-y\|\le 1$, then it is easy to see that the integral is finite when ${k\choose 2} \ga^2 <d$. If  $\|x\|\le 1$ and $\|x-y\|> 1$, then the integral is finite when $k\ga^2 <d$.  Finally, if both $\|x\|$ and $\|y-x\|$ are at most 1, then it suffices to show that 
\begin{align*}
\int_{B^2} \|x\|^{-k\ga^2} \|x-y\|^{- {k\choose 2} \ga^2} dx dy < \infty,
\end{align*}
where $B$ is the Euclidean ball centered at the origin of radius $2$. We have
\begin{align*}
\int_{B^2} \|x\|^{-k\ga^2} \|x-y\|^{- {k\choose 2} \ga^2} dx dy  \le \int_B \|x\|^{-k\ga^2} dx \int_{2B} \|z\|^{-{k\choose 2} \ga^2} dz <\infty
\end{align*}
as long as both $k\ga^2<d$ and ${k\choose 2} \ga^2<d$. This shows $K_2<\infty$ and ends the proof of the theorem. 
\end{proof}

\subsection{\textcolor{black}{$k$-Clique Counts}} \label{sec:KCC}

A $k$-clique is a complete subgraph with $k$ vertices. In what follows, we will denote by $Q_k$ the number of $k$-cliques in a realization of the FGN model. Clearly, $Q_k$ is a random variable and an important motif count for the FGN. We have the following results on the asymptotic expected value of $Q_k$.

\begin{theorem} \label{thm:KCC}
For integer $k\geq 3$, suppose that $\nu= \ga^2/d < \min((k-1)^{-1},  2((k-1)(k-2))^{-1})$. Then there exists a finite constant $C(\gamma, d, k, \rho)$ such that
\begin{align*}
\EE[Q_{k}] \sim C(\gamma, d, k, \rho)  n^{1+\nu} (1+o(1)).
\end{align*}
\end{theorem}

\begin{proof}
The central arguments are on similar lines to those in the case of hub-and-spoke counts,  so we provide a sketch herein. Since all possible edges among the vertices are connected in a clique, one  has that  
\begin{align*}
\EE[Q_{k+1}] = \frac{n^{k+1}}{(k+1)!} \EE \int_{\Omega^{k+1}} \exp( -\sum_{\{i,j\}\subset [k+1], i\neq j} \|x_i-x_j\| )  \prod_{i=1}^{k+1} M^\ga( dx_i) .
\end{align*}
By a routine computation as before, one derives 
\begin{align*}
\EE[Q_{k+1}]= (1+o(1)) \frac{n^{k+1}}{(k+1)!} \sigma^{(k+1)d-\ga^2} \int_{(\Omega/\sigma)^{k+1}}  \prod_{\{i,j\}\subset[k+1], i\neq j} \exp( -\|x_i-x_j\|^2) \|x_i-x_j\|^{-\ga^2} \prod_{i=1}^{k+1}dx_i.
\end{align*}
Isolating one variable, say $x_{k+1}$, one is led to the desired estimates as long as one can show that 
\begin{align*}
K':= \int_{\RR^{d k}} \exp( -\sum_{\{i,j\}\subset [k+1], i\neq j} \|x_i-x_j\|^2 -\sum_{i=1}^k \|x_i\|^2) \prod_{\{i,j\}\subset[k], i\neq j} \|x_i - x_j\|^{-\ga^2} \prod_{j\in[k]} \|x_j\|^{-\ga^2} \prod_{i=1}^{k} dx_i 
\end{align*}
is finite. Since $K'<K<\infty$, the proof is complete. 
\end{proof}

\subsection{\textcolor{black}{A heavy-tailed structural pattern}}
The expected behavior of the counts for the four motifs considered in the preceding sections - i.e., edges, triangles, hub-and-spokes and cliques (of all possible sizes) - suggests a distinct pattern. In all these cases, the expected count scales with the size parameter $n$ as $n^{1+\nu}$. We believe that this would hold true for \textit{all subgraph counts} in the FGN model. However, a mathematical proof of such a result is rendered difficult by the technical challenges posed by the GMC measure, not the least of which are the heavy-tailed phenomena induced by it, a topic that we take up for discussion below.

From the proofs of \cref{thm:EC} to \cref{thm:KCC}, it appears that the expected edge, triangle, hub-and-spoke and $k$-clique counts might not be finite for values of $\nu$ beyond the thresholds in the respective theorems, due to the non-convergence of the integrals appearing in respective factors $C(\ga,d)$. Such behavior, in fact, would be commensurate with the general heavy-tail character of statistics naturally associated with the FGN model, particularly in the regime where fractal effects become significant (i.e., for relatively  large $\nu$). The blow-up of expectations may be explained heuristically by the presence of a moderately small likelihood of extremely large values, a characteristic of heavy-tailed distributions. Such heavy-tailed features also renders other possible techniques for theoretical study of the FGN model to be ineffective; for example it is challenging to make effective use of second moment based arguments to rigorously establish concentration phenomena for FGN statistics in the single-pass observation model.  Providing a theoretical characterization of concentration phenomenon, for both the edge and triangle count, is an extremely interesting problem for future work.

\subsection{Spectrum of the FGN}
The spectrum of the adjacency matrix or the Laplacian matrix of a graph is considered to be one of its most fundamental aspects, and is of independent academic interest; see, for example, ~\cite{farkas2001spectra, Chung6313, erdHos2013spectral, tran2013sparse, van2016random, benaych2019largest}. In this section, we undertake an empirical investigation of the eigenvalue distribution of the Laplacian matrix of the FGNs and illustrate several intriguing properties.

We recall that for an undirected graph with $n$ vertices, its Laplacian matrix is defined as $L:=D-A$, where $D$ is the (diagonal) degree matrix and $A$ is the (symmetric) adjacency matrix. We now describe the experimental setup in detail.  For a given value of $\nu$, after generating an FGN with $n=5000$, we compute the Laplacian matrix and obtain its eigenvalues. We then count the multiplicity of the eigenvalues. At this point, we noticed that there are some eigenvalues that have extremely large multiplicities. We collected those eigenvalues apart and first plot the remaining eigenvalues. This corresponds to the right column of~\cref{fig:eigplot}. Then, we superimposed the eigenvalues with large multiplicities on top of the previous histogram. This corresponds to the left column of~\cref{fig:eigplot}. This process is repeated for increasing values of $\nu$, which corresponds to the rows of \cref{fig:eigplot}.

There are several observations to be made regarding the obtained histogram. First, we observe that there appears to be a singular component of the spectrum which is characterized by extremely large peaks (i.e., eigenvalues with extremely large multiplicity) scattered through the entire support of the histogram, as evident in the left column of ~\cref{fig:eigplot}.
Next, there seems to be an absolutely continuous component of the spectrum, whose histogram is separately plotted in the right column of ~\cref{fig:eigplot}. However, this component also seems to exhibit relatively moderate  peaks and  dispersed through its support. In view of the peaks of widely different sizes and the irregular contour of the absolutely continuous part, the spectrum of the Laplacian appears to exhibit a multi-scale structure.  \textcolor{black}{Such a nuanced spectral behavior is markedly different from that of the relatively simple spectrum exhibited by the Erd\H{o}s-R\'enyi random graph Laplacian. In particular, in~\cref{fig:EReigplot} we plot the Laplacian spectrum of an Erd\H{o}s-R\'enyi random graph. The edge probability was set to be $0.5$, $0.1$ and $0.01$ thereby covering a class of sparse random graphs for comparison. We notice that as the graph gets sparse, the spectrum exhibits peaks, however, they are not as significant as that exhibited by the FGN. }

\begin{figure*}[t!]
\centering
\includegraphics[scale=0.45]{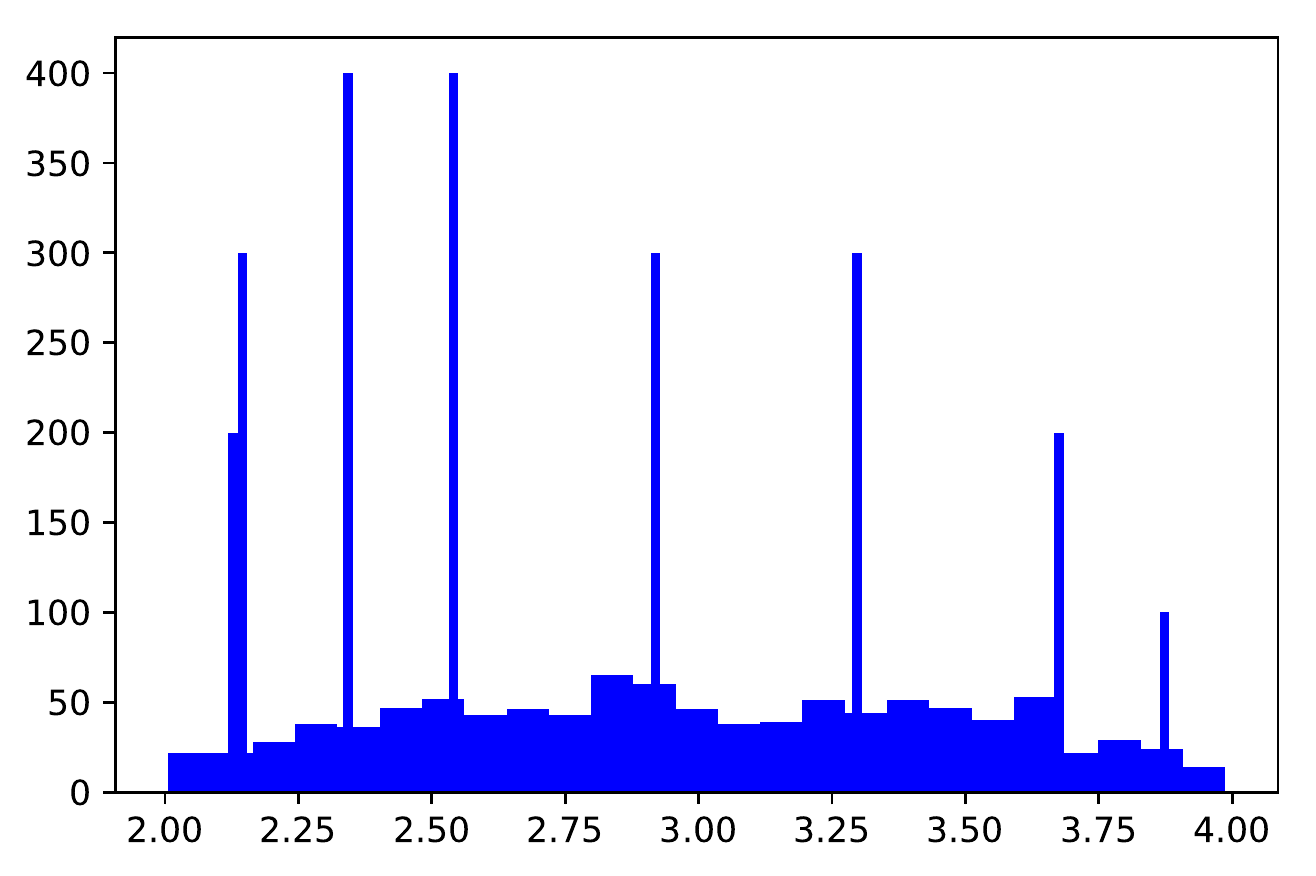}
\includegraphics[scale=0.45]{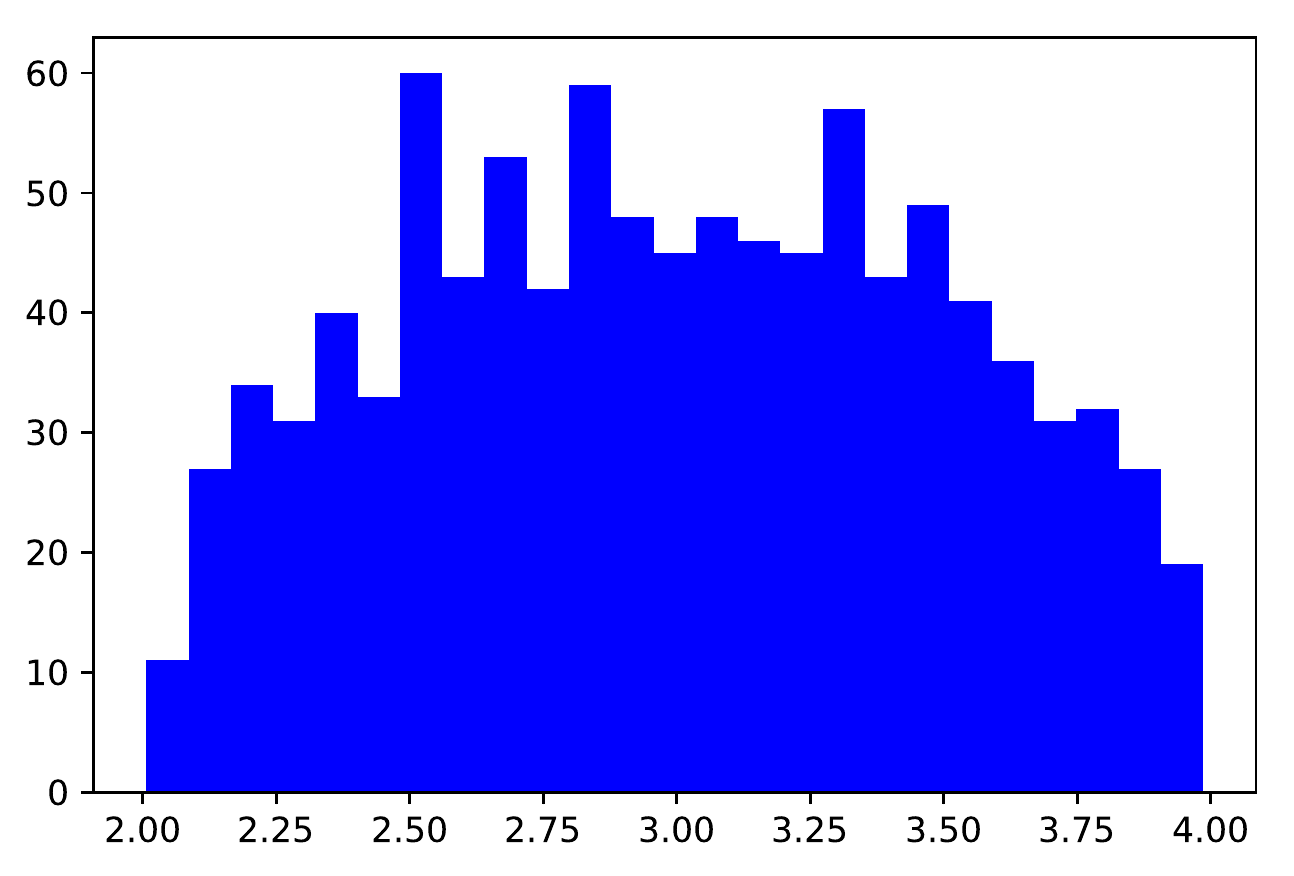}\\
\includegraphics[scale=0.45]{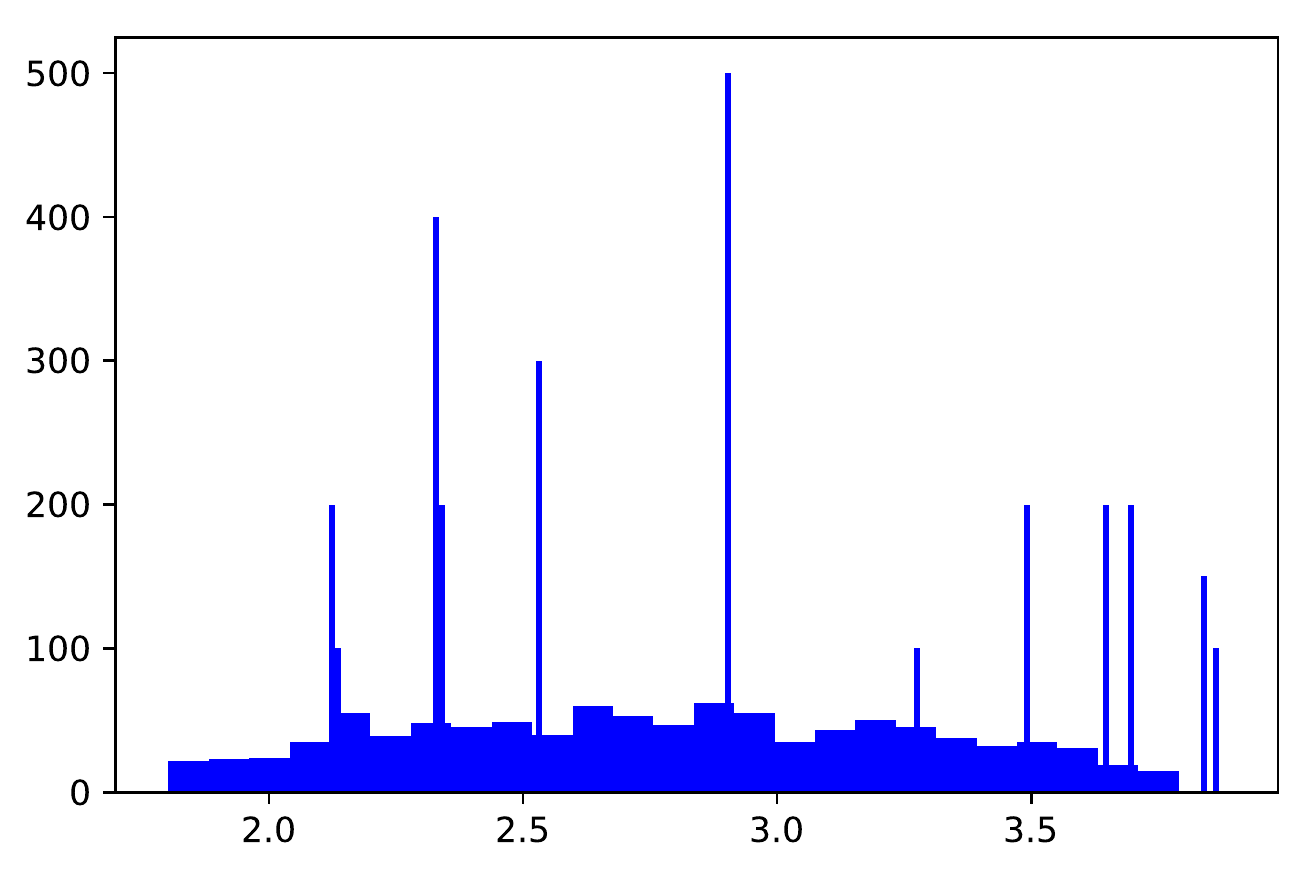}
\includegraphics[scale=0.45]{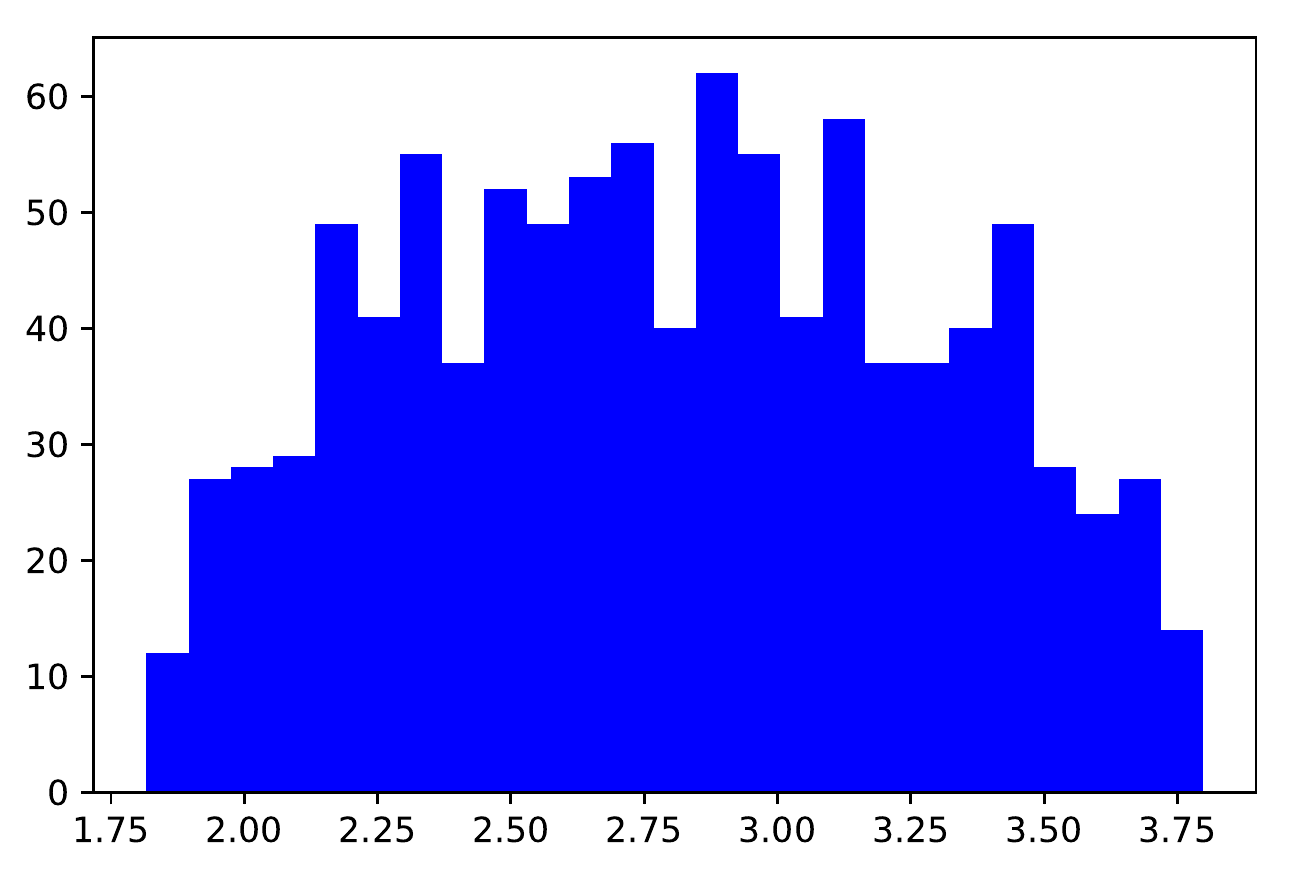}\\
\includegraphics[scale=0.45]{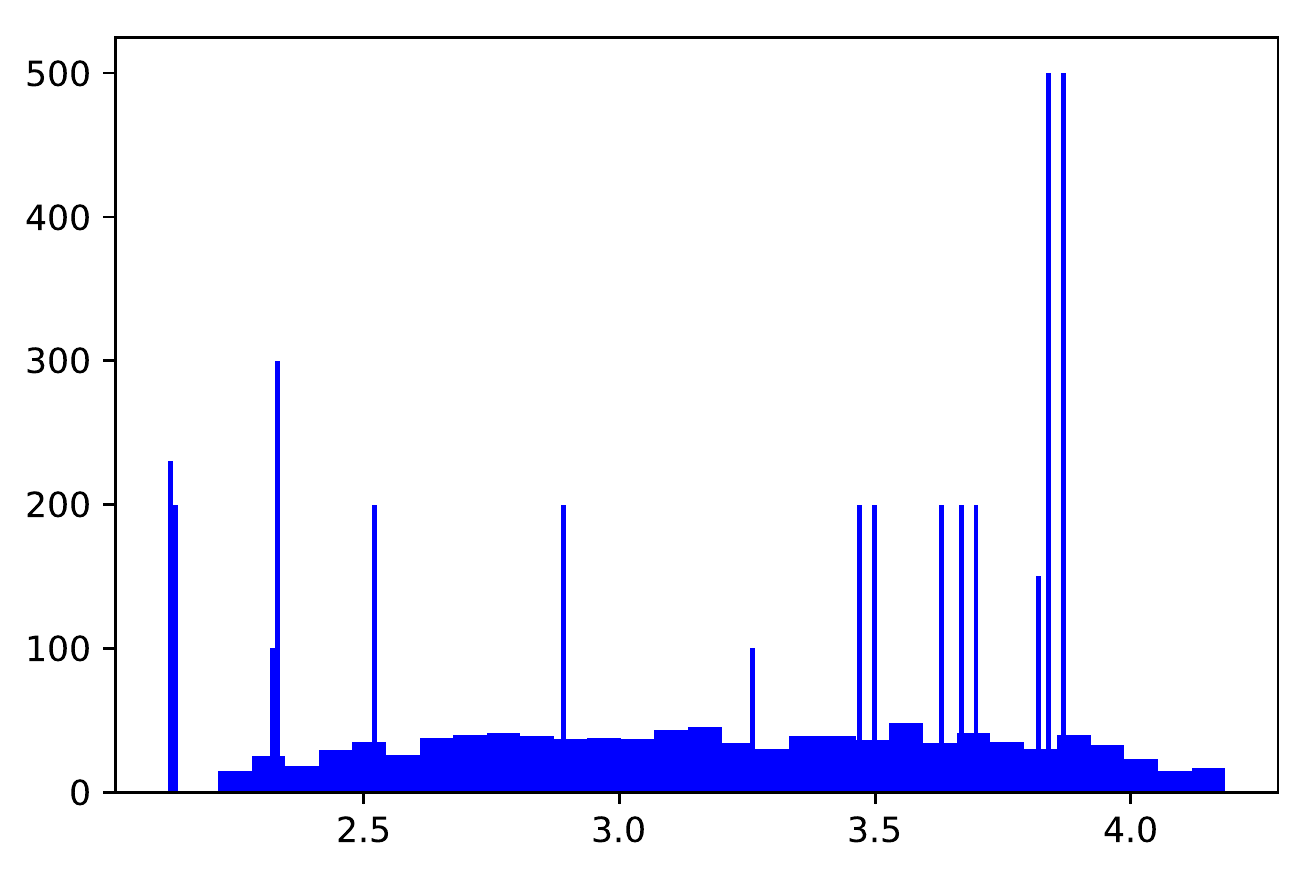}
\includegraphics[scale=0.45]{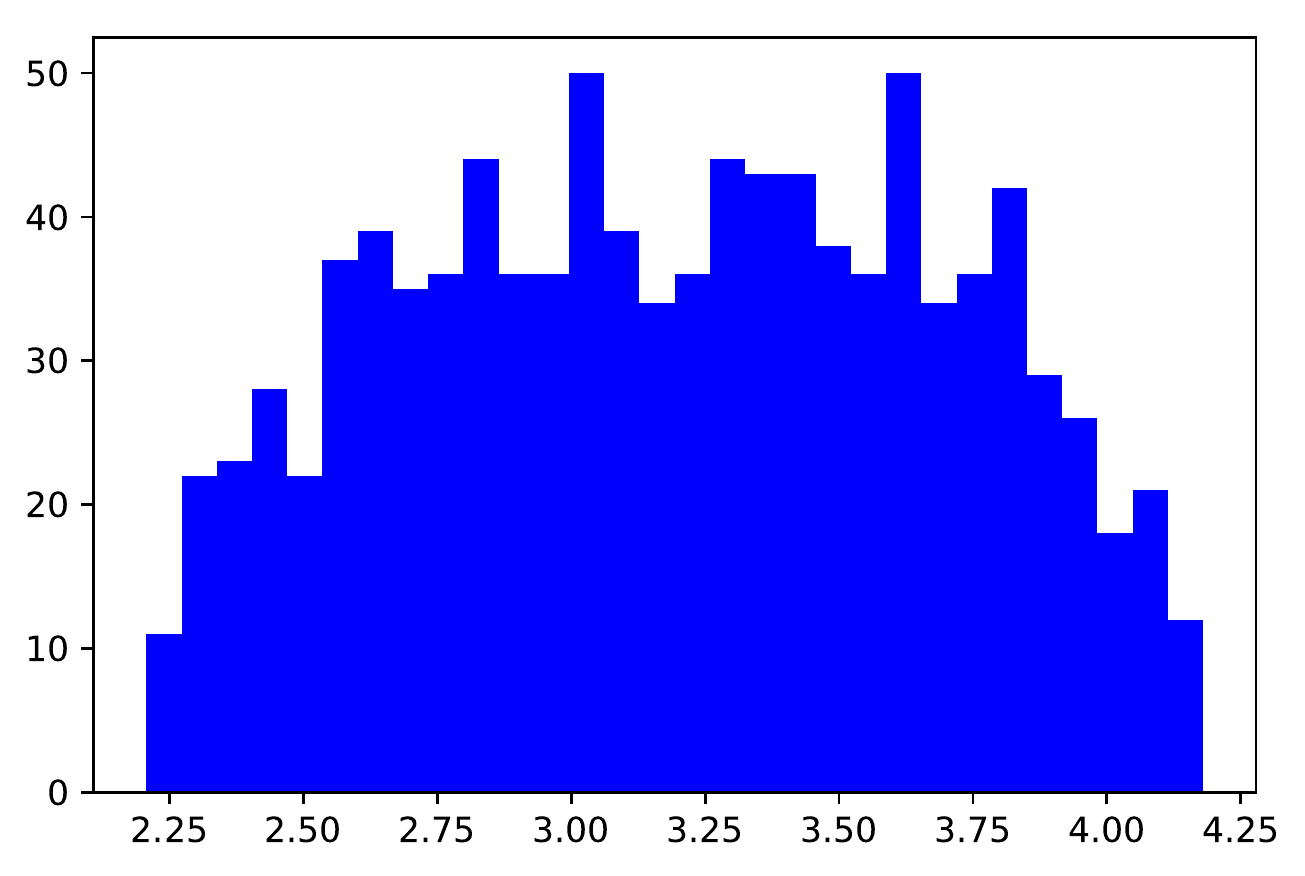}\\
\includegraphics[scale=0.45]{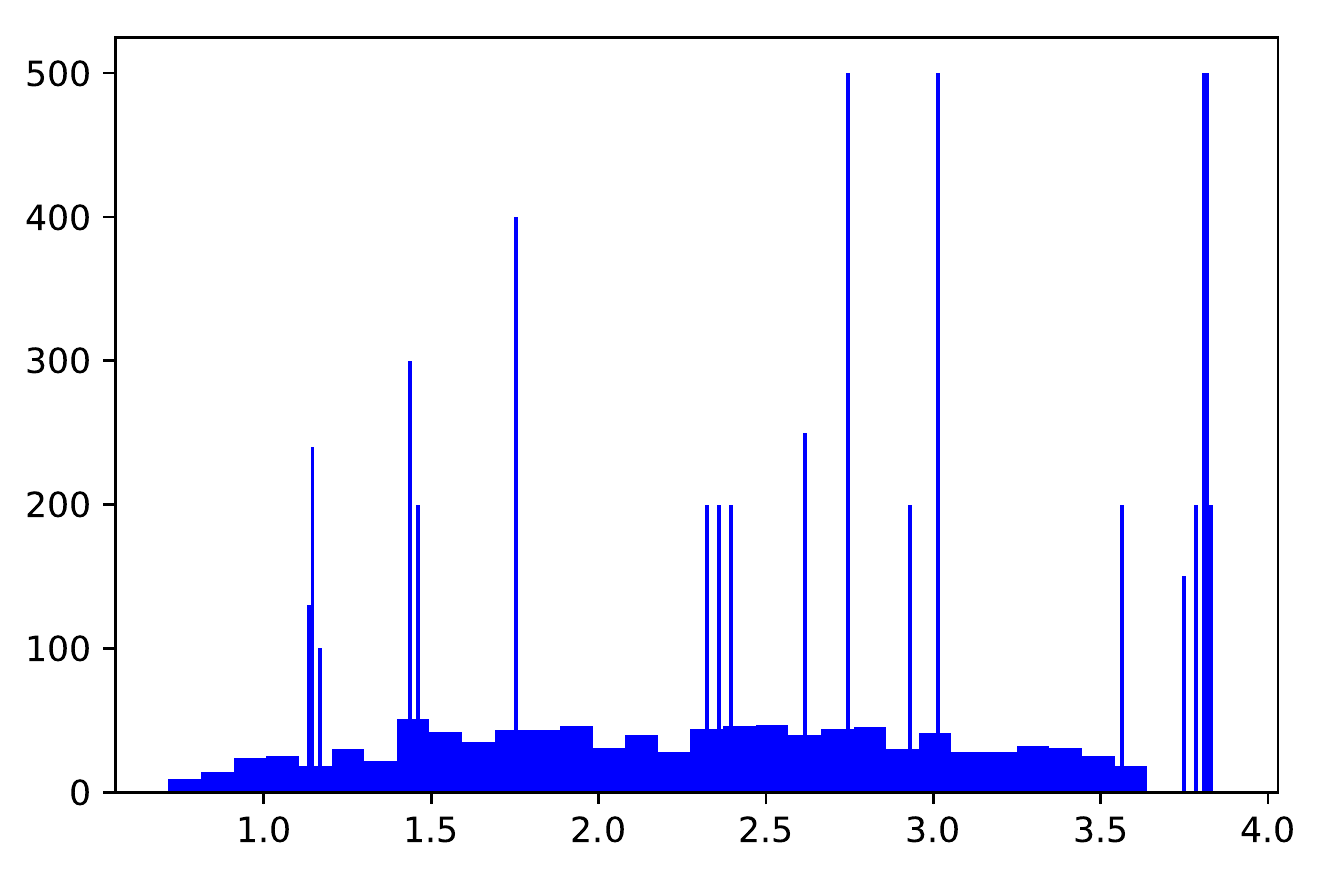}
\includegraphics[scale=0.45]{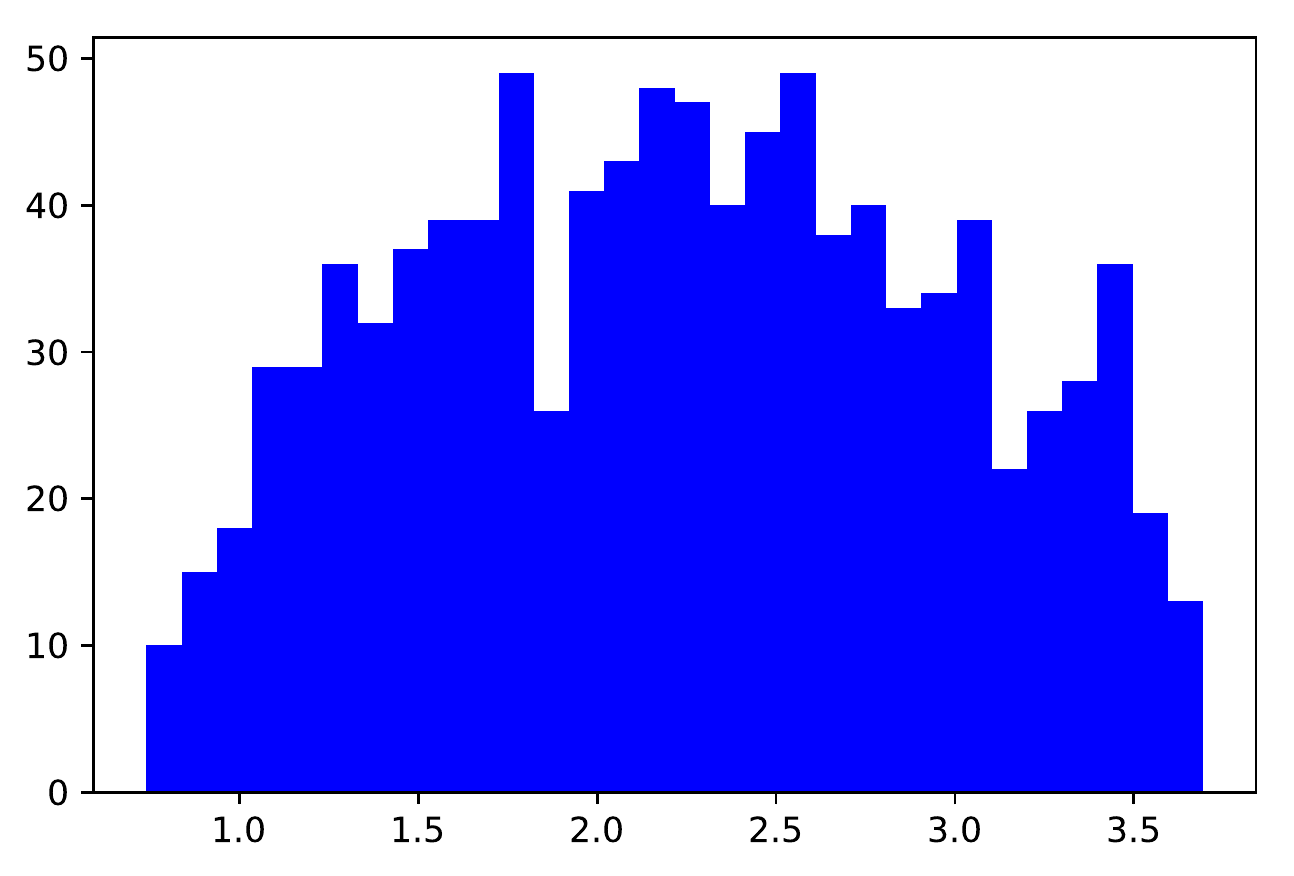}\\
\includegraphics[scale=0.45]{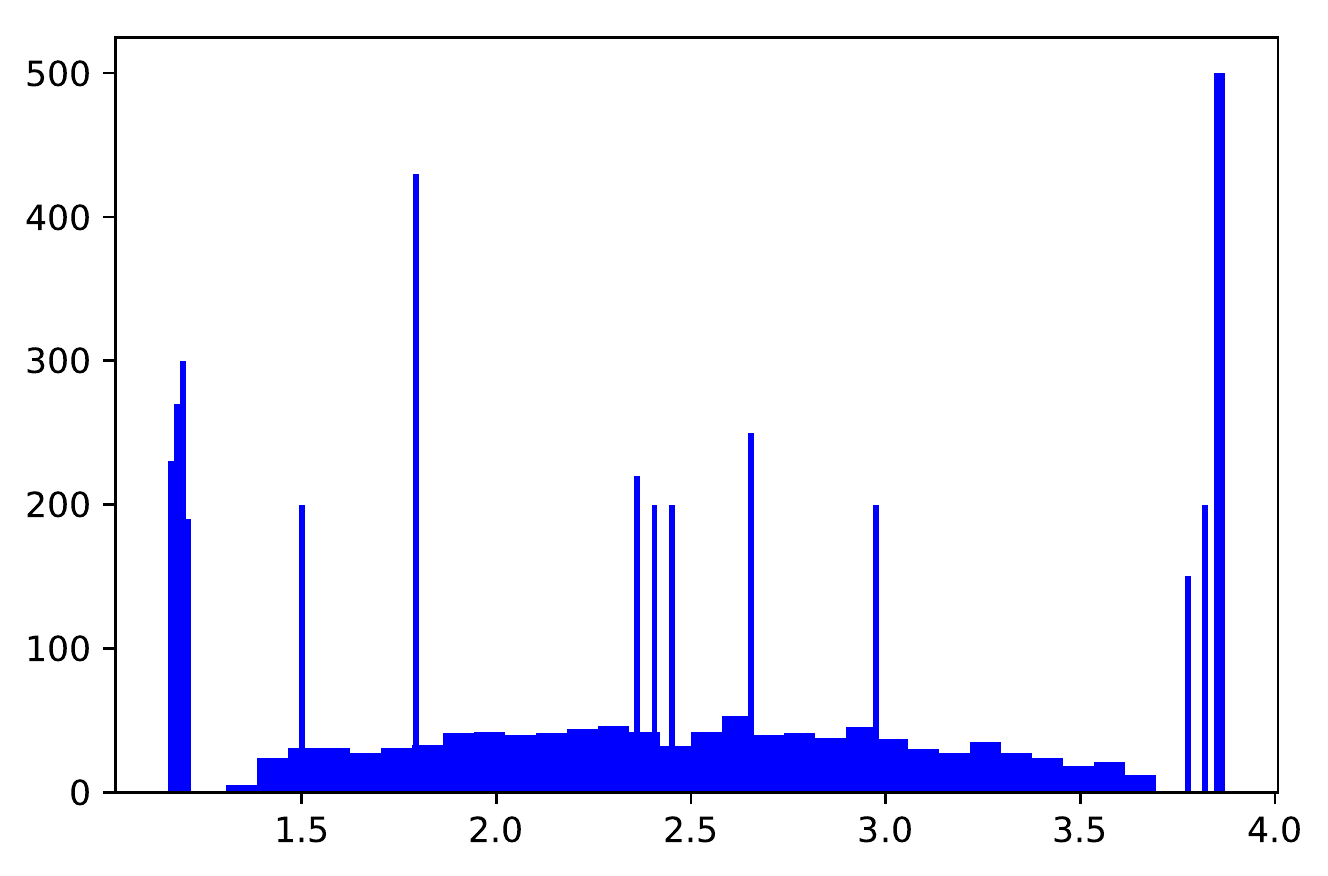}
\includegraphics[scale=0.45]{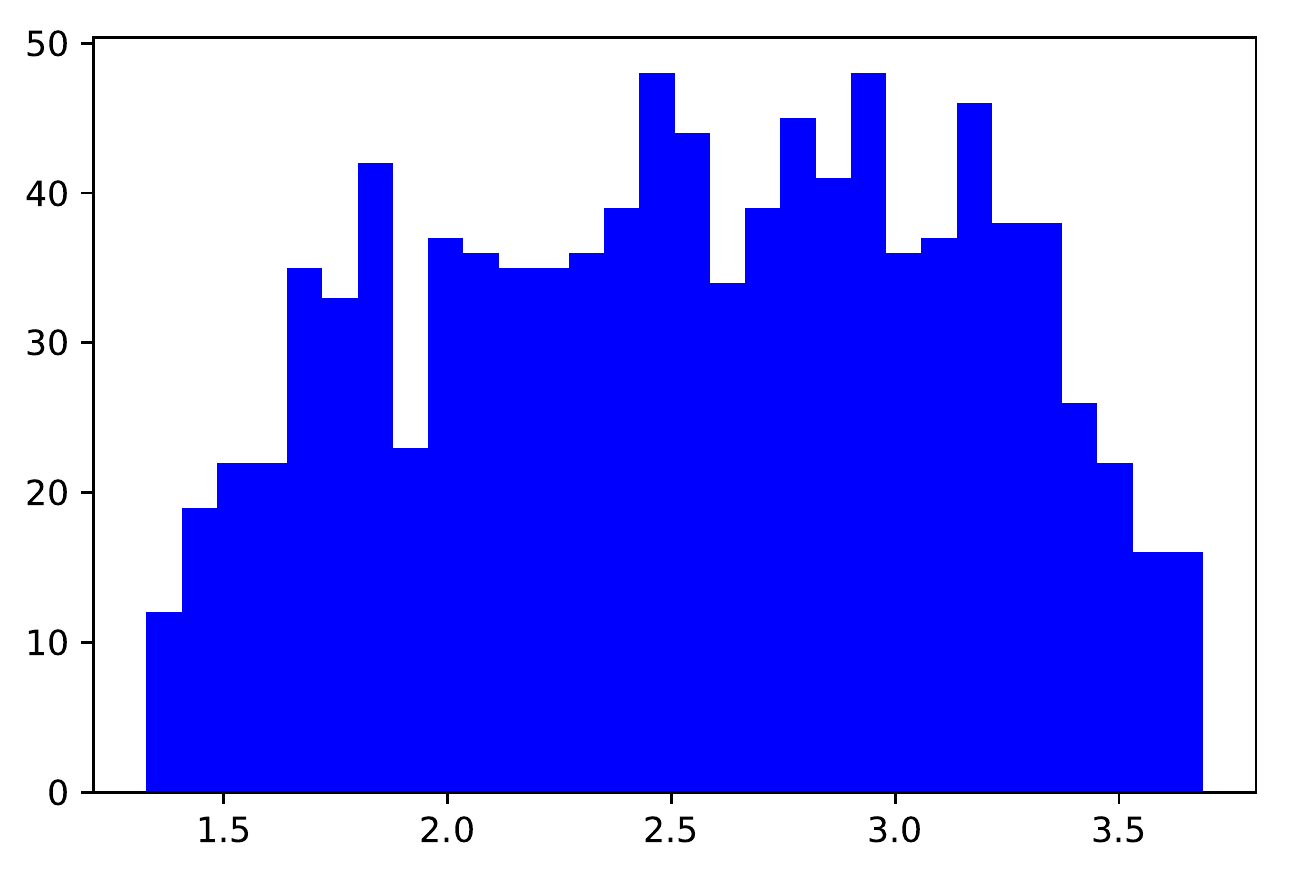}\\
\caption{Eigenvalue distribution of the Laplacian matrix of FGN: Left column corresponds to the entire spectrum and right column corresponds to a zoomed-in part of the spectrum without the tall peaks. As we move from top to bottom, the value of $\nu$ increases.}
\label{fig:eigplot}
\end{figure*}

\begin{figure}[h!]
\centering
\includegraphics[scale=0.35]{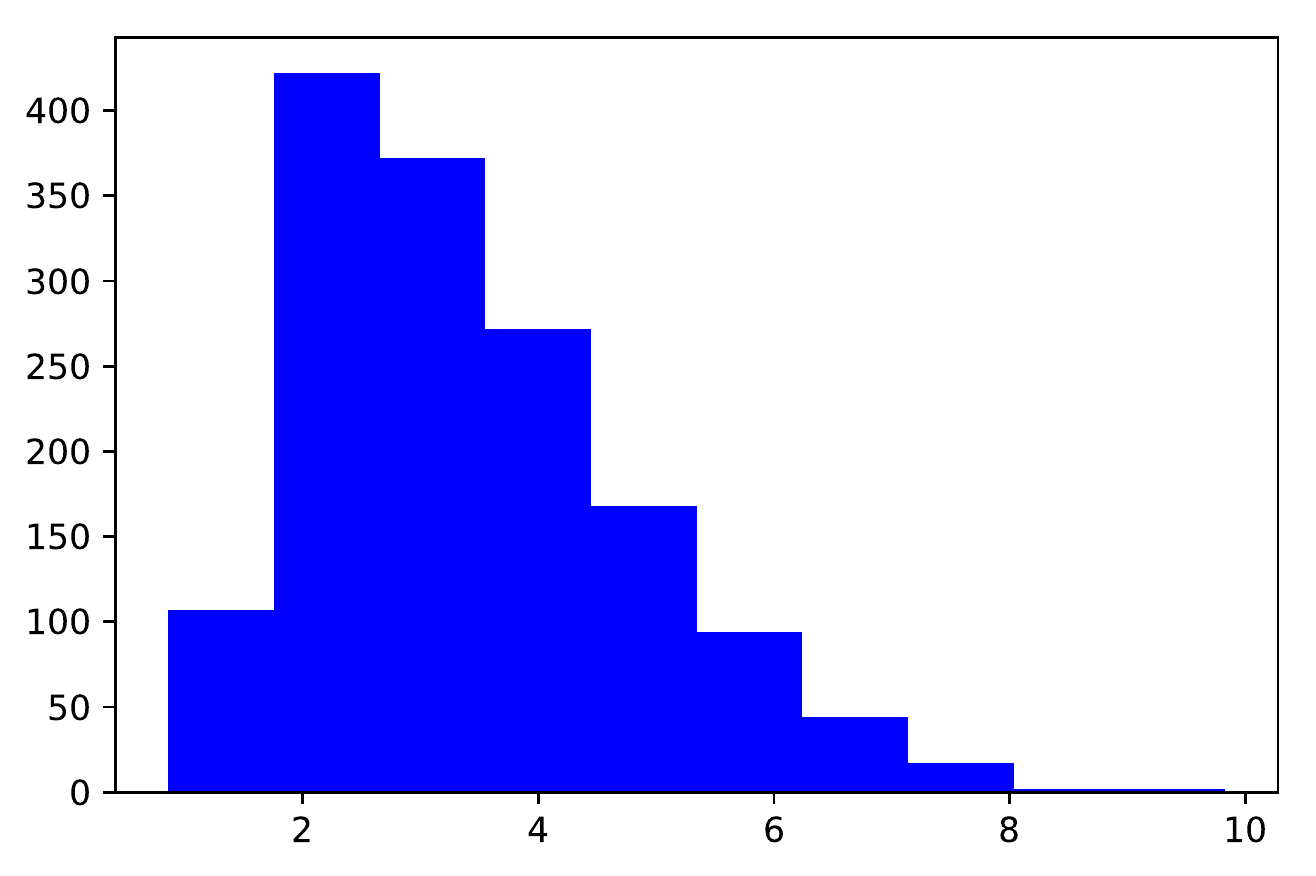}
\includegraphics[scale=0.35]{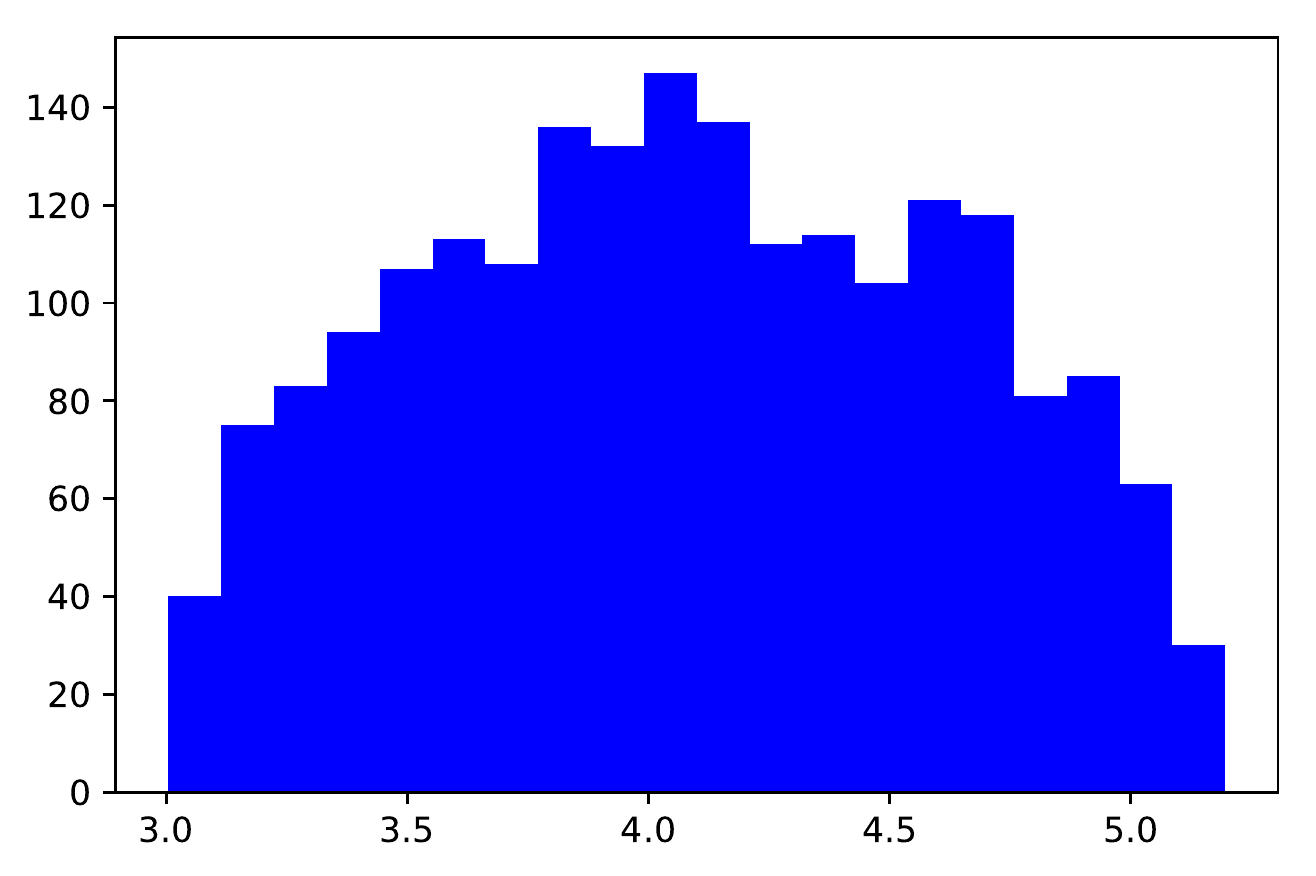}
\includegraphics[scale=0.35]{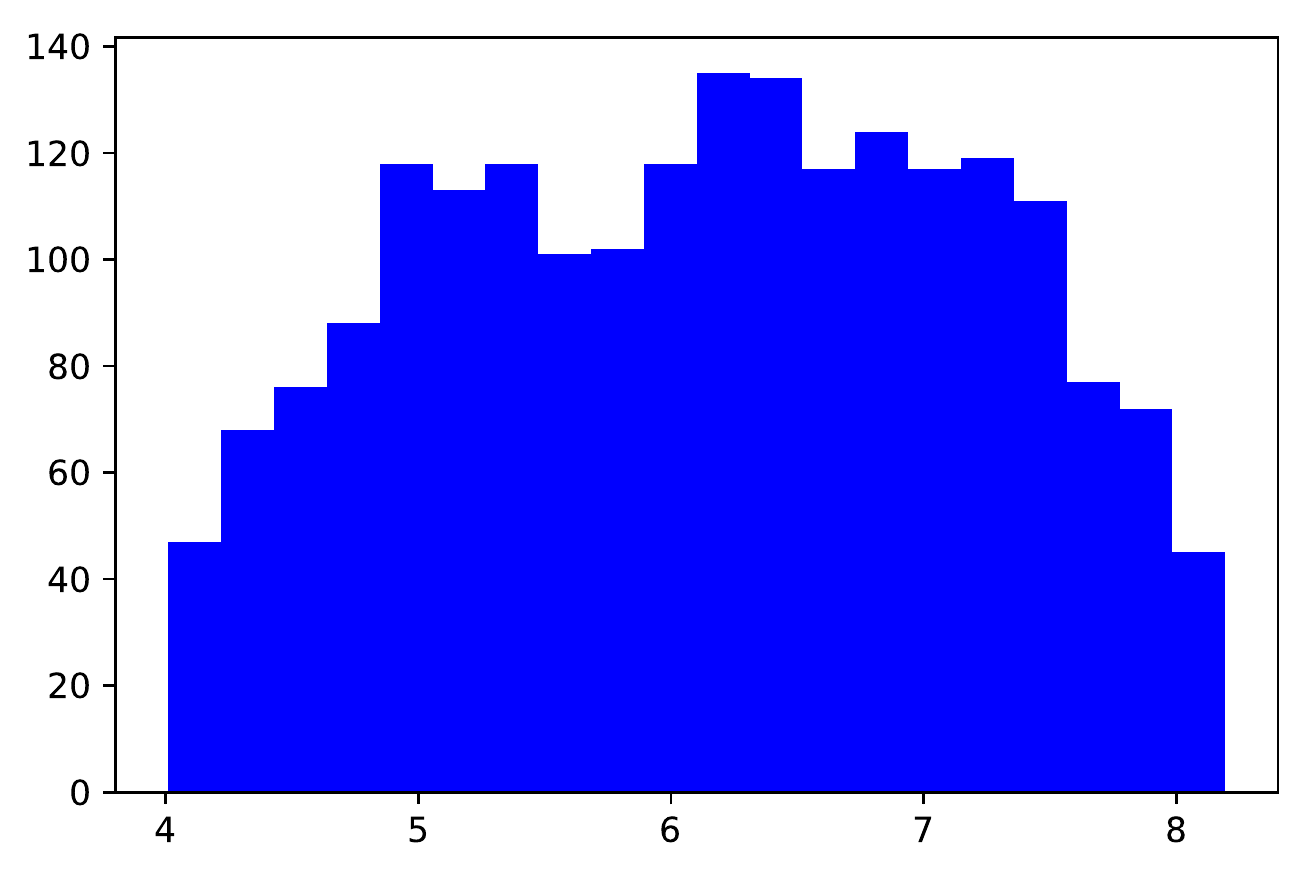}
\caption{Eigenvalue distribution of the Laplacian matrix of Erd\H{o}s-R\'enyi  random graph: The number of nodes is $1000$. The probability of an edge were set to be $0.5$ (left), $0.1$ and $0.01$ respectively.}
\label{fig:EReigplot}
\end{figure}

\subsection{Estimating the fractality parameter $\nu$.}
\label{sec:estimation}
In this section we propose an estimator for the crucial fractality parameter $\nu$ in the FGN model.  To this end,  we will focus on small subgraph counts in the network, and utilize our analysis of edge counts from~\cref{sec:EC}) in order to detect fractal structures. In this section, we will work in the setting $\ga^2<d$, so that the results of \cref{sec:EC} would be applicable. Interestingly, this would correspond to the so-called $L^2$ \textit{regime} in the theory of GMC, where many of the mathematical technicalities are known to be relatively more tractable.

In the single-pass observation model, we consider the statistic
\begin{equation} \label{eq:nu-estimator-single}
\nh_{\single}:=\frac{\log \D}{\log N}-1
\end{equation}
as an estimator for $\nu$, where $\D$ is the edge count of the FGN.
%
To see why $\nh_\single$ is a good estimator, we write
\begin{align*}
\log \D &= \log(\frac{\D}{n^{1+\nu}} \cdot n^{1+\nu}) \\
&= (1+\nu) \log n +   \log ({\D}/{n^{1+\nu}})\\
& = \log n \left[ 1 + \nu + \frac{ \log ({\D}/{n^{1+\nu}}) }{ \log n}  \right].
\end{align*}
But \cref{thm:EC} and \cref{fig:edgecounts} suggest that ${\D}/{n^{1+\nu}}$ is a $\Theta(1)$ quantity, which indicates that $\frac{\log \D}{\log n}  \sim 1+ \nu$ as $n \to \infty$. But we may now make use of the fact that we have $\log N = \log n (1+o(1))$ (as in \cref{eq:N-n}), which, coupled with the last equation, implies that $\frac{\log \D}{\log N}$ is approximately $1+\nu$, or equivalently $\frac{\log \D}{\log N}-1$  is approximately $\nu$ in the regime of large size parameter $n$, as desired.

In the multi-pass observation model,  we have $m$ i.i.d. samples of the FGN with $\D_i$ and $N_i$ being the edge count and the vertex count of the $i$-th sample. Then we may define $\ol{\D}$ as the mean edge count $\ol{\D}:=\frac{1}{m}\sum_{i=1}^m \D_i$ and $\ol{N}$ as the mean vertex count $\ol{N}:=\frac{1}{m}\sum_{i=1}^m N_i$. We observe that, in the regime of large $m$, the mean edge count $\ol{\D}$ and the mean vertex count $\ol{N}$ strongly concentrate around their expectations. As such, in the regime of large $m$ we have $\ol{\D}=\EE[\D](1+o_P(1))$ and $\ol{N}=\EE[N](1+o_P(1))=n(1+o_P(1))$. This, coupled with our analysis of the single-pass setting, naturally suggests consideration of the following estimator of $\nu$ in the multi-pass observation model:
\begin{equation} \label{eq:nu-estimator-multi}
\nh_{\multi}:=\frac{\log \ol{\D}}{\log \ol{N}}-1 .
\end{equation}
The efficacy of $\nh_\multi$ as an estimator for $\nu$ follows from the afore-mentioned asymptotics of $\ol{\D}$ and $\ol{N}$ in the large $m$ regime, which lead us to deduce that
\begin{align*}
\nh_{\multi} =&\frac{\log \ol{\D}}{\log \ol{N}}-1 \\
 =&\frac{\left( \log \EE[\D] + \log (1+o_P(1)) \right)}{ \left( \log \EE[N]  + \log(1+o_P(1))  \right) - 1 }\\ =&(1+\nu) + o_P(1) -1 \\
 =& \nu +o_P(1),
 \end{align*}
 where, in the last step, we have used the asymptotics of $\EE[\D]$ in \cref{thm:EC}, coupled with a small parameter expansion of $\log(1+x)$. The estimator $\nh$ has the form of a log-log plot between network observables and system size. Such log-log plots have been used effectively in studying growth exponents and fractal behavior in the phenomenological literature, and therefore are well-motivated and thoroughly contextualized in the setting of fractal networks.

\subsection{Detecting the presence of fractality.}
\label{sec:testing}
We examine the presence of fractality in the network by examining whether the combinatorial data of the graph points to the occurrence of such structures. As argued earlier, in the context of FGN this would entail determining whether $\nu=0$ (absence of fractality), and compare it with the alternative possibility $\nu \ge \nu_0$ for some given threshold $\nu_0$ (presence of a substantive degree of fractality). Choosing a positive threshold for the alternative, separated from 0, is a natural framework, because as discussed earlier the FGN interpolates \textsl{continuously} between homogeneity and gradually increasing fractality. In this section, we will once again work in the setting $\ga^2<d$, so that the results of \cref{sec:EC} would be applicable. As observed earlier, this would correspond to the so-called $L^2$ \textit{regime} in the theory of GMC.

In the single-pass observation model, we again exploit our analysis of edge counts for this purpose. We recall that when $\nu=0$, that is for the Poisson random geometric graph, $\D$ is a sum of indicators of all possible edges on the vertex set. Since edges are usually formed when the underlying points $x_i, x_j, x_k$ are close to each other at the scale $\sigma$, and since $\sigma=O(n^{-1/d})$, we may conclude that $\D$ is a sum of a large (and Poisson) number of weakly dependent random variables. As such, it can be well-approximated by a compound Poisson random variable, which in turn admits a normal approximation with appropriate centering and scaling (c.f. \cite{penrose2003random}, \cite{van2016random}).

The upshot of this is that under $\nu=0$, for large $n$, the normalized edge count $\frac{\D-\EE[\D]}{\sqrt{\mathrm{Var}}[\D]}$ is approximately normally distributed (\cite[Theorem 3.4]{penrose2003random}). Under $\nu=0$, the edge count is known to satisfy ${\mathrm{Var}}[\D] = C(d,\rho) n (1+o(1))$ (c.f. \cite{penrose2003random}), which implies the approximate upper tail bound $$\PP[\D \ge C_2(0,d)\rho^2 n + t] \le C\exp\left(-\frac{ct^2}{n}\right).$$ This suggests that, under $\nu=0$, the probability  $\PP[\D \ge n^{1+\frac{1}{2}\nu_0}] \le \exp(-cn^{1+\nu_0})$. On the other hand, under the alternative we have
$$\EE[\D]=C(\ga,d,\rho)n^{1+\nu}(1+o(1)) \ge C(\ga,d,\rho)n^{1+\nu_0}(1+o(1))  \gg n^{1+\nu_0/2},$$
 as $n \to \infty$.  This suggests that the threshold $n^{1+\nu_0/2}$ for the edge count separates the $\nu=0$ and $\nu \ge \nu_0$ settings.

However, in our observation models, we do not have direct access to the latent size parameter $n$. Nonetheless, as discussed in \cref{sec:sizepar}, the observed network size $N$ provides a good approximation of $n$ upto an $O(1)$ multiplicative factor. Since $\D$ under the null and the alternative hypotheses are orders of magnitude (in $n$) apart (which is a consequence of the positive separation between the null and the alternative), we can use $N$ as a substitute for $n$ for obtaining a separation threshold. Thus, in the single-pass observation model,
\begin{equation} \label{eq:test-single}
\text{Declaring the presence of fractality if }  \D > N^{1+\frac{1}{2}\nu_0}
\end{equation}
would provide a detection procedure for fractality with good discriminatory power. In the multi-pass observation model, we make use of the mean edge count $\ol{\D}=\frac{1}{m}\sum_{i=1}^m \D_i$ and the mean vertex count $\ol{N}=\frac{1}{m}\sum_{i=1}^m N_i$. In the regime of large $m$, they concentrate strongly around their respective means, with Gaussian CLT like
effects. Thus, in the multi-pass observation model
\begin{equation} \label{eq:test-single}
\text{Declaring the presence of fractality if }  \ol{\D} > {\ol{N}}^{1+\frac{1}{2}\nu_0}
\end{equation}
would provide a detection procedure for fractality with good discriminatory power.

\section{\textcolor{black}{Fractality of  FGN : a phenomenological perspective}}\label{sec:perspective}
While there has been interest in the presence of `fractal phenomena' in networks in the scientific literature, this has not been developed as a rigorous mathematics discipline. 
We envisage our work as a step towards development of a general mathematical theory of fractal phenomena in networks. In this context, we explore in the present section the fractal features of the FGN model from a real-world phenomenological perspective, and demonstrate that many of the indicators of fractal properties that are popular in the scientific literature arise in the context of the FGN. This would substantiate the FGN as a naturally relevant model that fits into the study of fractal phenomena in networks.

\textbf{Anomalous growth exponents.} Fractal behavior is understood to appear in physical systems \textit{at criticality}, which refers to the critical point (in terms of a relevant driving parameter of the system, such as temperature for spin systems) where a \textit{phase transition} takes place. There has been extensive research in the physical sciences on the emergence of fractal phenomena at criticality (c.f., \cite{suzuki1983phase,stella1989scaling,isichenko1992percolation}). For the purposes of the present article, we will content ourselves with the excellent overview of the main phenomena presented in \cite{stinchcombe1989fractals}.

Onset of criticality in statistical mechanical models, such as  percolation  and Ising models on Euclidean lattices, is understood in physics to be accompanied with the appearance of anomalous growth exponents. In percolation, for example, the volume growth exponent of the infinite cluster (to be precise, the \textit{incipient infinite cluster}) at criticality is known to be $d - \alpha$, where the exponent $\alpha$ is the so-called \textit{length scaling exponent} of the system, and $d$ is the dimension of the ambient Euclidean lattice. This is quite different from the volume growth exponent $d$ for the infinite cluster in the \textit{supercritical regime}, where the large scale geometry of the infinite cluster is believed to be Euclidean \cite{stinchcombe1989fractals}. In a similar vein, the growth exponent of the total magnetic moment (with system size) of the Ising model of magnetism at its critical temperature is characterized by an analogous non-integer behavior via a similar length-scaling exponent \cite{stinchcombe1989fractals}.

In the setting of networks, for usual sparse networks, small subgraph counts are anticipated to grow in linear proportion to the `volume' (i.e., the number of nodes) of the network (the case of sparse Erdos-Renyi random graphs or the sparse Poisson random geometric graphs are concrete mathematical model that illustrates this fact). On the other hand, such counts growing non-linearly as volume raised to a fractional power indicates an anomalous growth behavior for motifs, and may be taken as an indicator of fractal properties.
For the FGN network, we demonstrate that the expected counts of important small subgraphs such as edges, triangles, hub-and-spoke motifs, and $k$-cliques grow as volume/size parameter raised to a power different from 1. This demonstrates the occurrence of anomalous growth exponents in the FGN model.

\textbf{Heavy tailed phenomena.} Heavy tailed behavior of distributions that are canonically associated with a model is believed to be another characteristic of fractal phenomena. These are often associated with power-laws and so-called \textit{scaling effects}. This connection is of interest particularly in finance \cite{mandelbrot2013fractals}, where the so-called \textit{scale-invariance} of power-law distributions has been classically understood as phenomenological behavior of income distributions captured by the celebrated Pareto's law. Fractal effects in financial time series are also believed to give rise to power-law behavior in its  \textit{power spectrum}, which in turn leads to the phenomenon of the so-called $1/f$ noise in such settings \cite{voss1992,voss1993,mandelbrot2013multifractals}
In the setting of networks, one of the simplest and most fundamental distributions associated with a network is its degree distribution. In this setting, power-law behavior of degree distributions are associated with the so-called \textit{scale-free networks}, which are believed to be of great interest in studying real-world network phenomena \cite{kim2007fractality}.
\begin{wrapfigure}{r}{0.5\textwidth}
\centering
\includegraphics[scale=0.5]{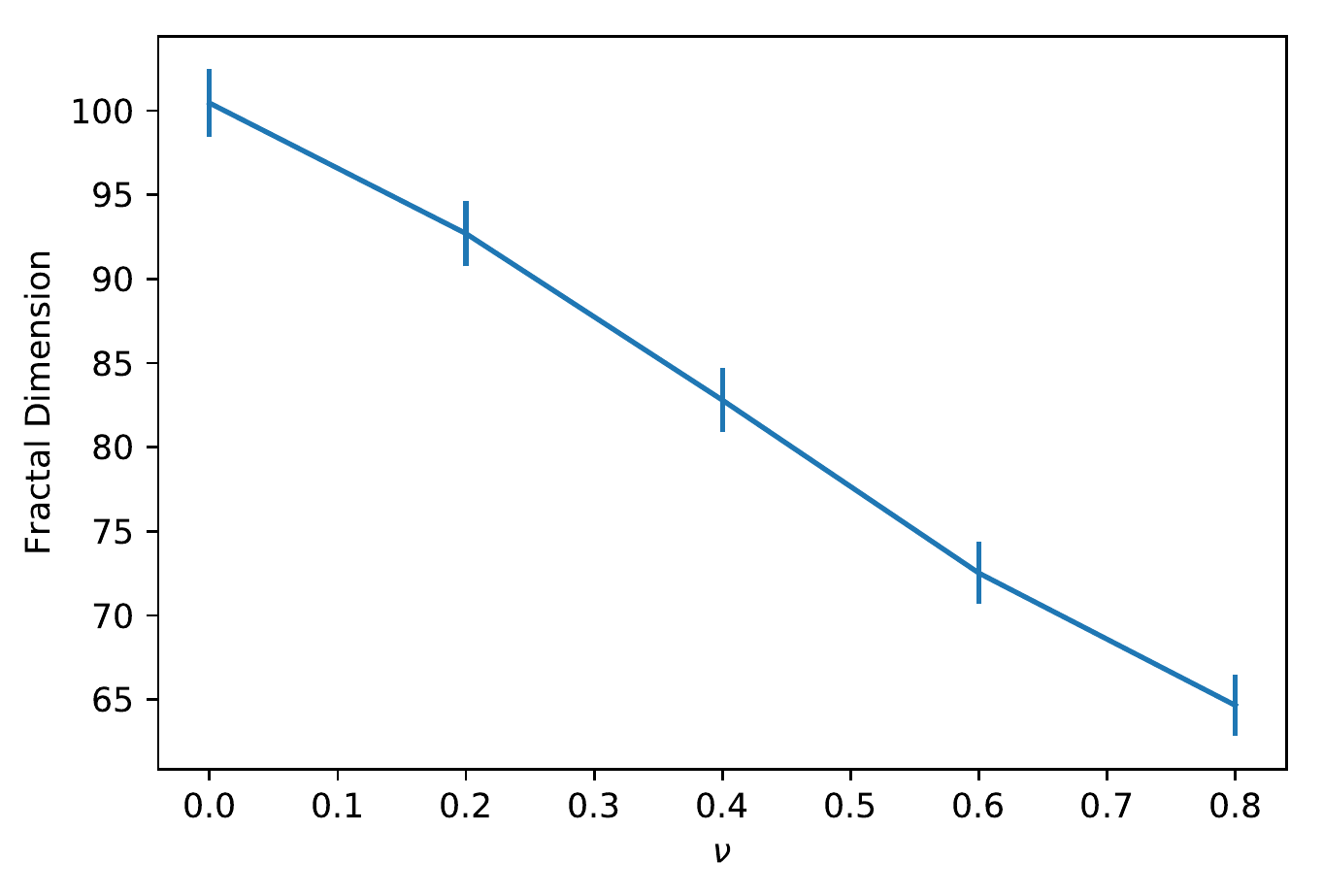}
\caption{Fractal Dimension computed based on the  scale-invariant renormalization procedure in~\cite{song2005self}.}
\label{fig:fractaldim}
\end{wrapfigure}

For most of the usual sparse network models, the degree distribution exhibits light tails. For instance, in the sparse Erdos-Renyi random graph model, the degree distribution is Poissonian. For the FGN network, we demonstrate empirically that the degree distribution appears to exhibit a power-law tail, as opposed to a Poisson-Binomial or Gaussian decay. In fact, the heavy-tailed nature of the degree distribution appears to become more accentuated as the value of the parameter $\nu$ increases --- for $\nu=0$, we have  a Poisson type tail, whereas the power-law behavior becomes more prominent with growing $\nu$. This not only demonstrates fractal characteristics in the FGN model, but also lends support to the salience of the parameter $\nu$ as the fractality parameter in the model. This is in addition to the observation that the `anomalous exponents' in various small subgraph counts appear to depend only on $\nu$ (and not on $\gamma$ or $d$ separately).

\textbf{Long range order.} Fractal behavior is also associated with the phenomena of long-range dependence and slow decay of correlations in physical systems \cite{stinchcombe1989fractals}. In the setting of the FGN, such long-range dependence is embodied in two different ways -- first, via the long-range dependence and slow decay of correlations in the Gaussian field $X$ that is underlying the latent GMC measure; and secondly, via the fact that the dependence on the vertex count on the GMC introduces additional global dependencies.

\textbf{Fractal dimension of networks.} We also investigate additional indicators of fractal behavior that are popular in the scientific literature, in particular the fractal dimension computed based on the scale-invariant renormalization procedure popularized by \cite{song2005self}. The box counting method is a standard method to calculate the fractal dimension of physical systems. In the context of networks, the work of~\cite{song2005self} showed that a certain renormalized version of the box counting provides a powerful tool for revealing the self-similar properties of heterogenous networks. 

In our experiment, we set the dimensionality $d=100$ and varied $\nu$ from 0 to 0.8 in steps of 0.2. We calculated the fractal dimension proposed by \cite{song2005self} based on averaging over 1000 instantiations of FGNs. In~\cref{fig:fractaldim} we plot the average along with the standard errors. We notice that the fractal dimension exhibits an inverse linear relationship with $\nu$. This is indicative of the fractality emerging in the network based on that from the underlying latent space. 


\subsection{Real-world Network data Analysis}
Finally, we reinforce our study of the FGN model by analyzing large-scale real-world network data that is believed to display fractal features. Specifically, we consider the \textsc{answers} and \textsc{flickr} datasets from the Stanford Large Network Dataset Collection~\cite{snapnets}. The \textsc{answers} dataset~\cite{leskovec2008statistical} depicts the interaction structure of users of the Yahoo! answers portal. The \textsc{flickr} dataset~\cite{kumar2010structure} is based on encoding the interaction structure of the Flickr photo-sharing website. The \textsc{answers} dataset has 598,314 nodes and 1,834,200 edges, while the \textsc{flickr} dataset has 584,207 nodes and 3,555,115 edge; hence they are relatively sparse networks. Furthermore, the works of~\cite{leskovec2008statistical} and ~\cite{kumar2010structure} have also demonstrated that these datasets exhibit power-law behavior. 

For our experiments, we fix the dimension of the latent space to be $d=2$ and use the parameter estimation procedure outlined above to estimate the fractality parameter $\nu$ for the above two network datasets.  We find that the estimated $\nu$ parameter for the \textsc{answers}  and \textsc{flickr} were 0.3234 and 0.1834 respectively highlighting the fact that the graphs exhibit fractal structures. With the estimated values of $\nu$, we generated FGN graphs and compare the properties of the real-world network and the simulated networks. Specifically, in Figure~\ref{fig:degreedistreal}, we plot the degree distribution of both networks. Furthermore, in Figure~\ref{fig:screeplot} we plot the scree plot -- a plot of the eigenvalues of the graph adjacency matrix, versus their rank (typically in the logarithmic scale). It has been shown analytically and empirically that such plots have a power-law behavior for several real-world networks~\cite{chung2003eigenvalues, farkas2011spectra}. From the plots we see that the generated FGN and the real-world network have great overlap in terms of the degree distribution and the scree plot, demonstrating the fit of the proposed model to real-world network datasets. 




\begin{figure}[t]
\centering
\includegraphics[scale=0.5]{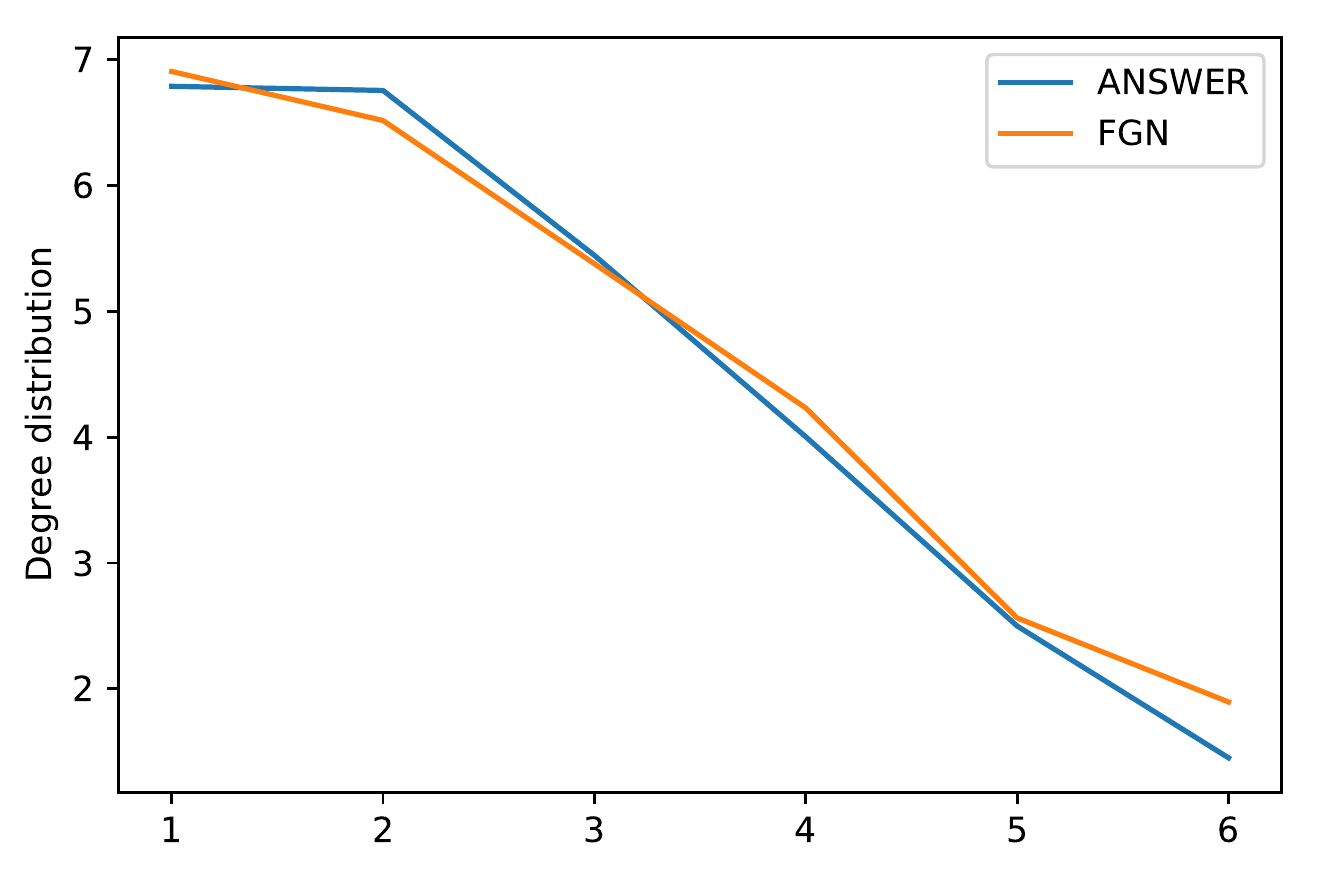}
\includegraphics[scale=0.5]{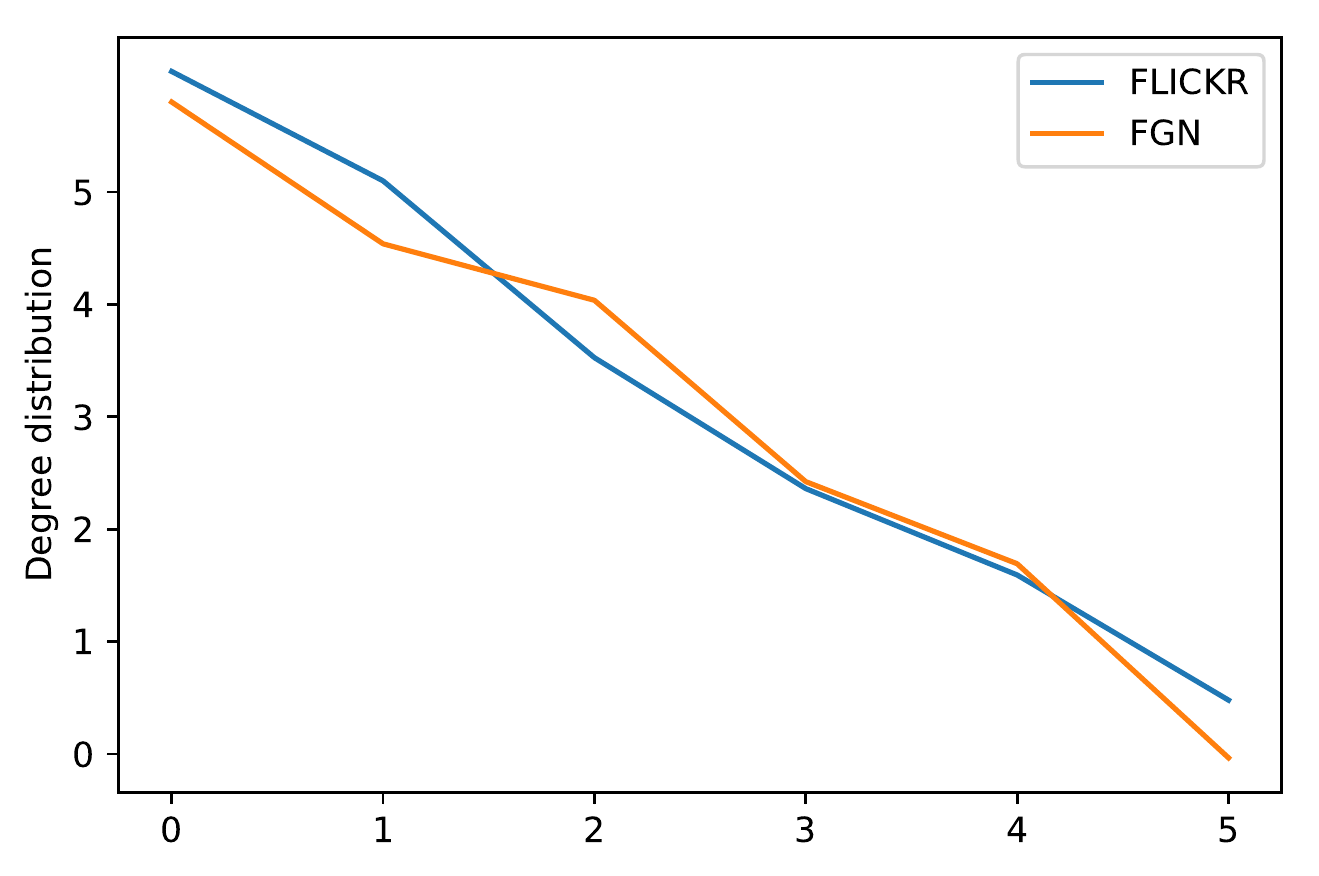}
\caption{Degree distribution comparison: Both the $x$ and $y$ axis are in log-scale.}
\label{fig:degreedistreal}
\end{figure}

\begin{figure}[t]
\centering
\includegraphics[scale=0.5]{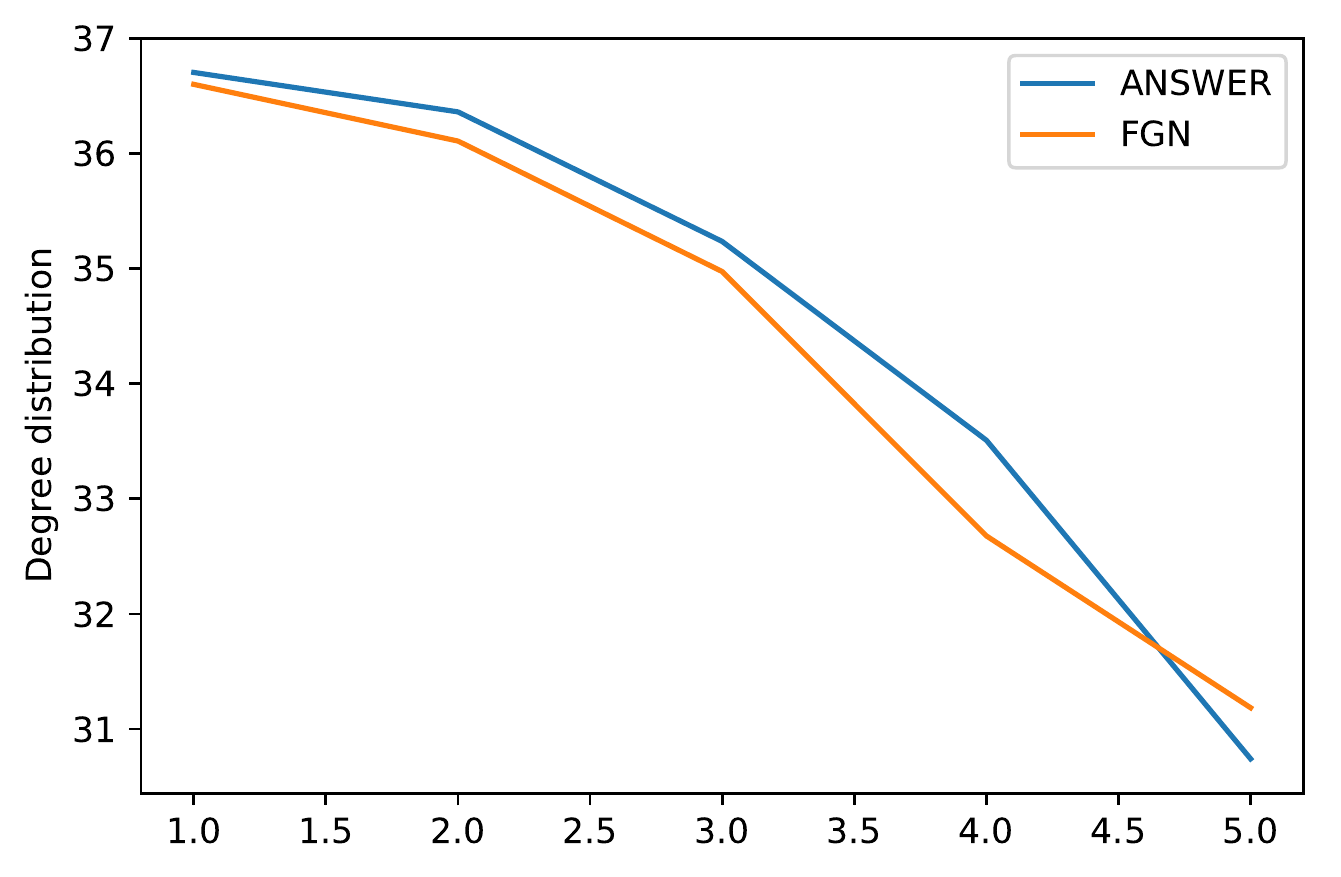}
\includegraphics[scale=0.5]{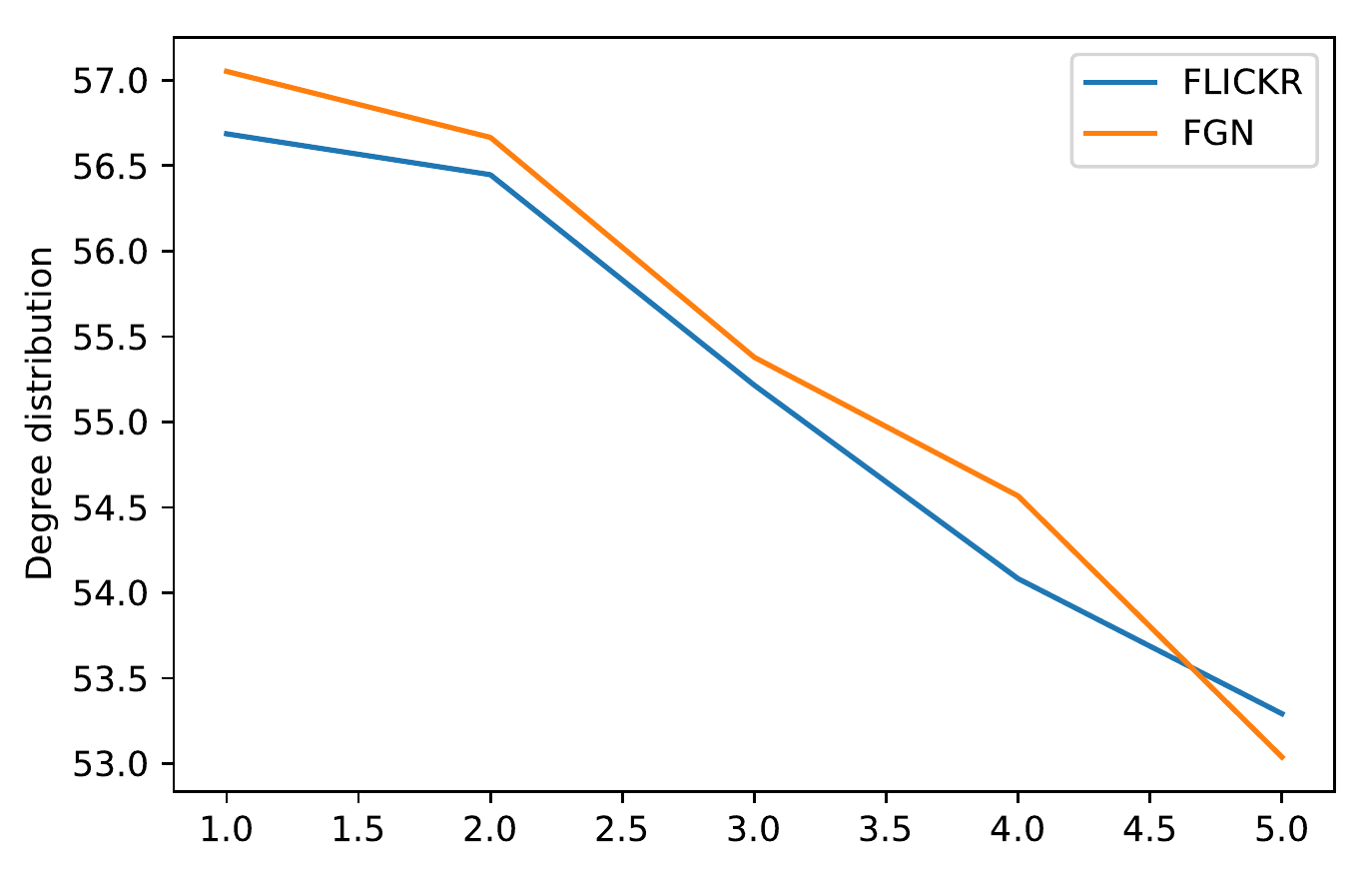}
\caption{Scree-Plot comparison: The $x$ axis is in log-sale. }
\label{fig:screeplot}
\end{figure}

\section{Conclusion}
We proposed and investigated a parametric statistical model of sparse random graphs called FGN that continuously interpolates between homogeneous, Poisson behavior on one hand, and fractal behavior with anomalous exponents and power law distributions on the other. We investigated the fundamental questions of parameter estimation and detecting  the presence of fractality based on observed network data. We demonstrated how to construct a natural stochastic block model within the FGN framework.

This work raises many natural questions for further investigations. These include a more detailed and rigorous mathematical study of the FGN  as a model of sparse random graphs. Another direction would be to obtain fundamental limits for natural statistical questions in this setting, particularly  the Stochastic Block Model in this context, and investigating the computational-statistical trade-off for these problems. Extending our analytical results, and consequently the range of the estimation and detection procedures, beyond the $L^2$ regime of the GMC would be a natural and interesting question. From a modeling perspective, it would be natural to explore beyond Gaussianity in the construction of our networks, for which the basic motivation and the  probabilistic fundamentals seem to be promising (see, e.g., \cite{barral2002multifractal}, \cite{bacry2003log}). Another direction would be to venture beyond the Euclidean set-up as the latent space.  We leave these and related questions as natural avenues for future investigation.




\subsection*{Acknowledgements}
We thank the anonymous referees for the careful reading of the manuscript, and their insightful comments and suggestions. SG was supported in part by the MOE grants R-146-000-250-133 and R-146-000-312-114. KB was supported in part by UC Davis CeDAR (Center for Data Science and Artificial Intelligence Research) Innovative Data
Science Seed Funding Program and NSF grant DMS-2053918. XY was supported by the FNR Grant MISSILe (R-AGR-3410-12-Z) at Luxembourg and Singapore Universities.
\bibliographystyle{alpha}
\bibliography{References}
\end{document}